\documentclass[letterpaper]{article}
\usepackage{aaai19}
\usepackage{times}
\usepackage{helvet}
\usepackage{courier}
\usepackage{url}
\usepackage{graphicx}
\usepackage[square,compress]{natbib}

\usepackage[utf8]{inputenc} 
\usepackage[T1]{fontenc}    
\usepackage{hyperref}       
\usepackage{url}            
\usepackage{booktabs}       
\usepackage{amsfonts}       
\usepackage{nicefrac}       
\usepackage{microtype}      

\title{Bayesian Fairness}

\author{Christos Dimitrakakis\\University of Oslo, Chalmers University\\ \textit{christos.dimitrakakis@gmail.com}
  \And
  Yang Liu\\UC Santa Cruz\\\textit{yangliu@ucsc.edu}
\And
David C. Parkes\\Harvard University\\\textit{parkes@eecs.harvard.edu}
\And
Goran Radanovic\\Harvard University\\\textit{gradanovic@g.harvard.edu}
}

\usepackage[textsize=tiny]{todonotes}
\DeclareSymbolFont{bbold}{U}{bbold}{m}{n}
\DeclareSymbolFontAlphabet{\mathbbold}{bbold}

\usepackage{enumerate}
\usepackage{amsmath}
\usepackage{amsfonts}
\usepackage{amssymb}
\usepackage{amsbsy}
\usepackage{bbm}
\usepackage{isomath}
\usepackage{amsthm}
\usepackage{dsfont}
\usepackage{algorithm}
\usepackage{algpseudocode}
\usepackage{mathrsfs}
\usepackage{paralist}
\usepackage{epsfig}
\usepackage[caption=false]{subfig}
\usepackage{makeidx} 
\usepackage{mathtools}
\usepackage{pgf}
\usepackage{tikz}
\usepackage{pgfplots}
\usepackage{gnuplot-lua-tikz}
\usepackage{pgfplotstable}

\pgfplotsset{compat=newest}
\usetikzlibrary{external}
\usetikzlibrary{automata,topaths,shapes,arrows,fit,decorations.markings}
\tikzstyle{utility}=[diamond,draw=black,draw=blue!50,fill=blue!10,inner sep=0mm, minimum size=8mm]
\tikzstyle{select}=[rectangle,draw=black,draw=blue!50,fill=blue!10,inner sep=0mm, minimum size=6mm]
\tikzstyle{hidden}=[dashed,draw=black]
\tikzstyle{RV}=[circle,draw=black,draw=blue!50,fill=blue!10,inner sep=0mm, minimum size=6mm]

\theoremstyle{plain}

\newtheorem{theorem}{Theorem}

\theoremstyle{definition}
\newtheorem{definition}{Definition}

\theoremstyle{remark}

\newtheoremstyle{example}  
{1em}       
{1em}       
{\small}      
{}          
{\scshape}  
{.}         
{.5em}      
{}          
\theoremstyle{example}
\newtheorem{example}{Example}

\definecolor{mycolor1}{rgb}{0.00000,0,1}%
\definecolor{mycolor2}{rgb}{1,0,0}%

\newcommand \E {\mathop{\mbox{\ensuremath{\mathbb{E}}}}\nolimits}

\renewcommand \Pr {\mathop{\mbox{\ensuremath{\mathbb{P}}}}\nolimits}

\newcommand{\set}[1]{\left\{\, #1 \,\right\} }
\newcommand{\cset}[2]{\left\{\, #1 ~\middle|~ #2 \,\right\} }

\newcommand{\indep}{\mathrel{\text{\scalebox{1.07}{$\perp\mkern-10mu\perp$}}}}

\newcommand\Simplex {{\mathcal P}}
\newcommand\Reals {{\mathds{R}}}

\newcommand \defn {\mathrel{\triangleq}}

\newcommand \argmax{\mathop{\rm arg\,max}}

\newcommand \grad {\nabla}

\newcommand \dd{\,\mathrm{d}}

\DeclareMathAlphabet{\mathpzc}{OT1}{pzc}{m}{it}

\def\clap#1{\hbox to 0pt{\hss#1\hss}}

\def\mathrlap{\mathpalette\mathrlapinternal}

\def\mathrlapinternal#1#2{%
           \rlap{$\mathsurround=0pt#1{#2}$}}

\newcommand \family {\mathscr{P}}

\newcommand {\CX} {\mathcal{X}}
\newcommand {\CY} {\mathcal{Y}}
\newcommand {\CZ} {\mathcal{Z}}
\newcommand {\CA} {\mathcal{A}}
\newcommand {\CB} {\mathcal{B}}

\newcommand \fair {f}
\newcommand \util {u}
\newcommand \val {V}

\newcommand \Param {\Theta}
\newcommand \param {\theta}

\newcommand \bel {\beta}
\newcommand \Bel {\CB}

\newcommand \act {a}
\newcommand \Act {\mathcal{A}}
\newcommand \out {y}
\newcommand \Out {\mathcal{Y}}
\newcommand \obs {x}
\newcommand \sns {z}
\newcommand \Obs {\mathcal{X}}
\newcommand \Sns {\mathcal{Z}}

\newcommand \pol {\pi}

\newcommand \Pbx[1] {\Pr_{\bel}{(#1 \mid x)}}

\newcommand\ind[1]{\mathop{\mbox{\ensuremath{\mathbb{I}}}}\left\{#1\right\}}

\newcommand \cd[1] {\todo{CD: #1}}

\pgfplotsset{
  marginal/.style={
    red,
    mark=o,
    line width=2pt
  },
  sample/.style={
    blue,
    mark=square,
    line width=2pt
  },
  bayes/.style={
    black,
    line width=2pt,
    dashed
  },
}

\newlength \fwidth

\def\clap#1{\hbox to 0pt{\hss#1\hss}}

\def\mathrlap{\mathpalette\mathrlapinternal}

\def\mathrlapinternal#1#2{%
           \rlap{$\mathsurround=0pt#1{#2}$}}

\ifodd 0
\newcommand{\rev}[1]{{\color{blue!20!black}#1}} 
\newcommand{\com}[1]{\textbf{\color{red}(COMMENT: #1)}} 
\newcommand{\clar}[1]{\textbf{\color{green!50!black}(NEED CLARIFICATION: #1)}}
\newcommand{\dcp}[1]{\textbf{\color{orange!50!black}(dcp: #1)}} 
\newcommand{\yang}[1]{\textbf{\color{red!50!black}(yang: #1)}} 
\newcommand{\goran}[1]{\textbf{\color{red!50!black}(goran: #1)}} 

\renewcommand{\cd}[1]{\textbf{\color{magenta!50!black}(cd: #1)}} 

\newcommand{\response}[1]{\textbf{\color{magenta}(RESPONSE: #1)}} 
\else
\newcommand{\rev}[1]{#1}
\newcommand{\com}[1]{}
\newcommand{\clar}[1]{}
\newcommand{\response}[1]{}
\newcommand{\dcp}[1]{} 
\newcommand{\yang}[1]{} 
\fi

\begin{document}

\maketitle

\begin{abstract}
  We consider the problem of how decision making can be fair when the
  underlying probabilistic model of the world is not known with
  certainty. We argue that recent notions of fairness in machine
  learning need to explicitly incorporate parameter uncertainty, hence
  we introduce the notion of {\em Bayesian fairness} as a suitable
  candidate for fair decision rules. Using balance, a definition of
  fairness introduced in~\citep{kleinberg2016inherent}, we show how a
  Bayesian perspective can lead to well-performing and fair decision
  rules even under high uncertainty.
  \end{abstract}

\section{Introduction}
\label{sec:introduction}
Fairness is a desirable property of policies applied to a
population of individuals. For example, college admissions should be
decided on variables that inform about merit, but fairness may also
require taking into account the fact that certain communities are
inherently disadvantaged.  At the same time, a person should not feel
that another in a similar situation obtained an unfair advantage.  All
this must be taken into account while still optimizing a decision
maker's utility function.

Much of the recent work on fairness in machine learning has focused on
analysing sometimes conflicting definitions. In this paper we do not
focus on proposing new definitions or algorithms. We instead take a
closer look at informational aspects of fairness. In particular, by adopting a Bayesian viewpoint,
we can explicitly take into account model uncertainty, something that
turns out to be crucial for fairness.

Uncertainty about the underlying reality has two main
effects. Firstly, most notions of fairness are defined with respect to
some latent variables, including model parameters. This means that we
need to take into account uncertainty in order to be fair. Secondly,
in many problems our decisions determine what data we will collect in
the future. Ignoring uncertainty may magnify subtle biases in our model.

By viewing fairness through a Bayesian decision theoretic perspective,
we avoid these problems.  In particular, we demonstrate that Bayesian
policies can optimally trade off utility and fairness by explicitly
taking into account uncertainty about model parameters.

We consider a setting where a decision maker (DM) makes a sequence of
decisions through some chosen policy $\pol$ to maximise her expected
utility $\util$. However, the DM must trade off utility with
some fairness constraint $\fair$. We assume the existence of some
underlying probability law $P$, so that the decision problem, when $P$
is known, can be written as:
\begin{align}
  \max_\pol  (1 - \lambda) \E^\pol_P \util - \lambda \E^\pol_P \fair,
  \label{eq:oracle-decision-problem-weights}
\end{align}
where $\lambda$ is the DM's trade-off between fairness and utility.\footnote{We do not consider the alternative constrained problem i.e.  $\max \cset{\E^\pol_P \util}{\E^\pol_P \fair \leq \epsilon}$, in the present paper.} In this paper we adopt a Bayesian viewpoint and assume the DM has some belief $\bel$ over some family of distributions $\family \defn \cset{P_\param}{\param \in \Param}$, which may contain the actual law, i.e. $P_{\theta^*} = P$ for some $\theta^*$.

The DM's policy $\pol$ defines what actions $\act_t \in \Act$ the DM
takes at different (discrete) times $t$ depending on the available
information. More precisely, at time $t$ the DM observes some data
$\obs_t \in \Obs$, and depending on her belief $\bel_t$ makes a
decision $\act_t \in \Act$, so that $\pol(\act_t \mid \bel_t, \obs_t)$
defines a probability over actions for every possible belief and
observation.  The DM's objective is to maximize her expected
utility. We model this as a function with structure $\util : \Act
\times \Out \to \Reals$, where $\Out$ is a set of \emph{outcomes}.
The fairness concept we focus on in this paper is a Bayesian version
of {\em balance}~\citep{kleinberg2016inherent}, which depends on the
policy at time $t$.  In the Bayesian setting, information is
central. The amount of uncertainty about the model parameters directly
influences how fairness can be achieved. Informally, the more
uncertain we are, the more stochastic the decision rule is.

\paragraph{Our contributions.} In this paper, we develop a framework for fairness that is defined as being
appropriate to the available information for the DM. The motivation for the Bayesian framework is that there can be a high degree of uncertainty, particularly when not a lot of data has been collected, or in sequential settings. This
informational notion of fairness is central to our discussion. It
entails that the DM should take into account how unfair she would be
under all possible models, weighted by their probability. While the
fairness concepts we use are grounded in conditional
independence~\citep{chouldechova2016fair,kleinberg2016inherent,HardtPNS16} type of
notions of fairness, we employ a Bayesian decision theoretic methodology. In particular, we cleanly separate model parameters from the DM's information, and the decision rule used by the DM. Fairness can thus be seen as a property of the decision rule with respect to the true model (which is used to \emph{measure} fairness), while
achieving it depends
on the DM's information (which is used to derive \emph{algorithms}).
\rev{
The Bayesian approach we adopt for fair decision making is generally applicable. In this paper, however, we focus on a simple setting so that we can work without model approximations, and proceed directly to the effect of uncertainty on fairness. The policies we obtain are qualitatively and quantitatively different when we consider uncertainty (by being Bayesian) compared to when we do not.}

The Bayesian algorithms we develop, based on gradient descent, take
into account uncertainty by considering fairness with respect to the
DM's information.  This inherent modeling of uncertainty allows us to
select better policies when those policies influence the data we
collect, and thus our knowledge about the model.  This is an important
informational feedback effect, that a Bayesian methodology can provide
in a principled way. We provide experimental results on the COMPAS
dataset~\citep{compas:dataset} as well as artificial data, showing the
robustness of the Bayesian approach, and comparing against methods that define fairness measures according to a single, marginalized model (e.g. \citep{HardtPNS16}). While we mainly treat the non-sequential setting, where the data is fixed, we can also accommodate sequential, bandits-style settings, as explained in Sections~\emph{The Sequential setting} and ~\emph{Sequential allocation}. The results provide a vivid
illustration of what can go wrong with a certainty-equivalent
approach to achieving fairness. 

All missing proofs and details can be found in our supplementary materials.

\paragraph{Related work.}
Recently algorithmic fairness has been studied quite extensively in
the context of statistical decision making.
But we are not aware of work that adopts
a Bayesian perspective. For instance,
\citep{dwork2012fairness,chouldechova2016fair,corbett2017algorithmic,kleinberg2016inherent,kilbertus2017avoiding}
studied fairness under the one-shot statistical decision making
framework. \citep{jabbari2016fair,joseph2016rawlsian} kicked off the
study of fairness in sequential decision making settings. Besides,
there is also a trending line of research on fairness in other machine
learning topics, such as clustering~\citep{Chierichetti2017fair},
natural language processing~\citep{BlodgettO17} and recommendation
systems~\citep{CelisV17}.  While the aforementioned works focused on
fairness in a specific context, such as classification,
\citep{corbett2017algorithmic} have considered how to satisfy some of
the above fairness constraints while maximizing expected utility.  For
a given model, they find a decision rule that maximizes expected
utility while satisfying fairness constraints. 
\citep{dwork2012fairness} consider
an individual-fairness approach,
and look for decision
rules that  are smooth in a sense that
similar \emph{individuals} are treated similarly.
\rev{
Finally, we'd like to mention the recent work of \citep{russell2017worlds}, which considers the problem of uncertainty from the point of view of causal modeling, with the three main differences being (a) They consider a PAC-like setting, rather than the Bayesian framework; (b) We show that the effect of uncertainty remains important even without varying the counterfactual assumptions, which is the main focus of that paper; (c) the Bayesian framework easily admits a sequential setting.}

In this paper, we focus on notions of fairness related to notions of
conditional independence, discussed next.

\section{Preliminaries}

\label{sec:preliminaries}
\citep{chouldechova2016fair} considers the problem of fair prediction
with disparate impact. She defines an action\footnote{Called a
  ``statistic'' in their paper.} $\act$ as \emph{test-fair} with respect
to the outcome $y$ and the sensitive variable $z$ if $y$ is
independent of $z$ under the action and parameter $\param$, i.e. if
$y \indep z \mid \act, \param$.  While the author does not explicitly
discuss the distribution $P_\param$, it is implicitly assumed to be that of the true model. We slightly generalize it as follows:
\begin{definition}[Calibrated decision rule]
  A decision rule $\pol(a \mid x)$ is \emph{calibrated} with respect to some distribution $P_\param$ if $y, z$ are independent for all actions $a$ taken, i.e. if
  \begin{equation}
    \label{eq:calibrated-rule}
    P^\pol_\param(y, z \mid a) =  P^\pol_\param(y \mid a) P^\pol_\param(z \mid a),
  \end{equation}
  where $P_\param^\pol$ is the distribution induced by $P_\param$ and the decision rule $\pol$.
  \label{def:calibrated-rule}
\end{definition}

\citep{kleinberg2016inherent} also consider two balance conditions, which we re-interpret as follows:
\ifdefined \longver
\begin{align}
  v(\act) = P_\param^\pol(y = 1 \mid \act, z)  \tag{calibration}\\
  \E_\param^\pol(\util \mid y, z) = \E_\param^\pol(\util \mid y, z') \tag{balance}.
\end{align}
The authors show that these cannot can be simultaneously achieved
under the distribution $P_\param^\pol$ induced by the underlying
parameter $\param$ and the decision rule $\pol$.
A sufficient condition for
\emph{balance} is the conditional independence :
$\util \indep z \mid y, \param, \pol$, i.e. when
$P_\param^\pol(\util \mid y , z) = P_\param^\pol(\util \mid y)$. This
will be the basis of our own balance definitions. 

\goran{Chouldechova (extended version) and Kleinberg conditions seems to be in spirit similar. If so, 
can we compress the above and simply state our version of them? (Or we consider this our contribution?)}
\cd{Yes, we can remove the above.}
\fi
\begin{definition}[Balanced decision rule]
  A decision rule~\footnote{Here we simplified the notation of the decision rule so that $\pol(a \mid x)$ corresponds to the probability of taking action $a$
given observation $x$.} $\pol(a \mid x)$ is balanced with respect to some distribution $P_\param$ if $a, z$ are independent for all $y$, i.e. if
  \begin{equation}
    \label{eq:balanced-rule}
    P_\param^\pol(a, z \mid y) =  P_\param^\pol(a \mid y) P_\param^\pol(z \mid y),
  \end{equation}
  where $P_\param^\pol$ is the distribution induced by $P_\param$ and the decision rule $\pol$.
  \label{def:balanced-rule}
\end{definition}
These authors also work with the true model, while we will
slightly generalize the definition, stating balance with respect to
any model parameter.

Unfortunately, the calibration and balanced conditions cannot be
achieved simultaneously for non-trivial environments
\citep{kleinberg2016inherent}. This is also true for our more general
definitions, as we show in Theorem~\ref{thm:impossible} in the
Supplementary material.  From a practitioner's perspective,
we must choose either the calibration condition or the balanced conditions
in order to find a fair decision rule.
We work with the balanced
condition, because it gracefully degrades to settings with
uncertainty. In particular, balance involves equality in the
expectation of a score function (when writing the probabilities as the
expectations of 0-1 score functions; also depending on an observation
$x$) under different values of a sensitive variable $z$, conditioned
on the true (but latent) outcome $y$. Consequently, balance can always
be satisfied---by using a trivial, for example randomized
decision rule, being independent of $x$.  The same, however, does not
hold for the calibration condition under model uncertainty. Note that there also exist other fairness notions that go beyond disparate treatment \citep{zafar2017fairness}. This merits future studies, and is out of the scope of the current draft.
 




\ifdefined \longver On \textbf{sequential decision
  problems}, such as multi-armed bandits and reinforcement learning,
\citet{joseph:fair-bandits} define an algorithm as fair if it plays
arms with highest means most of the time. \citet{jabbari:fair-mdp}
study a similar notion for Markovian environments, whereby the
algorithm's is fair the algorithm is more likely to play actions that
have a higher utility under the optimal policy. However, this notion
of fairness relies on oracle knowledge, while we focus on
subjectivity.  \fi

%


\section{Bayesian Formulation}
\label{sec:formulation}

We first introduce a concrete, statistical decision problem. The true
(latent) outcome $\out$ is generated independently of the DM's
decision, with a probability distribution that depends on the
available information $\obs$. There also exists a sensitive attribute
variable $\sns$, which may be dependent on $\obs$.\footnote{Depending
  on the application scenario, $\sns$ may actually be a subset of
  $\obs$ and thus directly observable, while in other scenarios it may
  be latent. Here we focus on the case where $\sns$ is not directly
  observed.}

\begin{definition}[Statistical decision problem]
 See Figure~\ref{fig:bayes-rule} for the decision diagram. The DM
 observes $\obs \in \Obs$, then takes a decision $\act \in \Act$ and
 obtains utility $\util(\out, \act)$ depending on a true (latent)
 outcome $\out \in \Out$ generated from some distribution
 $P_\param(\out \mid \obs)$.  The DM has a belief $\bel \in \Bel$ in
 the form of a probability distribution on parameters $\param \in
 \Param$ on a family $\family \defn \cset{P_\param(y \mid x)}{\param
   \in \Param}$ of distributions. In the Bayesian case, the belief
 $\bel$ is a posterior formed through a prior and available data. The
 DM has a utility function $\util : \CY \times \CA \to \Reals$, with
 utility depending on the DM's action and the outcome.
  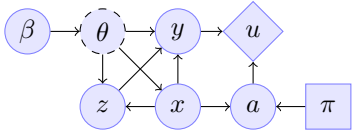
\begin{figure}
    \centering
    \begin{tikzpicture}
      \node[RV] at (-1,0) (b) {$\bel$}; 
      \node[RV,hidden] at (0,0) (p) {$\param$}; 
      \node[RV] at (0,-1) (z) {$z$}; 
      \node[RV] at (1,-1) (x) {$x$}; 
      \node[RV] at (1,0) (y) {$y$}; 
      \node[utility] at (2,0) (u) {$\util$}; 
      \node[RV] at (2,-1) (a) {$a$}; 
      \node[select] at (3,-1) (pol) {$\pol$}; 
      \draw[->] (x)--(y);
      \draw[->] (z)--(y);
      \draw[->] (x)--(z);
      \draw[->] (p)--(x);
      \draw[->] (p)--(y);
      \draw[->] (p)--(z);
      \draw[->] (b)--(p);
      \draw[->] (y)--(u);
      \draw[->] (a)--(u);
      \draw[->] (pol)--(a);
      \draw[->] (x)--(a);
    \end{tikzpicture}
    \caption{The basic Bayesian decision problem with observations $x$, outcome $y$, action $a$,  sensitive variable $z$,  utility $\util$, unknown parameter $\param$, belief $\bel$ and policy $\pol$. The joint distribution of $x,y,z$ is fully determined by the unknown parameter $\param$, while the conditional distribution of actions $a$ given observations $x$ is given by the selected policy $\pol$. The DM's  utility function is $\util$, while the fairness of the policy depends on the problem parameters.}
    \label{fig:bayes-rule}
  \end{figure}
\end{definition}

For simplicity, we will
assume that $\Obs$, $ \Act$, and $\Out$, are finite and discrete,
whereas $\Param$ will be a subset of $\Reals^n$. We focus on Bayesian
decision rules, i.e. rules whose decisions  depend upon a
posterior belief $\bel$. The Bayes-optimal
decision rule for a given posterior and utility, but ignoring
fairness, is defined below. 
\begin{definition}[Bayes-optimal decision rule]
  The Bayes-optimal decision rule $\pol^* : \Bel \times \CX \to \CA$ is a
  deterministic policy that maximizes the utility in expectation,
  i.e. takes action
  $\pol^*(\bel, x) \in \argmax_{a \in \CA} \util_{\bel}(a \mid x)$, with
  $\util_{\bel}(a \mid x) \defn \sum_y \util(y, a)  \Pbx{y}$,
  where  $\Pbx{y} \defn \int_\Param P_\param(y \mid x) \dd \bel (\param)$
  is the marginal distribution over outcomes conditional on the observations according to the DM's belief $\bel$.
  \label{def:Bayes-rule}
\end{definition}

The Bayes-optimal decision rule does not directly depend on the
sensitive variable $z$.  We are interested in operating over multiple
time periods. At time $t$, the DM observes $x_t$ and makes a
decision $a_t$ using policy $\pol_t$ and obtains some instantaneous
payoff $U_t = u(y_t, a_t)$ and fairness violation $F_t$.  As always,
the DM's utility is the sum of instantaneous payoffs over time, $U
\defn \sum_{t=1}^T u(y_t, a_t)$ and she is interested in finding a policy
maximising $U$ in expectation. Note the decision problem and its variables stay unchanged over time. 

Although the Bayes-optimal decision rule brings the highest expected reward to 
the DM, it may be unfair.  
In the sequel, we will define analogs of the \emph{balance} notion of fairness in
terms of decision rules $\pol$, and investigate appropriate decision rules, that 
possibly result in randomized policies.  In particular, we shall consider a utility function that combines the DM's utility with the societal benefit due to fairness, and search for the Bayes-optimal decision rules with respect to this new,
combined utility.

In particular, we define a Bayesian analogue of the maximization problem defined in~\eqref{eq:oracle-decision-problem-weights}:
\begin{align}
  &\max_\pol  (1 - \lambda) \E^\pol_\bel \util - \lambda \E^\pol_\bel \fair
  \nonumber \\
  =&
  \max_\pol  \int_\Param \left[(1 - \lambda) \E^\pol_\param \util - \lambda \E^\pol_\param \fair\right] \dd\bel(\param).
  \label{eq:bayes-decision-problem-weights}
\end{align}
To make this concrete, in the sequel we shall define the appropriate Bayesian version of the fairness-as-balance condition.


\subsection{Bayesian Balance}
\label{sec:bayesian-balance}

In the Bayesian setting, we would like our decisions to take into
account the impact on all possible models.
While perfect balance is generally achievable, it turns out that
sometimes only
a trivial decision rule can satisfy it
in a setting with model
uncertainty (where balance
must hold exactly, for all possible model
parameters).
%
%
\begin{theorem}
  A trivial decision rule of the form $\pol(a \mid x) = p_a$ can always satisfy balance for a Bayesian decision problem. However, it may be the only balanced decision rule, even when a non-trivial balanced policy can be found for every possible $\param \in \Param$.
  \label{lem:trivial-balance}
\end{theorem}

The proof, as well as  an example illustrating this
result, are in the supplementary materials. For this reason, we consider the the $p$-norm of the deviation from fairness with respect to our belief $\bel$:
\begin{definition}[Bayesian Balance]
  We say that a decision rule $\pol(\cdot)$ is $(\alpha,p)$-Bayes-balanced with respect to $\bel$ if:
  \begin{align}
    f(\pol) &\defn 
    \int_{\Param }
    \sum_{a, y, z}
    \biggl |\sum_x \pol(a | x)
    [P_\param(x, z | y)
    \nonumber \\
    &- P_\param(x | y) P_\param(z | y) ] 
    \biggr |^p
    \dd \bel(\param) 
    \leq \alpha^p.
      \label{eq:balanced-bayes-norm}
  \end{align}
\end{definition}

This definition captures the expected deviation from balance of  policy $\pol$, for a Bayesian DM under their belief $\bel$. It measures the deviation of the specific policy $\pol$ from perfect balance with respect to each possible parameter $\param$, and weighs it according to the probability of that model. This
provides a graceful trade-off between achieving near-balance in the most likely models, while avoiding extreme unfairness in less likely ones.

Why not use a single point estimate for the model, instead of the full Bayesian approach? This would entail simply measuring balance (and utility) with respect to the marginal model
$\Pr_\bel \defn \int_\Param P_\param \dd \bel(\param)$.
\begin{definition}[Marginal balance]
A decision rule $\pol(\cdot)$ is $(\alpha,p)$-marginal-Balanced if  $\forall a, y,  z$:
\begin{align}
  \sum_{a, y, z}  \biggl |\sum_x \pol(a | x)
    \left[\Pr_\bel(x, z | y)
    - \Pr_\bel(x | y) \Pr_\bel(z | y) \right]\biggl|^p &\leq \alpha.
  \label{eq:marginal-balance}
\end{align}
\label{def:marginal-balance}
\end{definition}

One problem with this, which we will see in our experimental results,
is that the DM would assume that the marginal model is the correct one, and
may be very unfair towards other high-probability models.

Still, both balance conditions can provide a bound on balance with respect
to the true model.
For this, denote the true underlying model as $\theta^*$, and define the $(\epsilon,\delta)$-accurate belief.
\begin{definition}
  We call $\bel(\param)$ an $(\epsilon,\delta)$-accurate belief with respect to the true model $\param^* \in \Param$, if with $\bel$-probability at least $1-\delta$, $\forall x,y,z$: $$|P_{\theta}(x|y,z) -P_{\theta^*}(x|y,z)| \leq \epsilon,~| P_{\theta}(x|y)-P_{\theta^*}(x|y)| \leq \epsilon,$$  i.e. that set $\Param_\epsilon$ for which the above conditions hold has measure $\bel(\Param_\epsilon) \geq 1 - \delta$.
\end{definition}

Under some conditions the balance achieved through either definition
provides an approximation to balance
under the true model, as shown by the following theorem.
\begin{theorem}\label{noise:model}
  If a decision rule satisfies either $(\alpha,1)$-marginal-balance or $(\alpha,1)$-Bayes-balance for $\bel$ or both, and $\bel$ is $(\epsilon,\delta)$-accurate, then the resulting decision rule is a $$(\alpha+2|\Act| \cdot |\Sns| \cdot |\Out|\cdot(\epsilon+\delta),1) \text{-balanced}$$ decision rule w.r.t. the true model $\theta^*$.
\end{theorem}

This theorem says that if our belief $\bel$ is concentrated around the
true model $P_{\theta^*}$, and our decision rule is fair with respect
to either definition, then it is also fair with respect to the true
model.


\subsection{The Sequential setting}
\label{sec:sequential}
\rev{We can extend the approach to a sequential setting,
where the information learned depends on the action.
For example, if we grant a loan application, we will only later discover if the loan is going to be paid off. This will affect our future decisions.
Analogous to other sequential decision making problems such as Markov decision processes\citep{Puterman:MDP:1994},} we need to solve the following optimization problem over a
time horizon $T$:
\begin{align}
  \max_\pol \E_{\bel_1} \left[\sum_{t=1}^T (1 - \lambda) U_t - \lambda F_t\right],
\end{align}
where $\pol$ now must explicitly map future beliefs $\bel_t$ to probabilities over actions.
 If the data that the DM obtains depends on her
decisions $a_t$, then she must consider adaptive policies, as the next
belief depends on what the data obtained by the policy was.

We can reformulate the maximization problem so as to explicitly include the future changes in belief:
\begin{align}
  \val^*(\bel_t) &\defn \sup_{\pol_t} \E^{\pol_t}_{\bel_t} \left[(1 - \lambda) U_t - \lambda F_t\right]
          \nonumber \\
          &+ \sum_{\bel_{t+1}} V^*(\bel_{t+1}) \Pr_{\bel_t}^{\pol_t} (\bel_{t+1}),
\end{align}
under the mild assumption that the set of reachable next beliefs is
finite (easily satisfied when the set of outcomes is finite). This formulation is not different from standard MDP formulation (e.g., the 
reinforcement learning settings) that features the trade-offs between \emph{exploration} (obtaining new knowledges) and \emph{exploitation} (maximizing utilities). We know in these settings a myopic policy will lead to sub-optimal solutions. 

However, just as in the bandits case \citep[c.f.][]{duff2002olc}), the
above computation is intractable, as the policy space is exponential in
$T$. For this reason, in this paper we only consider
{\em myopic policies}
that select a policy (and decision) that is optimal for the current
step $t$,
trading utility and fairness
as well as the value of `single-step' information.
%
%
A specific instance of this type of sequential version of the problem
is experimentally studied in Section {\em Sequential allocation}.


\section{Algorithms} 
\label{sec:optim-balanc-bayes}
 
The algorithms we employ in this paper are based on gradient descent. We compare the full Bayesian framework with the simpler approach of assuming that the marginal model is the true one. In particular, for the Bayesian framework, we directly optimize~\eqref{eq:bayes-decision-problem-weights}. Using the marginal simplification, we maximize~\eqref{eq:oracle-decision-problem-weights} with respect to the marginal model $\Pr_\bel$.

\subsection{Balance gradient descent}
\label{sec:credible-fairness}

Again, as in the Bayesian setting, we have a family of models
$\set{P_\param}$ with a corresponding subjective distribution
$\bel(\param)$. In order to derive algorithms, we shall focus on the quantity:
\begin{equation}
  C(\pol, \param)
   \defn
   \sum_{y,z} \big\|\sum_x \pol(a \mid x)  \Delta_\param(x,y,z)\big\|_p,
\end{equation}
to be the deviation from balance for decision rule $\pol$ under
parameter $\param$, where
\begin{equation}
  \Delta_\param(x,y,z) \defn P_\param(x, z \mid y)
  - P_\param(x \mid y) P_\param(z \mid y).
\end{equation}
Then the Bayesian balance of the policy is $f(\pol) = \int_\Param C(\pol,  \param) \dd \bel(\param)$.

In order to find a rule trading off utility for balance, we can
maximize a convex combination of the expected utility and
deviation specified in \eqref{eq:bayes-decision-problem-weights}. In particular, we can look for a parametrized rule $\pol_w$
solving the following unconstrained maximization problem.
\begin{align}
  \max_{\pol_w} &
           \int_\Param
           \val_\param(\pol_w)
           \dd \bel(\param),
\nonumber \\
&
  \val_\param(\pol_w) \defn 
                (1 - \lambda) \E_\param^{\pol_w} \util 
                - \lambda C(\pol_w, \param)
  \label{eq:penalty}
\end{align}
To perform this maximization we use parametrized policies and
 stochastic gradient descent.  In particular, for a finite set
$\CX$ and $\CY$, the policies can be defined in terms of parameters
$w_{xa} = \pol(a \mid x)$. Then we can perform stochastic gradient
descent as detailed in Section \emph{Gradient calculations for optimal balance decision} of supplementary materials, by sampling $\param \sim \bel$ and calculating the gradient for each sampled $\param$.

For the \emph{marginal} decision rule, we employ the same approach, but instead of sampling the parameters from the posterior, we use the parameters of the marginal model. The approach is otherwise identical.

%

%



\setlength \fwidth {0.2\textwidth}

\begin{figure*}  
\centering
   \subfloat[$\lambda=0$]{
%

\begin{tikzpicture}

\begin{axis}[%
width=0.951\fwidth,
height=0.75\fwidth,
at={(0\fwidth,0\fwidth)},
scale only axis,
xmode=log,
xmin=1,
xmax=100,
xminorticks=true,
xlabel style={font=\color{white!15!black}},
xlabel={t},
ymin=5,
ymax=7,
ylabel style={font=\color{white!15!black}},
ylabel={V},
axis background/.style={fill=white}
]
\addplot [color=mycolor1, line width=2.0pt, forget plot]
  table[row sep=crcr]{%
1	5.43561523059217\\
2	5.90370595362012\\
3	5.86193877425643\\
4	6.1974251447707\\
5	6.37538745144895\\
6	6.29084098710727\\
7	6.37533818756081\\
8	6.50213010808347\\
9	6.5991724386949\\
10	6.52046523295858\\
11	6.72543938736623\\
12	6.79086221447693\\
13	6.73720035131321\\
14	6.73338708212267\\
15	6.72543680895541\\
16	6.69340595276163\\
17	6.7374401920428\\
18	6.76459636329051\\
19	6.81917107124367\\
20	6.81961739074537\\
21	6.80848937741292\\
22	6.76496273549776\\
23	6.81943224558462\\
24	6.82707956663492\\
25	6.82513922838762\\
26	6.80872080097762\\
27	6.8194279261825\\
28	6.8189260277632\\
29	6.80836851035401\\
30	6.80853114441057\\
31	6.81931954445528\\
32	6.73856115284411\\
33	6.73816264740434\\
34	6.73350236577814\\
35	6.70687112292105\\
36	6.6785416628419\\
37	6.67686803452494\\
38	6.73847232580075\\
39	6.81942547936638\\
40	6.82688492791769\\
41	6.82666694620967\\
42	6.82685885290022\\
43	6.82704019063763\\
44	6.82689924230517\\
45	6.82677823960367\\
46	6.8270597587393\\
47	6.8194560097514\\
48	6.82689843014005\\
49	6.82701602901863\\
50	6.82681169930566\\
51	6.81949797532045\\
52	6.82669354234299\\
53	6.81943329511497\\
54	6.82651491512469\\
55	6.82704770792639\\
56	6.82679743590289\\
57	6.82658758155496\\
58	6.82669828218253\\
59	6.82660948782845\\
60	6.82669798274291\\
61	6.8265894448791\\
62	6.82677142197074\\
63	6.82681491374407\\
64	6.82666114793132\\
65	6.80855816288338\\
66	6.81943481045244\\
67	6.82660396672482\\
68	6.82386355702123\\
69	6.81947471789097\\
70	6.82700074175287\\
71	6.82693449357647\\
72	6.82709850214994\\
73	6.82660499877464\\
74	6.82697649873285\\
75	6.82670593695921\\
76	6.82675415174255\\
77	6.82687238452097\\
78	6.82695391513762\\
79	6.82698434483159\\
80	6.82686631929662\\
81	6.82688791087935\\
82	6.82646811164595\\
83	6.82644252014829\\
84	6.82665183718036\\
85	6.82669758685249\\
86	6.82676665652696\\
87	6.82695751973057\\
88	6.82658946086894\\
89	6.82655885737069\\
90	6.82664771118609\\
91	6.82677188302957\\
92	6.82669723134228\\
93	6.82649131899508\\
94	6.82694362296807\\
95	6.82701100800417\\
96	6.82689276186491\\
97	6.82697942973237\\
98	6.8265959260873\\
99	6.82685566499548\\
100	6.82661794251553\\
};
\addplot [color=mycolor2, dashed, line width=3.0pt, forget plot]
  table[row sep=crcr]{%
1	5.41383459952623\\
2	6.03977740818528\\
3	6.34445896962323\\
4	6.36360081485984\\
5	6.5327140246455\\
6	6.54312011538435\\
7	6.60180812138142\\
8	6.73373741924754\\
9	6.73566894837812\\
10	6.79427286043087\\
11	6.79651547742057\\
12	6.81745557776205\\
13	6.82361089403465\\
14	6.82453008513155\\
15	6.81832088322606\\
16	6.78744690174988\\
17	6.8245307167606\\
18	6.82360223044389\\
19	6.82320123869577\\
20	6.82664800772996\\
21	6.82680811620926\\
22	6.82647925470442\\
23	6.82692086547104\\
24	6.8269265759928\\
25	6.82693425257449\\
26	6.8268229060682\\
27	6.82700552417735\\
28	6.82667175358748\\
29	6.82668179704755\\
30	6.82696045898621\\
31	6.82660833511297\\
32	6.82680589726841\\
33	6.82671589421294\\
34	6.82682294359754\\
35	6.82701707308604\\
36	6.82650493640409\\
37	6.82699220517708\\
38	6.82668021066387\\
39	6.82672209879424\\
40	6.82681624339904\\
41	6.82673440179816\\
42	6.82689211021934\\
43	6.82677720152432\\
44	6.82694542985608\\
45	6.82666056532852\\
46	6.82671478552557\\
47	6.82693547661731\\
48	6.82672478318047\\
49	6.82677490664466\\
50	6.82656216462366\\
51	6.82648634438907\\
52	6.82704568401014\\
53	6.826862274337\\
54	6.82707378292588\\
55	6.82701064453263\\
56	6.82655872230876\\
57	6.82680284995162\\
58	6.82675484205814\\
59	6.82669898677491\\
60	6.8269587402534\\
61	6.82659491807925\\
62	6.8270850859566\\
63	6.82673295666194\\
64	6.82699550615579\\
65	6.82685223576769\\
66	6.82671348441515\\
67	6.82673202844262\\
68	6.82712820153405\\
69	6.82689184429809\\
70	6.82683272593571\\
71	6.82648201507614\\
72	6.82674866309257\\
73	6.82671747591638\\
74	6.82699024215261\\
75	6.82717941910823\\
76	6.82694493886605\\
77	6.8267239335274\\
78	6.82696730470131\\
79	6.82694340122485\\
80	6.82682976358288\\
81	6.82660119141053\\
82	6.82686412764447\\
83	6.82699084765003\\
84	6.8267722735916\\
85	6.82674429780571\\
86	6.82694302876559\\
87	6.82675714405635\\
88	6.82696031751509\\
89	6.82689255300199\\
90	6.82702023326717\\
91	6.82696488326549\\
92	6.82660850972417\\
93	6.8268004525055\\
94	6.82679611886495\\
95	6.82682594191007\\
96	6.82686735298234\\
97	6.82672601577022\\
98	6.82680403091386\\
99	6.82704486981606\\
100	6.8267221923614\\
};
\end{axis}
\end{tikzpicture}%
  }
  \subfloat[$\lambda=0.25$]{
%
%
%
\begin{tikzpicture}

\begin{axis}[%
width=0.951\fwidth,
height=0.75\fwidth,
at={(0\fwidth,0\fwidth)},
scale only axis,
xmode=log,
xmin=1,
xmax=100,
xminorticks=true,
xlabel style={font=\color{white!15!black}},
xlabel={t},
ymin=2.6,
ymax=4,
axis background/.style={fill=white}
]
\addplot [color=mycolor1, line width=2.0pt, forget plot]
  table[row sep=crcr]{%
1	3.1225306321747\\
2	3.26014310612664\\
3	3.29418445127395\\
4	3.47879313017756\\
5	3.35564611143873\\
6	3.49131322269369\\
7	3.68174585577693\\
8	3.78562307885069\\
9	3.69494014047982\\
10	3.72674568562776\\
11	3.787407280996\\
12	3.82404009032336\\
13	3.86926786613041\\
14	3.88002640718584\\
15	3.86337187268527\\
16	3.86014177831422\\
17	3.85061996870957\\
18	3.8684063856261\\
19	3.87904334272389\\
20	3.88104037446538\\
21	3.87143700762788\\
22	3.87090910352101\\
23	3.88004509693712\\
24	3.880832160368\\
25	3.87976511107898\\
26	3.87897087972067\\
27	3.8807014825704\\
28	3.88019719087416\\
29	3.8812304870751\\
30	3.88052311027887\\
31	3.8798427816137\\
32	3.87861235876383\\
33	3.8807233122881\\
34	3.87984089474894\\
35	3.8802735035861\\
36	3.87754941698894\\
37	3.87837752879586\\
38	3.88046084969416\\
39	3.87960762987224\\
40	3.87983157966137\\
41	3.88123912305534\\
42	3.87944935826282\\
43	3.87293414584046\\
44	3.87157858607664\\
45	3.87980874478316\\
46	3.87937443446057\\
47	3.87730991191998\\
48	3.87945060825207\\
49	3.88090895100173\\
50	3.8785734519277\\
51	3.87798151347719\\
52	3.88050434508109\\
53	3.8814939766162\\
54	3.8798130196323\\
55	3.87879504351125\\
56	3.87844452517725\\
57	3.87941766592103\\
58	3.87909631541187\\
59	3.88060278968173\\
60	3.87866496543081\\
61	3.87912634146572\\
62	3.88009790597964\\
63	3.88132634873409\\
64	3.87694614361366\\
65	3.87797524969697\\
66	3.87106233317266\\
67	3.87673382579365\\
68	3.88016311938393\\
69	3.87988042738732\\
70	3.88137871848826\\
71	3.88020439983991\\
72	3.88164305705249\\
73	3.88057374163073\\
74	3.87944122973096\\
75	3.87946057691413\\
76	3.87083378953187\\
77	3.87990659133386\\
78	3.87979161675014\\
79	3.88031060736714\\
80	3.87839586303847\\
81	3.87952551007611\\
82	3.87951969538205\\
83	3.88139125450766\\
84	3.87973732140646\\
85	3.87940163393449\\
86	3.88066178195297\\
87	3.88103006714394\\
88	3.88066953653238\\
89	3.88082078319006\\
90	3.87472749688655\\
91	3.88008082069647\\
92	3.87967361812959\\
93	3.88090413136021\\
94	3.88056106501739\\
95	3.88080328962013\\
96	3.87934628830826\\
97	3.8808920865691\\
98	3.88011750144638\\
99	3.8777921466677\\
100	3.87849475700746\\
};
\addplot [color=mycolor2, dashed, line width=3.0pt, forget plot]
  table[row sep=crcr]{%
1	2.99143665012784\\
2	2.79898361574143\\
3	3.15512594897397\\
4	2.79576701999414\\
5	3.18726210134358\\
6	2.92375990338545\\
7	3.32483013400853\\
8	3.53112814539514\\
9	3.33965991778691\\
10	3.45131607217245\\
11	3.94508976123525\\
12	3.836863715312\\
13	3.87355633606793\\
14	3.8719470750431\\
15	3.86222833932847\\
16	3.88206495550576\\
17	3.88007478789528\\
18	3.87963515750772\\
19	3.8797485147608\\
20	3.87898372292093\\
21	3.88009598496287\\
22	3.87772440262873\\
23	3.8806354564985\\
24	3.87573025540419\\
25	3.87826831420734\\
26	3.87985855809266\\
27	3.87884319804875\\
28	3.88058528643557\\
29	3.88052456475016\\
30	3.87942264859486\\
31	3.87975333957797\\
32	3.87992375478627\\
33	3.88051357656132\\
34	3.87927059405848\\
35	3.87671387834963\\
36	3.87969164478351\\
37	3.88082583721043\\
38	3.87985772513412\\
39	3.87999217653948\\
40	3.87927306298158\\
41	3.8799357327431\\
42	3.88090400097951\\
43	3.88049136655767\\
44	3.87672782145846\\
45	3.87852424319418\\
46	3.88034765801848\\
47	3.8802042495979\\
48	3.87329348986675\\
49	3.8800517467665\\
50	3.87877612818924\\
51	3.8797014718911\\
52	3.8793884521176\\
53	3.88017277417863\\
54	3.879654766403\\
55	3.87686141534805\\
56	3.87973454691705\\
57	3.88048999589008\\
58	3.88008419478083\\
59	3.87982220574849\\
60	3.87973992874786\\
61	3.88016494762213\\
62	3.87906769114994\\
63	3.87848579166382\\
64	3.88055051614756\\
65	3.87960129703661\\
66	3.87930770439852\\
67	3.88088071418367\\
68	3.87869130188828\\
69	3.87877227892613\\
70	3.88050790929068\\
71	3.87751667405787\\
72	3.87981779239174\\
73	3.87826469476027\\
74	3.87916773593815\\
75	3.88076423127758\\
76	3.87933903266768\\
77	3.87926593704359\\
78	3.87972589195196\\
79	3.87951586734992\\
80	3.87668671328687\\
81	3.87551969840022\\
82	3.88067178473997\\
83	3.87959289592913\\
84	3.87939055326519\\
85	3.87973358807666\\
86	3.88055795763589\\
87	3.87942793781209\\
88	3.88155441003813\\
89	3.88059074052121\\
90	3.88063558027451\\
91	3.87983587943771\\
92	3.88062276375354\\
93	3.87944892503933\\
94	3.87978727399675\\
95	3.87842192127326\\
96	3.88008862130387\\
97	3.87957329653882\\
98	3.88065708183656\\
99	3.88163646802876\\
100	3.880945903385\\
};
\end{axis}
\end{tikzpicture}%
  }
  \subfloat[$\lambda=0.5$]{
%
%
%
\begin{tikzpicture}

\begin{axis}[%
width=0.951\fwidth,
height=0.75\fwidth,
at={(0\fwidth,0\fwidth)},
scale only axis,
xmode=log,
xmin=1,
xmax=100,
xminorticks=true,
xlabel style={font=\color{white!15!black}},
xlabel={t},
ymin=-0.5,
ymax=2,
axis background/.style={fill=white}
]
\addplot [color=mycolor1, line width=2.0pt, forget plot]
  table[row sep=crcr]{%
1	1.49482275002729\\
2	1.55189801874307\\
3	1.83882282524574\\
4	1.52676135329926\\
5	1.52250894001475\\
6	1.50531193637208\\
7	1.49370990313388\\
8	1.37096218428233\\
9	1.69252155043745\\
10	1.55460080897603\\
11	1.77758287824582\\
12	1.78571302890306\\
13	1.54890376428204\\
14	1.71169980084041\\
15	1.76677542583961\\
16	1.79953306466444\\
17	1.82405896323217\\
18	1.80562258987096\\
19	1.80346372081525\\
20	1.80490522019214\\
21	1.80251864349951\\
22	1.7995840760207\\
23	1.79674085586432\\
24	1.806276571382\\
25	1.81194647146786\\
26	1.82251291652957\\
27	1.80498094116489\\
28	1.8272796312303\\
29	1.80961959279471\\
30	1.81460137914183\\
31	1.81914916541425\\
32	1.81392379924514\\
33	1.8175432722861\\
34	1.8201478669015\\
35	1.81322273333226\\
36	1.81866634019849\\
37	1.81584318931693\\
38	1.82674652739832\\
39	1.80998697957214\\
40	1.81282450380663\\
41	1.82359620006886\\
42	1.82111823610408\\
43	1.82290286766311\\
44	1.81848658484606\\
45	1.82597066631328\\
46	1.81211306569529\\
47	1.82229032524054\\
48	1.81015879277648\\
49	1.82773686379265\\
50	1.81919753307513\\
51	1.81588515919281\\
52	1.81601964781797\\
53	1.81749265215332\\
54	1.82007077847615\\
55	1.82443336078475\\
56	1.82566220552982\\
57	1.82280628482688\\
58	1.83010292007432\\
59	1.82074163362279\\
60	1.81866781943321\\
61	1.82007236165932\\
62	1.82851746138336\\
63	1.81651773830087\\
64	1.81292309805879\\
65	1.81097040505865\\
66	1.82112946548917\\
67	1.81770401450192\\
68	1.82937997687998\\
69	1.81678649459996\\
70	1.8165514209844\\
71	1.82296124397186\\
72	1.81862636466603\\
73	1.82650919211409\\
74	1.81607019489911\\
75	1.81867067516648\\
76	1.81455612740102\\
77	1.81867226688053\\
78	1.81606266049731\\
79	1.81975419210626\\
80	1.8199782675556\\
81	1.82408715716298\\
82	1.81743782363045\\
83	1.81223449479776\\
84	1.82717780566641\\
85	1.81912057558481\\
86	1.8172207355233\\
87	1.82568116472211\\
88	1.81790553467899\\
89	1.81656810711344\\
90	1.82428308742398\\
91	1.8302964899591\\
92	1.8235600515568\\
93	1.82461711494186\\
94	1.81752760149805\\
95	1.82427771424488\\
96	1.81013748370312\\
97	1.81434628986094\\
98	1.81752237732997\\
99	1.82299730857796\\
100	1.82163132633734\\
};
\addplot [color=mycolor2, dashed, line width=3.0pt, forget plot]
  table[row sep=crcr]{%
1	0.200825888532816\\
2	-0.18008397338503\\
3	-0.223529206493137\\
4	-0.0753925034742253\\
5	0.929582939469821\\
6	0.676750152052953\\
7	1.02527985927929\\
8	1.1907138620247\\
9	1.2237426133549\\
10	1.06210547603728\\
11	1.39678777894161\\
12	1.40974585704355\\
13	1.46876264509228\\
14	1.28575761797557\\
15	1.14882024066358\\
16	1.46377660532148\\
17	1.56963239593687\\
18	1.66907185300346\\
19	1.72286071590441\\
20	1.71545125972188\\
21	1.72076857522247\\
22	1.70977648241448\\
23	1.71421934606972\\
24	1.81961918396706\\
25	1.82486972192735\\
26	1.82205160751454\\
27	1.82970447855445\\
28	1.81975762676593\\
29	1.82608890642267\\
30	1.82790882921885\\
31	1.82305138404279\\
32	1.8251774473024\\
33	1.82180970765009\\
34	1.82535183629553\\
35	1.81491795840182\\
36	1.81139862864023\\
37	1.81469578968229\\
38	1.82336600035984\\
39	1.81964088422403\\
40	1.81156044832757\\
41	1.810193776878\\
42	1.81604224661703\\
43	1.81250494134455\\
44	1.81608869175998\\
45	1.8282763978939\\
46	1.81991287660889\\
47	1.81582197002685\\
48	1.8216328147671\\
49	1.81509662525265\\
50	1.81967603345366\\
51	1.82217581494403\\
52	1.81614879299987\\
53	1.82145119528924\\
54	1.82274949930973\\
55	1.81087216488805\\
56	1.82123830998816\\
57	1.82383525525647\\
58	1.82396221554174\\
59	1.81879833466988\\
60	1.82491278694981\\
61	1.81620584157348\\
62	1.81404426119364\\
63	1.81664144185502\\
64	1.82267074046262\\
65	1.81879699133598\\
66	1.82336546099748\\
67	1.82165386661804\\
68	1.82023091505951\\
69	1.81012488462726\\
70	1.81501540952032\\
71	1.81181672885677\\
72	1.81254029915353\\
73	1.81600596980627\\
74	1.82931752912828\\
75	1.82015694163513\\
76	1.82358452377752\\
77	1.82866674935601\\
78	1.81692298347132\\
79	1.81944525961421\\
80	1.81664892791178\\
81	1.81442168852228\\
82	1.81792302272156\\
83	1.81625638592239\\
84	1.81487418368066\\
85	1.81950552074845\\
86	1.82379573251609\\
87	1.82325367665394\\
88	1.81973949148139\\
89	1.82191418084967\\
90	1.82040245967622\\
91	1.81445589110821\\
92	1.8240198622246\\
93	1.82666876362131\\
94	1.82016060196285\\
95	1.81832870479655\\
96	1.81289773601116\\
97	1.81923026308184\\
98	1.81644826216325\\
99	1.8285673621865\\
100	1.82098487160116\\
};
\end{axis}
\end{tikzpicture}%
  }
  \\
  \subfloat[$\lambda=0.75$]{
%
%
%
\begin{tikzpicture}

\begin{axis}[%
width=0.951\fwidth,
height=0.75\fwidth,
at={(0\fwidth,0\fwidth)},
scale only axis,
xmode=log,
xmin=1,
xmax=100,
xminorticks=true,
xlabel style={font=\color{white!15!black}},
xlabel={t},
ymin=-2.5,
ymax=1,
ylabel style={font=\color{white!15!black}},
ylabel={$V$},
axis background/.style={fill=white}
]
\addplot [color=mycolor1, line width=2.0pt, forget plot]
  table[row sep=crcr]{%
1	0.647144134432733\\
2	0.354850976127959\\
3	0.0168639129456241\\
4	0.33766832503612\\
5	-0.0784205215215747\\
6	0.40982240965253\\
7	0.363433709756422\\
8	0.487537525939838\\
9	0.431152195949632\\
10	0.466753420623816\\
11	0.466046683781554\\
12	0.487639867766271\\
13	0.49928611677378\\
14	0.490004904071104\\
15	0.498703670535377\\
16	0.482282069479057\\
17	0.486215750602667\\
18	0.51200210477187\\
19	0.519158867275507\\
20	0.511139624471004\\
21	0.506337079832673\\
22	0.519756013471365\\
23	0.512478132539362\\
24	0.50561853158497\\
25	0.50367451864099\\
26	0.505789589179426\\
27	0.498493731862348\\
28	0.513999762823966\\
29	0.509031336730771\\
30	0.502243676582301\\
31	0.506433275705123\\
32	0.521016485542623\\
33	0.526842976746132\\
34	0.508251384089969\\
35	0.500493352905132\\
36	0.503129210767534\\
37	0.503757032056775\\
38	0.515181885197508\\
39	0.506533240100524\\
40	0.496901036088606\\
41	0.504225507189952\\
42	0.514208053813117\\
43	0.504088815994414\\
44	0.518542650637304\\
45	0.517960786175664\\
46	0.498037687784748\\
47	0.503132758322316\\
48	0.504029263192454\\
49	0.503828931380531\\
50	0.500842613994183\\
51	0.50501721061598\\
52	0.500134221801918\\
53	0.505075687615254\\
54	0.507469697299228\\
55	0.496716583067885\\
56	0.509488749928148\\
57	0.511840097569015\\
58	0.511998689972054\\
59	0.503126545909836\\
60	0.495692655782744\\
61	0.506223884739539\\
62	0.513183980970198\\
63	0.508301694112646\\
64	0.504416408744524\\
65	0.511060119454027\\
66	0.506322788389227\\
67	0.514699572608527\\
68	0.51261061903061\\
69	0.503902880515783\\
70	0.498133699719478\\
71	0.506180412178859\\
72	0.515481595602637\\
73	0.505243395812402\\
74	0.505417707455431\\
75	0.525999253554647\\
76	0.514900594260543\\
77	0.513133569073988\\
78	0.514813914295146\\
79	0.5154428029658\\
80	0.502309066900648\\
81	0.514003087192598\\
82	0.502056969897616\\
83	0.513961366209833\\
84	0.508132541201984\\
85	0.50951474504085\\
86	0.502497616777838\\
87	0.516745077782601\\
88	0.509533550568892\\
89	0.509441419800372\\
90	0.511621119905325\\
91	0.499281497010895\\
92	0.506791977356655\\
93	0.510924531333717\\
94	0.507140657929554\\
95	0.507758822160545\\
96	0.518941046352285\\
97	0.503241690881632\\
98	0.526932315786321\\
99	0.509588662131115\\
100	0.496698046513878\\
};
\addplot [color=mycolor2, dashed, line width=3.0pt, forget plot]
  table[row sep=crcr]{%
1	-2.10972868392334\\
2	-2.19333624498975\\
3	-1.91686278456048\\
4	-2.18778897925519\\
5	-0.717836941991068\\
6	-1.00519431495389\\
7	-0.31865696501745\\
8	-0.214449450673659\\
9	0.0620315037330776\\
10	-0.0633437599414009\\
11	0.240439293753107\\
12	0.20117040629708\\
13	0.0852074084356361\\
14	0.374007133800986\\
15	0.523308465766127\\
16	0.53039286708511\\
17	0.551012210639785\\
18	0.538278671077171\\
19	0.544743920266012\\
20	0.515266804624304\\
21	0.556430402855409\\
22	0.552565315791179\\
23	0.313887880702347\\
24	0.0331619757448477\\
25	0.517352407502555\\
26	0.540052093575284\\
27	0.28886150661496\\
28	0.600047358828579\\
29	0.32982671469552\\
30	0.544235034067292\\
31	0.515184747722274\\
32	0.507180641676155\\
33	0.505475646263916\\
34	0.508548928549737\\
35	0.513534075937625\\
36	0.507474231127796\\
37	0.499904861246735\\
38	0.510236131090415\\
39	0.514855654446097\\
40	0.505060205862894\\
41	0.500826433832186\\
42	0.502292978741724\\
43	0.520676097163094\\
44	0.51597318015191\\
45	0.497647450391874\\
46	0.517890127212812\\
47	0.505297941187003\\
48	0.501818466736886\\
49	0.532976167629579\\
50	0.518933643676142\\
51	0.509803853934258\\
52	0.510472817279734\\
53	0.533884347670337\\
54	0.507901986716691\\
55	0.510845626710028\\
56	0.505955585646727\\
57	0.513293548042703\\
58	0.508550933318634\\
59	0.506468912419602\\
60	0.507577904310786\\
61	0.503009204595177\\
62	0.513329086293772\\
63	0.510445350549673\\
64	0.501183087893078\\
65	0.504290200539491\\
66	0.501087329612797\\
67	0.507266957925408\\
68	0.508421467994892\\
69	0.503527571532604\\
70	0.501016621070913\\
71	0.517349051109618\\
72	0.5083699408298\\
73	0.49581628606482\\
74	0.510944095141426\\
75	0.507049795537558\\
76	0.517292706210585\\
77	0.498966281939918\\
78	0.522936223793223\\
79	0.504252862191017\\
80	0.509890781903466\\
81	0.520460448881835\\
82	0.495246536698139\\
83	0.514663301602922\\
84	0.499815826106441\\
85	0.5071506212488\\
86	0.510366693146688\\
87	0.517792447358288\\
88	0.506120892862666\\
89	0.506836410928914\\
90	0.506534601103221\\
91	0.516362761948352\\
92	0.496097892533478\\
93	0.509368284534166\\
94	0.502859144059977\\
95	0.509198803344931\\
96	0.509419538221135\\
97	0.513944071470028\\
98	0.517357368190226\\
99	0.517365405811031\\
100	0.515137732972473\\
};
\end{axis}
\end{tikzpicture}%
  }
  \subfloat[$\lambda=1$]{
%
%
%
\begin{tikzpicture}

\begin{axis}[%
width=0.951\fwidth,
height=0.75\fwidth,
at={(0\fwidth,0\fwidth)},
scale only axis,
xmode=log,
xmin=1,
xmax=100,
xminorticks=true,
xlabel style={font=\color{white!15!black}},
xlabel={t},
ymin=-3,
ymax=-0,
axis background/.style={fill=white}
]
\addplot [color=mycolor1, line width=2.0pt, forget plot]
  table[row sep=crcr]{%
1	-0.0079538799589155\\
2	-0.029543354739046\\
3	-0.0433687510715851\\
4	-0.0267997530098269\\
5	-0.0552496970950875\\
6	-0.0426903095165145\\
7	-0.0679414709532705\\
8	-0.0337160620542988\\
9	-0.00774590526565739\\
10	-0.0133546186947674\\
11	-0.00840124016438263\\
12	-0.00820499270165407\\
13	-0.00590896597439416\\
14	-0.0176499482814558\\
15	-0.00252354607061481\\
16	-0.00717860489756771\\
17	-0.00319735702605778\\
18	-0.00187242718476113\\
19	-0.00288948427503447\\
20	-0.00496589148972686\\
21	-0.00224283116617609\\
22	-0.00150386647593838\\
23	-0.00159912641771346\\
24	-0.00368014203475327\\
25	-0.00109444086079066\\
26	-0.000446191788070531\\
27	-0.000734292624807064\\
28	-0.000401051896476695\\
29	-0.000932649177910425\\
30	-0.00141656183508094\\
31	-0.000515309722339206\\
32	-0.00100075848721307\\
33	-0.000844245985664303\\
34	-0.000208029214434817\\
35	-0.000817863343109707\\
36	-6.87462086350267e-05\\
37	-2.66168146450068e-05\\
38	-0.000119567334337687\\
39	-5.30513852849204e-05\\
40	-0.000150595335306735\\
41	-0.000156079689737337\\
42	-0.000132583609651288\\
43	-0.000105671878860136\\
44	-0.000363674074689149\\
45	-0.000332094737739339\\
46	-0.000132280430068013\\
47	-0.000360764969405453\\
48	-0.000202804047634601\\
49	-0.000914969854600546\\
50	-0.00193609011034195\\
51	-0.000485374853592484\\
52	-0.000669058007541339\\
53	-0.000100638550352485\\
54	-0.000125608139962873\\
55	-5.81618915309334e-05\\
56	-1.65360876138627e-05\\
57	-1.52893349942399e-05\\
58	-4.50288483102586e-06\\
59	-4.30220492720571e-06\\
60	-1.47759406493938e-05\\
61	-1.13574777807691e-06\\
62	-3.05814881114346e-07\\
63	-8.03477068955045e-07\\
64	-6.11637159899579e-06\\
65	-1.5135799037908e-06\\
66	-3.34000395704954e-07\\
67	-2.68482160484598e-07\\
68	-4.02673438368847e-07\\
69	-2.57633046635655e-06\\
70	-1.50530556519574e-06\\
71	-4.19887918924763e-06\\
72	-1.10386452954666e-06\\
73	-1.237121797162e-06\\
74	-8.25241630388637e-08\\
75	-9.6784364482321e-08\\
76	-5.18656159601938e-08\\
77	-2.0863147713059e-08\\
78	-2.19152522685541e-07\\
79	-4.46832491850821e-08\\
80	-7.72699679516323e-10\\
81	-8.9349853205317e-08\\
82	-5.30542790635243e-10\\
83	-1.98465955523162e-08\\
84	-8.02736813289686e-09\\
85	-3.91370113361051e-08\\
86	-8.73236313782212e-08\\
87	-7.71171273296649e-09\\
88	-4.60709804182315e-11\\
89	-5.66442994947387e-08\\
90	-2.12750392924871e-08\\
91	-4.04237672076645e-08\\
92	-3.88491488358296e-08\\
93	-2.65185909008292e-08\\
94	-2.16508821825751e-08\\
95	-4.37319159845521e-09\\
96	-7.49389154669673e-08\\
97	-2.08453943166065e-08\\
98	-2.6708469428986e-08\\
99	-5.90199758325732e-09\\
100	-1.66200369130425e-08\\
};
\addplot [color=mycolor2, dashed, line width=3.0pt, forget plot]
  table[row sep=crcr]{%
1	-2.722153411267\\
2	-1.94516327026953\\
3	-1.64508920163008\\
4	-2.46256694199538\\
5	-2.01900980362994\\
6	-0.952183343545042\\
7	-0.885305409684854\\
8	-1.48994419555813\\
9	-0.801874165878863\\
10	-1.03111506169953\\
11	-0.698290341643939\\
12	-1.28703555289472\\
13	-0.906153080690147\\
14	-0.361312065333853\\
15	-0.528382063924027\\
16	-0.5004332650923\\
17	-0.179922228143026\\
18	-0.24485383970683\\
19	-0.518467911031\\
20	-0.396631360348243\\
21	-0.286584107530347\\
22	-0.0603461255916548\\
23	-0.148105826934939\\
24	-0.0208517310325713\\
25	-0.185801109769805\\
26	-0.0389451605100446\\
27	-0.0852808488988264\\
28	-0.04951256976962\\
29	-0.00293319542166685\\
30	-0.0200069023336759\\
31	-0.0167216620657047\\
32	-0.0344243494755248\\
33	-0.040236107222278\\
34	-0.0201039315812888\\
35	-0.016858757911495\\
36	-0.0241656670145098\\
37	-0.0198255436378841\\
38	-0.0318284149608048\\
39	-0.0162088393933059\\
40	-0.039594321688047\\
41	-0.00450617792555742\\
42	-0.0109730978164902\\
43	-0.0129224627755731\\
44	-0.0864968986718645\\
45	-0.0481343182828859\\
46	-0.0556736894707966\\
47	-0.0191104076358362\\
48	-0.0412646735520683\\
49	-0.0176149892716187\\
50	-0.00441802660286454\\
51	-0.000432317205929555\\
52	-0.000240846243408926\\
53	-0.000188863174901928\\
54	-5.70154076065477e-05\\
55	-0.000208560630301867\\
56	-7.85284996924208e-05\\
57	-2.21993073117616e-05\\
58	-6.54599360020493e-06\\
59	-4.69267031139676e-05\\
60	-0.000122224255520404\\
61	-1.29640550327843e-05\\
62	-5.90687176957188e-06\\
63	-4.43604109352983e-06\\
64	-5.58923974059919e-07\\
65	-9.0581274547844e-07\\
66	-8.18019431653096e-07\\
67	-4.83986305006053e-07\\
68	-1.03839991204521e-06\\
69	-6.2129588397099e-07\\
70	-4.64933447872208e-07\\
71	-1.69060089887443e-07\\
72	-3.05849363198025e-07\\
73	-1.20927203685865e-07\\
74	-9.55692055288801e-08\\
75	-7.60571618399762e-08\\
76	-6.24518519058812e-08\\
77	-6.93872577478441e-08\\
78	-4.02800714128452e-08\\
79	-1.93759546630455e-08\\
80	-3.9343338870102e-08\\
81	-1.50059774810078e-08\\
82	-1.59401431176721e-08\\
83	-7.27357793840842e-08\\
84	-8.26336498923122e-09\\
85	-3.80585755198375e-08\\
86	-4.86056369782277e-08\\
87	-3.717722801735e-08\\
88	-1.62525653097387e-09\\
89	-1.59693921129335e-08\\
90	-6.77232274862789e-08\\
91	-4.8663724644361e-08\\
92	-1.14002947485849e-08\\
93	-4.70024474232401e-08\\
94	-1.65402175987577e-08\\
95	-3.54811160226814e-08\\
96	-1.2803379939206e-08\\
97	-6.09303501530596e-08\\
98	-1.45453130886335e-08\\
99	-5.65938338638241e-08\\
100	-3.36173018218328e-08\\
};
\end{axis}
\end{tikzpicture}%
  }
  \subfloat[legend]{
    \raisebox{4em}{\begin{tikzpicture}

  \begin{axis}[%
    hide axis,
    xmin=10,
    xmax=50,
    ymin=0,
    ymax=0.4,
    legend style={draw=white!15!black,legend cell align=left}
    ]
    \addlegendimage{color=mycolor1, line width=2.0pt}
    \addlegendentry{Bayes};
    \addlegendimage{color=mycolor2, dashed, line width=3.0pt};
    \addlegendentry{Marginal};
  \end{axis}
\end{tikzpicture}}
  }
  \caption{\textbf{Synthetic data.} Test of effect of amount of data for Bayesian versus marginal decision rules, for different values of the $\lambda$ parameter, with respect to the true model. As more weight is placed on guaranteeing fairness, we see that the Bayesian approach is better able to guarantee fairness for the true model. The plots show the average performance over 10 runs, with an initially uniform prior over a set of 8 models, one of which is the correct one. In this setting $|\CA| = |\CY| = |\CZ| = 2$ and $|\CX| = 8$.}
  \label{fig_exp_1}
\end{figure*}
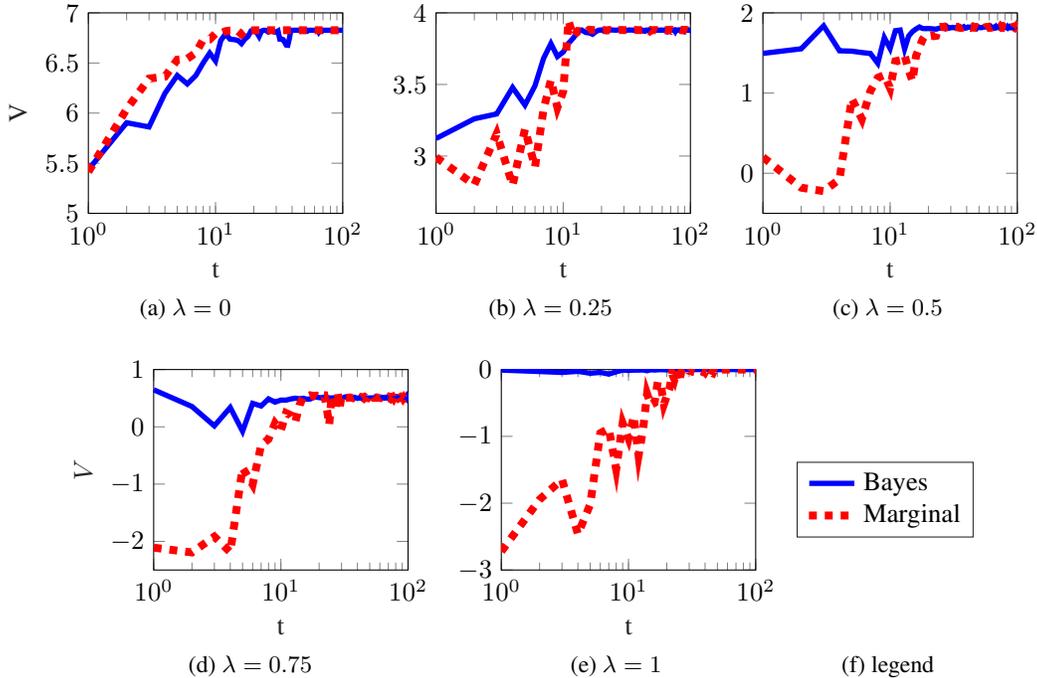

\section{Experiments} 
\label{sec:experiments}
In this section we study the utility-fairness trade-off on
artificial and real data sets.  We compare our approach, which uses a
decision rule based on the full Bayesian problem, to classical
approaches such as~\cite{HardtPNS16} which simply optimizes the DM's
policy with respect to a single model. Rather than introducing a new
fairness metric, we use a generalized version of the balance metric in
\cite{kleinberg2016inherent}, which is also a generalization of the
equality of opportunity in \cite{HardtPNS16}. We see that the Bayesian
approach very gracefully handles fairness, even with high model
uncertainty, while a marginal approach can be blatantly unfair.
For a fair comparison, in both cases we assume the same prior distribution for the
parameters. We focus on a simple model where posterior distributions can be calculated in closed-form, in order to focus on the choice of policy, rather than the case with approximate inference. However, our algorithm is generally applicable and could be combined with e.g. MCMC inference.

\textbf{Performance} is evaluated with respect to actual balance and
utility achieved: for the synthetic data this will be measured
according to the actual data-generating distribution, while for the
COMPAS data, it will be the empirical distribution on a holdout set.

\textbf{The algorithm} for optimising policies uses (stochastic) gradient descent. In particular, the Bayesian policy minimizes \eqref{eq:balanced-bayes-norm} by sampling $\param$ from the posterior distribution $\bel$ and then taking a step in the gradient direction. The marginal policy simply performs steepest gradient descent for the marginal model.

\textbf{The results} shown in
Figures~\ref{fig_exp_1}--\ref{fig:sequential-allocation} display the
performance of the corresponding (Bayesian or marginal) decision rule
for different value of $\lambda$ as more data is acquired. In the
first two experiments, we assume that no matter what the decision of
the DM is, $z_t, y_t$ are always observed after the DM's decision and
so the model is fully updated. In that setting, it is not necessary
for the DM to take into account expected future information for her
actions. However, in the third experiment, described in
Section {\em Sequential allocation}, 
the values of $z_t$ and $y_t$
are only observed when the DM makes the decision $a_t = 1$, and the DM
faces a generalized exploration problem.


\textbf{The model} we employ throughout is a discrete Bayesian network
model, with finite $\CX, \CY, \CZ, \CA$. The models are thus described
through multinomial distributions that capture the dependency between
different random variables. The available data is used to calculate a
\emph{posterior} distribution $\bel(\param)$. From this, we calculate
both an approximate marginal balanced rule as well as a Bayesian
balanced rule. The former uses the marginal model directly, while the
latter uses $k = 16$ samples from the posterior
distribution.\footnote{We found empirically that 16 was a sufficient
  number for stable behaviour and efficient computation. For $k=1$ the
  algorithm devolves into an approximation of Thompson sampling.} We
tested these approaches both on synthetic data and on the COMPAS
dataset.
The conjugate prior distribution to this model is a simple Dirichlet-product, as the network is discrete. The graphical model is fully connected, so the model uses the factorization $P_\param(x, y, z) = P_\param(y \mid x, z) P_\param(x \mid z) P_\param(z)$. We used this simple modeling choice throughout the paper, apart from the small experiment on synthetic data in the following section. In all cases where a Dirichlet prior was used, the Dirichlet prior parameters were all set equal to $1/2$.

\subsection{Experiments on synthetic data.}
Here we consider a discrete decision problem, with $|\CX| = 8$,
$|\CY| = |\CZ| = |\CA| = 2$, and $\util(y, a) = \ind{y = a}$.  In our
first experiment, we generate $100$ observations from this model. We
performed the experiment 10 times, each time generating data from a
fully connected discrete Bayesian network with uniformly randomly
selected parameters. Unlike the rest of the paper, in this example,
the prior distribution has finite support on only 8 models. This means that the posterior will have effectively converged to the true model after 100 observations.

As can be seen in Figure \ref{fig_exp_1}, the relative performance of
the Bayesian approach w.r.t. the marginal approach increases as we put
more emphasis on fairness (Figure \ref{fig_exp_1}
(a) cares nothing about fairness.). In some cases (e.g. Figure \ref{fig_exp_1}
(c)), value for the marginal approach decreases at the beginning and
eventually reaches the same value as the Bayesian approach after
sufficient amount of data is received.  This conforms with our
hypothesis that one should take into account model uncertainty.  The
fact that both approaches converge toward the maximum value is in
accordance with our formal results (Theorem \ref{noise:model}).

Finally, Figure~\ref{fig_exp_1:tradeoff} and its extended version (Figure~\ref{fig_exp_1:tradeoff_extend} in supplementary materials) more clearly shows how well the two different solutions perform with respect to the utility fairness trade-off. As we vary $\lambda$ and the amount of data, both methods achieve the same utility. However the Bayesian approach consistently achieves lower fairness violations for similar $U$.

\subsection{Experiments on COMPAS data.}
For the COMPAS dataset, we consider a discretization where fields such as the number of offenses are converted to binary features.\footnote{We arrived at the specific discretization through cross validating the performance of a discrete Bayesian classifier over possible discretizations.} 
We used the first 6000 data points for training and the remaining 1214 points for validation. Two attributes are sensitive (sex, race), while 6 attributes (relating to prior convictions and age) are used for the policy.
With discretization, there are a total of 12 distinct values for the sensitive attributes and 141 for the observables used for the underlying model. The prediction is whether or not there is recidivism in the next two years, with utility function $\util(a, y) = \ind{a = y}$.

Figure~\ref{fig:compas-dbn} and its extended version (Figure~\ref{fig:compas-dbn_extend} in supplementary materials) show the results of applying our analysis
to the COMPAS dataset used by ProPublica. Since in this case the true
model was unknown, the results are calculated with respect to the
marginal model estimated on the holdout set. In this scenario we can see that when we only focus on classification performance, the marginal and Bayesian decision rules perform equally well. However, as we place more emphasis on fairness, we observe that the Bayesian approach dominates.
\footnote{We note here that measured performance performance may not monotonically increase with respect to the (rather small) holdout set. Even if we had converged to the true model, measuring with respect to an empirical estimate is problematic, as it will be $\epsilon$-far away from the true model. This is particularly important for fairness considerations.}

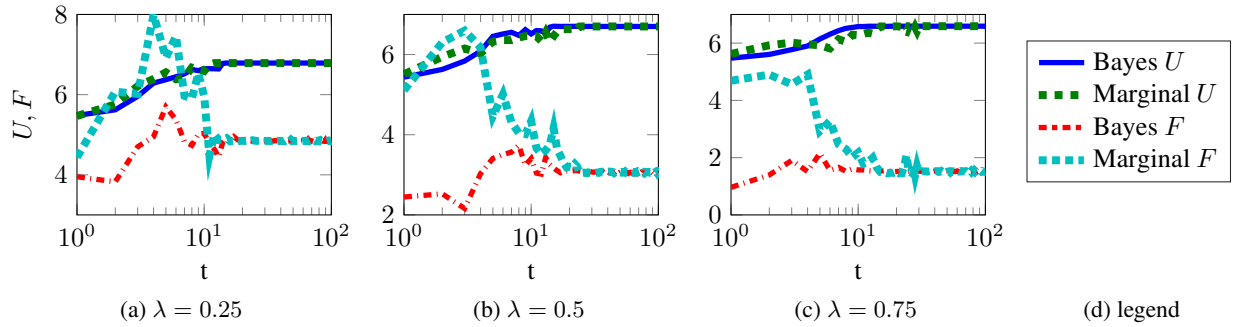
\begin{figure*}
\centering
  \subfloat[$\lambda=0.25$]{
%
%
\definecolor{mycolor1}{rgb}{0.00000,0.75000,0.75000}%
\begin{tikzpicture}

\begin{axis}[%
width=0.951\fwidth,
height=0.75\fwidth,
at={(0\fwidth,0\fwidth)},
scale only axis,
xmode=log,
xmin=1,
xmax=100,
xminorticks=true,
xlabel={t},
ylabel={$U, F$},
ymin=3,
ymax=8,
axis background/.style={fill=white}
]
\addplot [color=blue, line width=2.0pt, forget plot]
  table[row sep=crcr]{%
1	5.48092030379997\\
2	5.62187002690099\\
3	5.95863155946218\\
4	6.28898882529671\\
5	6.37263646084564\\
6	6.44619863535187\\
7	6.51978070369432\\
8	6.6260112803707\\
9	6.59969215942956\\
10	6.64710138462075\\
11	6.64739151299339\\
12	6.64510042770851\\
13	6.64152382540431\\
14	6.77898626490996\\
15	6.78942211106642\\
16	6.78461368559541\\
17	6.78747870326049\\
18	6.78938642097554\\
19	6.78832161718462\\
20	6.78775050690658\\
21	6.78944722876843\\
22	6.78917577314025\\
23	6.78819139814796\\
24	6.78724347966865\\
25	6.78821620436663\\
26	6.78831214850696\\
27	6.7880922368358\\
28	6.7881723236791\\
29	6.78794233482332\\
30	6.78813814586129\\
31	6.78820354351858\\
32	6.78841099428386\\
33	6.78808353541935\\
34	6.7881908225767\\
35	6.78774827696057\\
36	6.7885187243922\\
37	6.78845419838743\\
38	6.78814836495893\\
39	6.78821228904136\\
40	6.78818735684172\\
41	6.78799405430338\\
42	6.7882713816731\\
43	6.78927591624498\\
44	6.78946030227827\\
45	6.78821020841118\\
46	6.78827185230104\\
47	6.78860337850886\\
48	6.78748352584787\\
49	6.78799354670293\\
50	6.78842672785986\\
51	6.7880902860153\\
52	6.78808388647639\\
53	6.78747738661796\\
54	6.78819916095484\\
55	6.78835895289234\\
56	6.78843293590063\\
57	6.78825725240079\\
58	6.78829796145325\\
59	6.788057466024\\
60	6.78839237562291\\
61	6.78834851740735\\
62	6.78735253906824\\
63	6.78775411240988\\
64	6.78863576035963\\
65	6.78755861187669\\
66	6.78954119015546\\
67	6.78873337301612\\
68	6.78816033804725\\
69	6.78820344967906\\
70	6.78776361064147\\
71	6.78814177015364\\
72	6.7879366996741\\
73	6.7880936081835\\
74	6.78825561226921\\
75	6.78821316644409\\
76	6.78918787026003\\
77	6.78818333235925\\
78	6.78819645711209\\
79	6.78810751942358\\
80	6.78842982644288\\
81	6.78739048420254\\
82	6.78823104668009\\
83	6.78781609618644\\
84	6.78820734692313\\
85	6.78828044289443\\
86	6.78722601773148\\
87	6.78755406937019\\
88	6.78811157856524\\
89	6.78808107645364\\
90	6.789012996106\\
91	6.7881650039973\\
92	6.78821720889247\\
93	6.78807168410838\\
94	6.78810800021703\\
95	6.78793379320975\\
96	6.78826043758827\\
97	6.78807781437123\\
98	6.78818537206819\\
99	6.78850763801694\\
100	6.7884621883911\\
};
\addplot [color=black!50!green, dashed, line width=3.0pt, forget plot]
  table[row sep=crcr]{%
1	5.46839092366965\\
2	5.75703558072284\\
3	6.21187308663476\\
4	6.39102664381444\\
5	6.55586234089632\\
6	6.38413615637934\\
7	6.42300581162868\\
8	6.67498693237809\\
9	6.63039319016224\\
10	6.58773391003488\\
11	6.75296577578321\\
12	6.786443793693\\
13	6.78805569029723\\
14	6.78745585366746\\
15	6.78832428887165\\
16	6.78830066490763\\
17	6.78839106970021\\
18	6.78808802788535\\
19	6.78798694582777\\
20	6.78817693494755\\
21	6.7881008909547\\
22	6.78842032667698\\
23	6.78783128904056\\
24	6.78875615083487\\
25	6.78845727850269\\
26	6.78819401413074\\
27	6.78834196285126\\
28	6.7872101033022\\
29	6.7879980507075\\
30	6.78824430728795\\
31	6.78816970586208\\
32	6.78818801908122\\
33	6.78814017664448\\
34	6.78825078323049\\
35	6.7881484859051\\
36	6.78818691493551\\
37	6.78808987102091\\
38	6.78816298338296\\
39	6.78799846176474\\
40	6.78828094214078\\
41	6.78727871939925\\
42	6.78803975010666\\
43	6.78725571301125\\
44	6.78871518089493\\
45	6.78840508970577\\
46	6.7875707596091\\
47	6.7881728716072\\
48	6.78846952964538\\
49	6.78817772400779\\
50	6.78840021851439\\
51	6.78822370152359\\
52	6.78740754507343\\
53	6.78818214694285\\
54	6.78826729796029\\
55	6.78794709000275\\
56	6.78819769453097\\
57	6.78773687436855\\
58	6.78760780764942\\
59	6.78820024812653\\
60	6.78822638564773\\
61	6.78816361161253\\
62	6.78832451810031\\
63	6.78840398227284\\
64	6.78810379421081\\
65	6.7882192655374\\
66	6.7882812746018\\
67	6.78767994390469\\
68	6.78789414506478\\
69	6.78833152743326\\
70	6.78812570838307\\
71	6.78853051769768\\
72	6.78820591270197\\
73	6.78840134179175\\
74	6.78833646862341\\
75	6.78714175770295\\
76	6.78827649125506\\
77	6.78761416139175\\
78	6.78818828234791\\
79	6.78820835330531\\
80	6.78869172117444\\
81	6.78888773725641\\
82	6.78807839890497\\
83	6.78822911534438\\
84	6.7882605284278\\
85	6.78824522127325\\
86	6.78813733439761\\
87	6.78828150957633\\
88	6.787818875421\\
89	6.78808023524246\\
90	6.78808690828464\\
91	6.78820378585089\\
92	6.78810395152641\\
93	6.78826222843939\\
94	6.78821516715624\\
95	6.7884244130633\\
96	6.78815086989971\\
97	6.78823200461007\\
98	6.78762008169212\\
99	6.78764861589493\\
100	6.78759288324044\\
};
\addplot [color=red, dashdotted, line width=2.0pt, forget plot]
  table[row sep=crcr]{%
1	3.95263838270111\\
2	3.82503765619641\\
3	4.69915687329073\\
4	4.95179395517991\\
5	5.69532493678198\\
6	5.37334301528087\\
7	4.83235868797524\\
8	4.73554152570934\\
9	5.01931591636941\\
10	5.03432141135122\\
11	4.79254541499617\\
12	4.63914092183212\\
13	4.44750001169128\\
14	4.81685316598652\\
15	4.9147788424582\\
16	4.91327394352937\\
17	4.95995623494322\\
18	4.8945337204222\\
19	4.84879148065828\\
20	4.83909002285821\\
21	4.88259365579378\\
22	4.88389090533673\\
23	4.8443938066954\\
24	4.83840179753397\\
25	4.84558816878397\\
26	4.84905292663818\\
27	4.84147078022581\\
28	4.84372820754065\\
29	4.83890505616955\\
30	4.84232199646838\\
31	4.84523950410094\\
32	4.85078354779625\\
33	4.84135735710565\\
34	4.84520888873435\\
35	4.84215081653731\\
36	4.85535850522084\\
37	4.85185247997884\\
38	4.84260169610015\\
39	4.84620634763515\\
40	4.84523575187968\\
41	4.83902567068879\\
42	4.84701671196802\\
43	4.87609116537311\\
44	4.88206656252824\\
45	4.84539564610091\\
46	4.84731781906082\\
47	4.85657048784665\\
48	4.84464814453533\\
49	4.84034483610184\\
50	4.85098637586876\\
51	4.85234480413716\\
52	4.84223427910482\\
53	4.83645625338907\\
54	4.84534540433531\\
55	4.84989668463203\\
56	4.85152070699286\\
57	4.84710109351825\\
58	4.84850862271226\\
59	4.8417612393451\\
60	4.8505172651455\\
61	4.84854018635913\\
62	4.84166599328614\\
63	4.83795694229326\\
64	4.85812270662427\\
65	4.8507748368422\\
66	4.88437423777575\\
67	4.85926481587375\\
68	4.84382853660602\\
69	4.8450886394879\\
70	4.83777595797139\\
71	4.8436077111013\\
72	4.83723787081233\\
73	4.8419858580276\\
74	4.84700191788378\\
75	4.84679719167575\\
76	4.88422845265261\\
77	4.84492363174228\\
78	4.84542290433572\\
79	4.84308012880218\\
80	4.85170602717477\\
81	4.84406941230317\\
82	4.84661435851207\\
83	4.83788327052865\\
84	4.84567275514356\\
85	4.84723479294532\\
86	4.83903092538253\\
87	4.83854193953482\\
88	4.84165658956618\\
89	4.84096009660066\\
90	4.86812900077183\\
91	4.84417172920604\\
92	4.84595715415907\\
93	4.84059852688428\\
94	4.84207974058152\\
95	4.84058822114876\\
96	4.84739615953176\\
97	4.84066509683728\\
98	4.84408611041905\\
99	4.85435432738002\\
100	4.85140753714343\\
};
\addplot [color=mycolor1, dotted, line width=3.0pt, forget plot]
  table[row sep=crcr]{%
1	4.43942617049761\\
2	6.07517227920279\\
3	6.01511546400837\\
4	7.99001185146673\\
5	6.91853861731467\\
6	7.45736885559621\\
7	5.9696968988519\\
8	5.90044821555372\\
9	6.53253989933909\\
10	5.95793744141482\\
11	4.47853828240861\\
12	5.01187651983101\\
13	4.86994172661999\\
14	4.87457926082996\\
15	4.91605950930105\\
16	4.83664217269984\\
17	4.84487405751951\\
18	4.84572345362514\\
19	4.84496677844009\\
20	4.84859591315895\\
21	4.8439187330126\\
22	4.85436336951604\\
23	4.84095204112765\\
24	4.86334743088784\\
25	4.85229857867871\\
26	4.84514781002158\\
27	4.84965309635877\\
28	4.83928916416435\\
29	4.84189589312187\\
30	4.84704232748438\\
31	4.84549575927436\\
32	4.84486903809857\\
33	4.84236622368815\\
34	4.84766997345754\\
35	4.85758994431681\\
36	4.84579416567247\\
37	4.84096626422101\\
38	4.84505804961238\\
39	4.8440266791363\\
40	4.84775057449601\\
41	4.84209322722535\\
42	4.84050324640196\\
43	4.83980167280306\\
44	4.85923425685094\\
45	4.85111829634059\\
46	4.84132164675338\\
47	4.84370161642999\\
48	4.87223462946916\\
49	4.84432618495737\\
50	4.8500961427862\\
51	4.84586521700634\\
52	4.84466882674989\\
53	4.84385534411402\\
54	4.84618282826888\\
55	4.85639560861603\\
56	4.8456548959247\\
57	4.84125063954534\\
58	4.84248664382494\\
59	4.84531192138565\\
60	4.84571944195174\\
61	4.84383104434909\\
62	4.84870278970119\\
63	4.85126878016323\\
64	4.84210931804219\\
65	4.84625260846576\\
66	4.84761300621132\\
67	4.83951697497941\\
68	4.84891722764123\\
69	4.84990546659528\\
70	4.84234548798647\\
71	4.85552485686154\\
72	4.84534656853895\\
73	4.85214524633418\\
74	4.84833846211763\\
75	4.83836834799854\\
76	4.84747334309443\\
77	4.84577873600088\\
78	4.84566127923588\\
79	4.84656159051625\\
80	4.85932831037586\\
81	4.86458441816834\\
82	4.84154805775504\\
83	4.84631576231663\\
84	4.84721937222264\\
85	4.8458013115131\\
86	4.84218017264925\\
87	4.84713277748063\\
88	4.83723898611049\\
89	4.84187774364254\\
90	4.84171840375588\\
91	4.8452678398018\\
92	4.84182079956507\\
93	4.84699098516083\\
94	4.84549640548173\\
95	4.85158555409688\\
96	4.84409812448365\\
97	4.84640282767493\\
98	4.84023191773009\\
99	4.83639997556977\\
100	4.83899503618132\\
};
\end{axis}
\end{tikzpicture}%
  }
  \subfloat[$\lambda=0.5$]{
%
%
\definecolor{mycolor1}{rgb}{0.00000,0.75000,0.75000}%
\begin{tikzpicture}

\begin{axis}[%
width=0.951\fwidth,
height=0.75\fwidth,
at={(0\fwidth,0\fwidth)},
scale only axis,
xmode=log,
xmin=1,
xmax=100,
xminorticks=true,
xlabel={t},
ymin=2,
ymax=7,
axis background/.style={fill=white}
]
\addplot [color=blue, line width=2.0pt, forget plot]
  table[row sep=crcr]{%
1	5.43626321740328\\
2	5.62955957449209\\
3	5.83892212841107\\
4	6.08455180148673\\
5	6.45389213646402\\
6	6.51404201900693\\
7	6.55425288332807\\
8	6.47841299626561\\
9	6.62069875641757\\
10	6.51422248577047\\
11	6.59838520559784\\
12	6.59032275016397\\
13	6.52849619086474\\
14	6.68240075733416\\
15	6.70628180193069\\
16	6.70366440054659\\
17	6.70078300565614\\
18	6.70264005866184\\
19	6.70272040838255\\
20	6.70215960026927\\
21	6.7035515874975\\
22	6.70405002626636\\
23	6.70380004965903\\
24	6.70280860477349\\
25	6.70188832548547\\
26	6.70154419712915\\
27	6.70290303559077\\
28	6.70099710415468\\
29	6.70193379267451\\
30	6.70120150981987\\
31	6.70173152738914\\
32	6.70159866600691\\
33	6.70148331010774\\
34	6.70024700158092\\
35	6.70162388648996\\
36	6.70122567452083\\
37	6.70119760380918\\
38	6.7008942387301\\
39	6.70207192686816\\
40	6.70206106174765\\
41	6.70118803780591\\
42	6.7011708537815\\
43	6.70121592271\\
44	6.70090701893929\\
45	6.70090952627623\\
46	6.70172791977211\\
47	6.7004604519536\\
48	6.70214417070081\\
49	6.70087927793523\\
50	6.7009280299354\\
51	6.70192571875951\\
52	6.70171886250275\\
53	6.70150820728044\\
54	6.70126159777429\\
55	6.70041242169488\\
56	6.7002853540649\\
57	6.70046583835858\\
58	6.70020648356859\\
59	6.70116548685459\\
60	6.70050519866534\\
61	6.7015615201992\\
62	6.7007685612289\\
63	6.70095050594716\\
64	6.70154935714809\\
65	6.70152349860255\\
66	6.70070897977991\\
67	6.7009292344296\\
68	6.70068914149036\\
69	6.70163551712463\\
70	6.70142476387294\\
71	6.70069739926064\\
72	6.7017336129983\\
73	6.70128008473107\\
74	6.70162543894496\\
75	6.70146658599244\\
76	6.70170456840915\\
77	6.7013768002497\\
78	6.70182944797324\\
79	6.70116360226985\\
80	6.701438682881\\
81	6.70092732271296\\
82	6.70135911617293\\
83	6.70188062229687\\
84	6.70075474316751\\
85	6.70097329900569\\
86	6.70040857565787\\
87	6.70120199869424\\
88	6.70120379847468\\
89	6.70170383955328\\
90	6.70113156603443\\
91	6.70058910840515\\
92	6.70055443379444\\
93	6.70091509145021\\
94	6.70170584002308\\
95	6.70130045964099\\
96	6.70227846487977\\
97	6.70187435109027\\
98	6.70119790741984\\
99	6.70133993723436\\
100	6.70064902183277\\
};
\addplot [color=black!50!green, dashed, line width=3.0pt, forget plot]
  table[row sep=crcr]{%
1	5.51006868459902\\
2	5.94349925651812\\
3	6.15110188956242\\
4	6.02712399624356\\
5	6.32392569534522\\
6	6.36115851019674\\
7	6.36007362082859\\
8	6.42304614951736\\
9	6.46234178919198\\
10	6.46733988080431\\
11	6.45899027408367\\
12	6.3374431395461\\
13	6.54985803215858\\
14	6.3167027540054\\
15	6.46001002676741\\
16	6.56942958143311\\
17	6.61343634104224\\
18	6.63292515036324\\
19	6.63448250199397\\
20	6.63488106426994\\
21	6.63474315974286\\
22	6.68991621855335\\
23	6.6335990864829\\
24	6.69904685642871\\
25	6.70141875769818\\
26	6.7012806526023\\
27	6.70070437469069\\
28	6.70101301806078\\
29	6.70067654491554\\
30	6.70099908199053\\
31	6.70110335421526\\
32	6.70109454031283\\
33	6.70048011400453\\
34	6.70015802246409\\
35	6.70146803260837\\
36	6.70137519147395\\
37	6.70074615877402\\
38	6.70091163421494\\
39	6.7008946179437\\
40	6.70109518272528\\
41	6.70205145033657\\
42	6.70115273392061\\
43	6.70186718313822\\
44	6.70177380992648\\
45	6.7003804780666\\
46	6.70153952966726\\
47	6.70174797129619\\
48	6.70057765354134\\
49	6.70081922974977\\
50	6.70164333602487\\
51	6.70125515905557\\
52	6.70152742403787\\
53	6.70163646612874\\
54	6.70157301629276\\
55	6.701414269396\\
56	6.70151509569165\\
57	6.70113067220305\\
58	6.70067348960663\\
59	6.70094302840244\\
60	6.70066241369534\\
61	6.70171037678393\\
62	6.7017258902379\\
63	6.70100013577623\\
64	6.70084519318045\\
65	6.70123086292297\\
66	6.70103013736592\\
67	6.70117679885823\\
68	6.70136115786824\\
69	6.70167814502445\\
70	6.70122192741145\\
71	6.70183309433243\\
72	6.70100999249549\\
73	6.70135356970916\\
74	6.7002459876897\\
75	6.70146912361473\\
76	6.70130258398882\\
77	6.70075589014854\\
78	6.70080348452352\\
79	6.70112183391194\\
80	6.70120314765165\\
81	6.70119447031542\\
82	6.70144475842244\\
83	6.70137848453271\\
84	6.70168043308462\\
85	6.70077627802382\\
86	6.70071854612631\\
87	6.70101623736037\\
88	6.70104533420828\\
89	6.70135513010366\\
90	6.70056144078923\\
91	6.70102557689087\\
92	6.7006300362861\\
93	6.70107085890124\\
94	6.70153290457477\\
95	6.70133276766025\\
96	6.70119862303415\\
97	6.70089139040385\\
98	6.70152026042981\\
99	6.70007620915822\\
100	6.70036852816079\\
};
\addplot [color=red, dashdotted, line width=2.0pt, forget plot]
  table[row sep=crcr]{%
1	2.44661771734869\\
2	2.52576353700595\\
3	2.16127647791958\\
4	3.0310290948882\\
5	3.40887425643453\\
6	3.50341814626276\\
7	3.56683307706032\\
8	3.73648862770094\\
9	3.23565565554268\\
10	3.40502086781842\\
11	3.04321944910621\\
12	3.01889669235785\\
13	3.43068866230066\\
14	3.25900115565335\\
15	3.17273095025148\\
16	3.10459827121771\\
17	3.0526650791918\\
18	3.09139487891992\\
19	3.09579296675206\\
20	3.09234915988499\\
21	3.09851430049847\\
22	3.10488187422497\\
23	3.11031833793038\\
24	3.09025546200949\\
25	3.07799538254974\\
26	3.05651836407002\\
27	3.09294115326098\\
28	3.04643784169409\\
29	3.08269460708509\\
30	3.07199875153622\\
31	3.06343319656065\\
32	3.07375106751663\\
33	3.06639676553553\\
34	3.05995126777791\\
35	3.07517841982544\\
36	3.06389299412385\\
37	3.06951122517533\\
38	3.04740118393345\\
39	3.08209796772388\\
40	3.07641205413439\\
41	3.05399563766818\\
42	3.05893438157334\\
43	3.05541018738379\\
44	3.06393384924717\\
45	3.04896819364968\\
46	3.07750178838152\\
47	3.05587980147252\\
48	3.08182658514786\\
49	3.04540555034994\\
50	3.06253296378515\\
51	3.07015540037389\\
52	3.06967956686681\\
53	3.0665229029738\\
54	3.06112004082199\\
55	3.05154570012539\\
56	3.04896094300526\\
57	3.05485326870482\\
58	3.04000064341995\\
59	3.059682219609\\
60	3.06316955979892\\
61	3.06141679688057\\
62	3.04373363846218\\
63	3.06791502934542\\
64	3.0757031610305\\
65	3.07958268848526\\
66	3.05845004880157\\
67	3.06552120542576\\
68	3.0419291877304\\
69	3.06806252792471\\
70	3.06832192190414\\
71	3.05477491131693\\
72	3.06448088366623\\
73	3.04826170050289\\
74	3.06948504914674\\
75	3.06412523565948\\
76	3.07259231360712\\
77	3.06403226648865\\
78	3.06970412697863\\
79	3.06165521805732\\
80	3.0614821477698\\
81	3.052753008387\\
82	3.06648346891204\\
83	3.07741163270134\\
84	3.04639913183469\\
85	3.06273214783606\\
86	3.06596710461126\\
87	3.04983966925001\\
88	3.06539272911671\\
89	3.06856762532641\\
90	3.05256539118647\\
91	3.03999612848695\\
92	3.05343433068084\\
93	3.05168086156649\\
94	3.06665063702698\\
95	3.05274503115123\\
96	3.08200349747352\\
97	3.07318177136839\\
98	3.06615315275989\\
99	3.05534532007844\\
100	3.05738636915808\\
};
\addplot [color=mycolor1, dotted, line width=3.0pt, forget plot]
  table[row sep=crcr]{%
1	5.10841690753339\\
2	6.30366720328818\\
3	6.59816030254869\\
4	6.17790900319201\\
5	4.46475981640557\\
6	5.00765820609084\\
7	4.30951390227002\\
8	4.04161842546795\\
9	4.01485656248218\\
10	4.34312892872975\\
11	3.66541471620045\\
12	3.517951425459\\
13	3.61233274197403\\
14	3.74518751805426\\
15	4.16236954544025\\
16	3.64187637079015\\
17	3.4741715491685\\
18	3.29478144435632\\
19	3.18876107018515\\
20	3.20397854482619\\
21	3.19320600929792\\
22	3.27036325372438\\
23	3.20516039434347\\
24	3.05980848849459\\
25	3.05167931384347\\
26	3.05717743757322\\
27	3.0412954175818\\
28	3.06149776452893\\
29	3.04849873207019\\
30	3.04518142355283\\
31	3.05500058612968\\
32	3.05073964570803\\
33	3.05686069870436\\
34	3.04945434987303\\
35	3.07163211580473\\
36	3.07857793419349\\
37	3.07135457940943\\
38	3.05417963349527\\
39	3.06161284949563\\
40	3.07797428607014\\
41	3.08166389658056\\
42	3.06906824068655\\
43	3.07685730044911\\
44	3.06959642640651\\
45	3.04382768227879\\
46	3.06171377644948\\
47	3.07010403124248\\
48	3.05731202400714\\
49	3.07062597924447\\
50	3.06229126911754\\
51	3.05690352916751\\
52	3.06922983803813\\
53	3.05873407555027\\
54	3.0560740176733\\
55	3.0796699396199\\
56	3.05903847571533\\
57	3.05346016169011\\
58	3.05274905852315\\
59	3.06334635906267\\
60	3.05083683979572\\
61	3.06929869363696\\
62	3.07363736785062\\
63	3.06771725206619\\
64	3.0555037122552\\
65	3.06363688025101\\
66	3.05429921537096\\
67	3.05786906562215\\
68	3.06089932774921\\
69	3.08142837576992\\
70	3.07119110837081\\
71	3.07819963661889\\
72	3.07592939418843\\
73	3.06934163009663\\
74	3.04161092943315\\
75	3.06115524034447\\
76	3.05413353643378\\
77	3.04342239143652\\
78	3.06695751758088\\
79	3.06223131468353\\
80	3.06790529182809\\
81	3.07235109327085\\
82	3.06559871297931\\
83	3.06886571268794\\
84	3.07193206572329\\
85	3.06176523652693\\
86	3.05312708109413\\
87	3.05450888405249\\
88	3.0615663512455\\
89	3.05752676840432\\
90	3.05975652143678\\
91	3.07211379467446\\
92	3.0525903118369\\
93	3.04773333165863\\
94	3.06121170064906\\
95	3.06467535806715\\
96	3.07540315101183\\
97	3.06243086424018\\
98	3.06862373610331\\
99	3.04294148478522\\
100	3.05839878495847\\
};
\end{axis}
\end{tikzpicture}%
  }
  \subfloat[$\lambda=0.75$]{
%
%
\definecolor{mycolor1}{rgb}{0.00000,0.75000,0.75000}%
\begin{tikzpicture}

\begin{axis}[%
width=0.951\fwidth,
height=0.75\fwidth,
at={(0\fwidth,0\fwidth)},
scale only axis,
xmode=log,
xmin=1,
xmax=100,
xminorticks=true,
xlabel={t},
ymin=0,
ymax=7,
axis background/.style={fill=white}
]
\addplot [color=blue, line width=2.0pt, forget plot]
  table[row sep=crcr]{%
1	5.47461127098968\\
2	5.60478890856896\\
3	5.77191033058464\\
4	5.90011496868576\\
5	6.14377790957163\\
6	6.31041948276384\\
7	6.43552914365433\\
8	6.5151531251348\\
9	6.54729734736349\\
10	6.57798495912498\\
11	6.57849571632306\\
12	6.58765985244295\\
13	6.58900819766999\\
14	6.58807470034135\\
15	6.58544669469427\\
16	6.58533577780727\\
17	6.58370067428244\\
18	6.58971365581165\\
19	6.5878980368345\\
20	6.59048092832187\\
21	6.59125044154283\\
22	6.5896710150567\\
23	6.58987286386257\\
24	6.58946117658065\\
25	6.58999333445314\\
26	6.59031718167873\\
27	6.59139413998349\\
28	6.59090165755066\\
29	6.59051023277067\\
30	6.59219941170473\\
31	6.59124669937017\\
32	6.58902226344559\\
33	6.58930522013265\\
34	6.59156801046039\\
35	6.59140118319131\\
36	6.59071800018758\\
37	6.5904962121109\\
38	6.58994289370681\\
39	6.59067144643273\\
40	6.5908473033638\\
41	6.58935869164842\\
42	6.59079752619829\\
43	6.59045257663437\\
44	6.59038156874226\\
45	6.58922214845872\\
46	6.5920430806679\\
47	6.59061552462121\\
48	6.5910853518908\\
49	6.5904707507656\\
50	6.59147532593014\\
51	6.59088905580135\\
52	6.59228441522875\\
53	6.59101992136476\\
54	6.59064064465195\\
55	6.59183783311411\\
56	6.59023817047242\\
57	6.59061017708232\\
58	6.5906193777462\\
59	6.59140530114315\\
60	6.5912579079603\\
61	6.59103059374945\\
62	6.59073429065668\\
63	6.5902464349662\\
64	6.59060596841733\\
65	6.59015307164583\\
66	6.59133709302133\\
67	6.59065526384808\\
68	6.58958255899406\\
69	6.59157679882054\\
70	6.59033990286414\\
71	6.59149880127007\\
72	6.59081749662089\\
73	6.58978538380844\\
74	6.59092637473552\\
75	6.58927529131999\\
76	6.58850343746042\\
77	6.59070961917545\\
78	6.59053467074177\\
79	6.58991848266028\\
80	6.59074090596838\\
81	6.59116797494771\\
82	6.59120811456219\\
83	6.58951958489378\\
84	6.59076403901614\\
85	6.58992010133594\\
86	6.59063253300245\\
87	6.59074681389064\\
88	6.59131520405029\\
89	6.58979941235916\\
90	6.58994888558284\\
91	6.59168176024839\\
92	6.59066396864525\\
93	6.59105539706456\\
94	6.59127950787074\\
95	6.59064023962094\\
96	6.59008138513551\\
97	6.59140183039189\\
98	6.58838036017117\\
99	6.5909343332388\\
100	6.59170125491083\\
};
\addplot [color=black!50!green, dashed, line width=3.0pt, forget plot]
  table[row sep=crcr]{%
1	5.62056595098105\\
2	5.9249699465451\\
3	6.01582827122022\\
4	5.95571514939206\\
5	5.92294247321893\\
6	5.81452251942304\\
7	6.13426812770014\\
8	6.04595676543718\\
9	6.29227294231064\\
10	6.27799612179181\\
11	6.32051997859251\\
12	6.36527532407389\\
13	6.32925991780952\\
14	6.45465237806581\\
15	6.58223240737318\\
16	6.57987442531153\\
17	6.55874072799058\\
18	6.5753041864462\\
19	6.58119554829056\\
20	6.58728407958621\\
21	6.57809056851983\\
22	6.56764059133101\\
23	6.41013588735602\\
24	6.38917378777253\\
25	6.58721628658142\\
26	6.58203616141645\\
27	6.4153585168044\\
28	6.56214565465418\\
29	6.42451672084935\\
30	6.57980421789207\\
31	6.59042602754726\\
32	6.59044529564649\\
33	6.59080569187187\\
34	6.59089611744954\\
35	6.58916534077888\\
36	6.59030181816158\\
37	6.59088695127249\\
38	6.58937832439878\\
39	6.59148671433531\\
40	6.59073721108426\\
41	6.59093404750414\\
42	6.59087742004054\\
43	6.58928541970452\\
44	6.59058564316579\\
45	6.59132828384597\\
46	6.59058851035584\\
47	6.59110263530547\\
48	6.59147553116448\\
49	6.58668393005854\\
50	6.59062363782734\\
51	6.59101515662169\\
52	6.59176251951947\\
53	6.58857868597947\\
54	6.59013423143927\\
55	6.59111909588808\\
56	6.5911457483472\\
57	6.59003852850639\\
58	6.59143982256265\\
59	6.59048536674524\\
60	6.59039176821099\\
61	6.59137154542976\\
62	6.5912535226034\\
63	6.58981582043099\\
64	6.5908918306807\\
65	6.59083300972301\\
66	6.59079121501016\\
67	6.59030883873119\\
68	6.59129857301594\\
69	6.59128229744416\\
70	6.59095332658505\\
71	6.59023081358638\\
72	6.59081865364851\\
73	6.59094026619255\\
74	6.59040142378658\\
75	6.59163684757572\\
76	6.59110676104285\\
77	6.59119701707623\\
78	6.58989897914213\\
79	6.5918191325617\\
80	6.59000392128006\\
81	6.59005624567473\\
82	6.59175293524183\\
83	6.5908999928299\\
84	6.59095496039431\\
85	6.58972726779356\\
86	6.59027090444657\\
87	6.59089103275845\\
88	6.59083663036135\\
89	6.59079717575128\\
90	6.59042755723988\\
91	6.59032670713473\\
92	6.59144685236426\\
93	6.5909022764905\\
94	6.59173282877921\\
95	6.58975668648628\\
96	6.58944811703204\\
97	6.5909704505714\\
98	6.58887279368135\\
99	6.59056883095371\\
100	6.59080516986559\\
};
\addplot [color=red, dashdotted, line width=2.0pt, forget plot]
  table[row sep=crcr]{%
1	0.962011577752915\\
2	1.39512833468571\\
3	1.90148489293405\\
4	1.51648055618043\\
5	2.15248666521931\\
6	1.55704328138457\\
7	1.66059810154288\\
8	1.52166767379182\\
9	1.60756285452165\\
10	1.57032375887657\\
11	1.57143632706561\\
12	1.54570012712595\\
13	1.53062124352496\\
14	1.54268502801898\\
15	1.53021067085092\\
16	1.55206916663035\\
17	1.54627922395726\\
18	1.51390174557473\\
19	1.50375418924415\\
20	1.51530747681262\\
21	1.52196737407071\\
22	1.50354898705708\\
23	1.51332011123504\\
24	1.52232901674692\\
25	1.52509841996306\\
26	1.52238627498701\\
27	1.5324730708447\\
28	1.51163420208493\\
29	1.51812829528253\\
30	1.52774156845851\\
31	1.52183786551656\\
32	1.5016521070917\\
33	1.49397777104937\\
34	1.51952082470017\\
35	1.52980925719026\\
36	1.52606705237248\\
37	1.52515602796127\\
38	1.50973845097226\\
39	1.52151282867688\\
40	1.53441438633646\\
41	1.52415222096287\\
42	1.51132177031528\\
43	1.5246991042189\\
44	1.50540365539768\\
45	1.50579300125202\\
46	1.5332974431763\\
47	1.52602816377732\\
48	1.52498943304033\\
49	1.52505167508116\\
50	1.52936828998447\\
51	1.52360673777914\\
52	1.53058250934036\\
53	1.52357239030125\\
54	1.52025395181835\\
55	1.53499050028086\\
56	1.51742772358661\\
57	1.51441659560209\\
58	1.51420820595266\\
59	1.5262997058346\\
60	1.53616242827644\\
61	1.52204501826377\\
62	1.51266612225863\\
63	1.5190132195052\\
64	1.52431344447975\\
65	1.51530419794324\\
66	1.52201531315481\\
67	1.51061899113799\\
68	1.51304669429054\\
69	1.52532175891914\\
70	1.53260170132874\\
71	1.52225905085154\\
72	1.50963037140345\\
73	1.52293726685294\\
74	1.52308518163793\\
75	1.4950927590338\\
76	1.50963368680608\\
77	1.51272511429317\\
78	1.51042633785373\\
79	1.50938242359903\\
80	1.5271682127886\\
81	1.5117185420591\\
82	1.52766007832391\\
83	1.51122470668482\\
84	1.51941129140273\\
85	1.51728704039085\\
86	1.52688068863037\\
87	1.50792216758675\\
88	1.51772700059157\\
89	1.51734457771922\\
90	1.51448813532051\\
91	1.53151859073494\\
92	1.52116535307287\\
93	1.51578575724323\\
94	1.52090562538418\\
95	1.51986831699292\\
96	1.50477239990879\\
97	1.52614502228845\\
98	1.4935503656753\\
99	1.51752656157145\\
100	1.53496968961844\\
};
\addplot [color=mycolor1, dotted, line width=3.0pt, forget plot]
  table[row sep=crcr]{%
1	4.6864935622248\\
2	4.8994383088347\\
3	4.56109313648738\\
4	4.90229035547094\\
5	2.9314300803944\\
6	3.2784332597462\\
7	2.46963199592331\\
8	2.30125152271061\\
9	2.01471564245944\\
10	2.17712372051914\\
11	1.78625426786003\\
12	1.85353123296186\\
13	1.99614342802232\\
14	1.65287461428729\\
15	1.49633284810289\\
16	1.48610098565703\\
17	1.45156396181048\\
18	1.47406316737917\\
19	1.46740662240884\\
20	1.50873895369633\\
21	1.45078965236607\\
22	1.45245977605543\\
23	1.71819478818221\\
24	2.08550862826438\\
25	1.5059355521904\\
26	1.47394259570511\\
27	1.75330416344819\\
28	1.38731873977995\\
29	1.70173662068909\\
30	1.46762136054097\\
31	1.50989567888606\\
32	1.52057424298062\\
33	1.52296770227207\\
34	1.51890013441686\\
35	1.51167634567613\\
36	1.52013496455013\\
37	1.53042250209518\\
38	1.51614460001237\\
39	1.51068803218364\\
40	1.52349879587756\\
41	1.5292094373918\\
42	1.52723516835788\\
43	1.50219367701738\\
44	1.50889764085272\\
45	1.53357949409282\\
46	1.5063426671682\\
47	1.52330362351915\\
48	1.52806722140564\\
49	1.48492641984674\\
50	1.50496302104093\\
51	1.51726658029489\\
52	1.51662375013351\\
53	1.48434709843271\\
54	1.51950876152417\\
55	1.51591219634932\\
56	1.52244113525343\\
57	1.51228811211186\\
58	1.51907869642937\\
59	1.52153657235561\\
60	1.52002671698928\\
61	1.52644490901635\\
62	1.51264572580944\\
63	1.5160114727441\\
64	1.52871982636947\\
65	1.52455740252168\\
66	1.52881396551966\\
67	1.52041366900985\\
68	1.51920423367879\\
69	1.52572400377125\\
70	1.52896228076713\\
71	1.50694486971597\\
72	1.5191129634431\\
73	1.53589170731109\\
74	1.51554168107362\\
75	1.52114588847516\\
76	1.5073119787335\\
77	1.53177729643885\\
78	1.49938469465641\\
79	1.52493589459921\\
80	1.51681359788873\\
81	1.50273815004913\\
82	1.53692226281642\\
83	1.51074892880607\\
84	1.53056388532285\\
85	1.52037492759945\\
86	1.51626804395327\\
87	1.5065737477751\\
88	1.52211768630356\\
89	1.52115051067854\\
90	1.521429717609\\
91	1.50829188644711\\
92	1.53568509407678\\
93	1.51780971278461\\
94	1.5267654175131\\
95	1.51765382436885\\
96	1.51725665471583\\
97	1.51173138823043\\
98	1.50648110697348\\
99	1.50703573590319\\
100	1.5100847459919\\
};
\end{axis}
\end{tikzpicture}%
  }
  \subfloat[legend]{
    \raisebox{4em}{\definecolor{mycolor1}{rgb}{0.00000,0.75000,0.75000}%

\begin{tikzpicture}

  \begin{axis}[%
    hide axis,
    xmin=10,
    xmax=50,
    ymin=0,
    ymax=0.4,
    legend style={draw=white!15!black,legend cell align=left}
    ]
    \addlegendimage{color=blue, line width=2.0pt}
    \addlegendentry{Bayes $U$};
    \addlegendimage{color=black!50!green, dashed, line width=3.0pt};
    \addlegendentry{Marginal $U$};
    \addlegendimage{color=red, line width=2.0pt, dashdotted}
    \addlegendentry{Bayes $F$};
    \addlegendimage{color=mycolor1, dotted, line width=3.0pt};
    \addlegendentry{Marginal $F$};
    
  \end{axis}
\end{tikzpicture}}
  }
  \caption{\textbf{Synthetic data, utility-fairness trade-off.} This plot is generated from the same data as Figure~\ref{fig_exp_1}. However, now we are plotting the utility and fairness of each individual policy separately. In all cases, it can be seen that the Bayesian policy achieves the same utility as the non-Bayesian policy, while achieving a lower fairness violation.}
  \label{fig_exp_1:tradeoff}
\end{figure*}

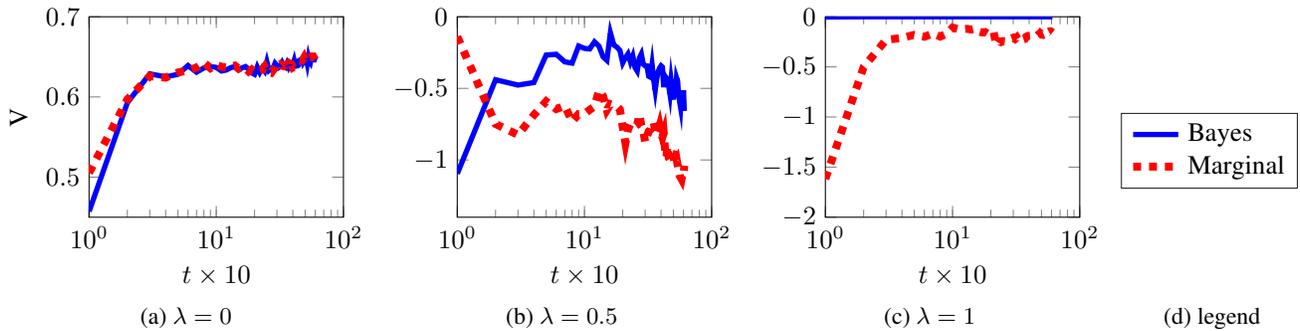
\begin{figure*}
\centering
  \subfloat[$\lambda=0$]{
%
%
\begin{tikzpicture}

\begin{axis}[%
width=0.951\fwidth,
height=0.75\fwidth,
at={(0\fwidth,0\fwidth)},
scale only axis,
xmode=log,
xmin=1,
xmax=100,
xminorticks=true,
xlabel={$t \times 10$},
ymin=0.45,
ymax=0.7,
ylabel={V},
axis background/.style={fill=white}
]
\addplot [color=blue, line width=2.0pt, forget plot]
  table[row sep=crcr]{%
1	0.457219642968273\\
2	0.592769351348169\\
3	0.628816705720447\\
4	0.62505185351612\\
5	0.628217785699876\\
6	0.639218851111152\\
7	0.630873927371614\\
8	0.636227797683053\\
9	0.638026015044111\\
10	0.636823616047142\\
11	0.632410660316369\\
12	0.633597617267562\\
13	0.636527784056834\\
14	0.637892469025015\\
15	0.633703787028147\\
16	0.636328078671473\\
17	0.637880712599377\\
18	0.634684165215899\\
19	0.629978414138595\\
20	0.628159353629191\\
21	0.637486382196772\\
22	0.630647724209231\\
23	0.636203699395604\\
24	0.632324272807024\\
25	0.643830720098772\\
26	0.633034238724461\\
27	0.635123838924277\\
28	0.636057361395139\\
29	0.642104432395723\\
30	0.63957063775151\\
31	0.636676944667785\\
32	0.638328034343207\\
33	0.642713754763223\\
34	0.638813995795796\\
35	0.637222326954388\\
36	0.641876694490855\\
37	0.637397077631842\\
38	0.641452323880008\\
39	0.639449683802413\\
40	0.63998585281562\\
41	0.645562739043969\\
42	0.638373912813148\\
43	0.63869610092111\\
44	0.639248573866283\\
45	0.64553183974159\\
46	0.64574217834242\\
47	0.647159828876256\\
48	0.642734787921472\\
49	0.645351749272455\\
50	0.647986777824105\\
51	0.648498084835137\\
52	0.641666027891812\\
53	0.649788940171897\\
54	0.644719769480268\\
55	0.642816500136809\\
56	0.648563504192992\\
57	0.647484270514486\\
58	0.646522917641747\\
59	0.646362741362458\\
60	0.64943324763383\\
61	0.650725119238089\\
};
\addplot [color=red, dashed, line width=3.0pt, forget plot]
  table[row sep=crcr]{%
1	0.50533782292246\\
2	0.59699112902373\\
3	0.625888847773672\\
4	0.624084738953654\\
5	0.630607004491168\\
6	0.632831760088705\\
7	0.636757087308418\\
8	0.643065048813731\\
9	0.635831074924727\\
10	0.640172268918583\\
11	0.637618500393285\\
12	0.637063543459606\\
13	0.637801195651245\\
14	0.639132776275881\\
15	0.635237491522697\\
16	0.63711453971561\\
17	0.635829219303754\\
18	0.631722475647075\\
19	0.634325927457996\\
20	0.631533283036592\\
21	0.634408642218423\\
22	0.633190718318277\\
23	0.634003706915972\\
24	0.637844298466811\\
25	0.632637216101121\\
26	0.630133780104625\\
27	0.629845209068525\\
28	0.633468263655565\\
29	0.636356158428604\\
30	0.641203029255701\\
31	0.638991861329888\\
32	0.640709523818283\\
33	0.638816379513034\\
34	0.63827201778754\\
35	0.639309083888535\\
36	0.640322622034011\\
37	0.636038032575251\\
38	0.641624574388145\\
39	0.640674600649206\\
40	0.636610965813643\\
41	0.637518273518509\\
42	0.637490556877938\\
43	0.642085811731694\\
44	0.642900366689167\\
45	0.64514581492819\\
46	0.644982417065964\\
47	0.645169810588366\\
48	0.646840576097694\\
49	0.649222992097267\\
50	0.645768342344984\\
51	0.646975527164982\\
52	0.648342200211879\\
53	0.646810289670806\\
54	0.648718371906363\\
55	0.649127848411779\\
56	0.650500132527732\\
57	0.650312507421561\\
58	0.650414667141449\\
59	0.648262622639421\\
60	0.646204756681934\\
61	0.64655970208432\\
};
\end{axis}
\end{tikzpicture}%
  }
  \subfloat[$\lambda=0.5$]{
%
%
\begin{tikzpicture}

\begin{axis}[%
width=0.951\fwidth,
height=0.75\fwidth,
at={(0\fwidth,0\fwidth)},
scale only axis,
xmode=log,
xmin=1,
xmax=100,
xminorticks=true,
xlabel={$t \times 10$},
ymin=-1.4,
ymax=0,
axis background/.style={fill=white}
]
\addplot [color=blue, line width=2.0pt, forget plot]
  table[row sep=crcr]{%
1	-1.09622117340752\\
2	-0.439738831631238\\
3	-0.477010262732245\\
4	-0.459566994825697\\
5	-0.267438768251134\\
6	-0.261884198822098\\
7	-0.314800806678737\\
8	-0.323934739140698\\
9	-0.204202454295437\\
10	-0.220866546747302\\
11	-0.224784320984486\\
12	-0.174850401323677\\
13	-0.204336810453115\\
14	-0.279220020694667\\
15	-0.304215110722433\\
16	-0.111899168962922\\
17	-0.196698433914746\\
18	-0.212212607848973\\
19	-0.234718297036081\\
20	-0.322668180703975\\
21	-0.286071421580403\\
22	-0.259499649273306\\
23	-0.316375891732036\\
24	-0.273840483206877\\
25	-0.355624719979204\\
26	-0.379757992540423\\
27	-0.288379194513057\\
28	-0.259927407512447\\
29	-0.331340818267421\\
30	-0.342494252379114\\
31	-0.343352143939611\\
32	-0.415313928296665\\
33	-0.422017029459941\\
34	-0.335179784452087\\
35	-0.440060447624601\\
36	-0.514029990888593\\
37	-0.43325890065029\\
38	-0.381527909556519\\
39	-0.357585870240933\\
40	-0.342513238337931\\
41	-0.351232807159593\\
42	-0.491960446194131\\
43	-0.414521352705201\\
44	-0.426725954088604\\
45	-0.442441956965056\\
46	-0.48463572780352\\
47	-0.539087119210198\\
48	-0.444488011314339\\
49	-0.50689510900555\\
50	-0.577616030570053\\
51	-0.438450894005509\\
52	-0.538629948619004\\
53	-0.427097161747581\\
54	-0.449574311025686\\
55	-0.549606168957645\\
56	-0.592101430233288\\
57	-0.633608692550753\\
58	-0.517199531938076\\
59	-0.606102743738442\\
60	-0.513954328614164\\
61	-0.657646711598371\\
};
\addplot [color=red, dashed, line width=3.0pt, forget plot]
  table[row sep=crcr]{%
1	-0.138097439995355\\
2	-0.748893406957869\\
3	-0.82460686564522\\
4	-0.686300787106743\\
5	-0.590927424830682\\
6	-0.644914108373764\\
7	-0.621536393076511\\
8	-0.732227183784745\\
9	-0.678985832465308\\
10	-0.677538882070215\\
11	-0.649492934072179\\
12	-0.626953188541119\\
13	-0.514430272564223\\
14	-0.552610078526062\\
15	-0.651427496731807\\
16	-0.610373827755142\\
17	-0.618730957216695\\
18	-0.62662953336171\\
19	-0.647029602765731\\
20	-0.771071949704786\\
21	-0.881634832338758\\
22	-0.787624432290519\\
23	-0.749027999750425\\
24	-0.768887901516103\\
25	-0.766265274524894\\
26	-0.754819034015115\\
27	-0.682709696959525\\
28	-0.755854713298122\\
29	-0.795893759226808\\
30	-0.87634212964454\\
31	-0.837761717305747\\
32	-0.800158278551069\\
33	-0.788079320268578\\
34	-0.757587585669062\\
35	-0.735977528675732\\
36	-0.792743783163313\\
37	-0.826465829671237\\
38	-0.79949075466688\\
39	-0.763326650871933\\
40	-0.812995026125147\\
41	-0.841258496638301\\
42	-0.811392209795322\\
43	-0.925172456274953\\
44	-0.981952927632142\\
45	-0.903004676499468\\
46	-0.985923064692685\\
47	-1.00004897431729\\
48	-0.981936952406439\\
49	-0.955363797496084\\
50	-0.95804465183494\\
51	-1.02561492499815\\
52	-1.00549051871228\\
53	-0.993986917287425\\
54	-1.03469402307911\\
55	-1.0713271054811\\
56	-1.11079025820582\\
57	-1.12157352569169\\
58	-1.13388551393297\\
59	-1.10161020503734\\
60	-1.10864470943151\\
61	-1.03972034327087\\
};
\end{axis}
\end{tikzpicture}%
  }
  \subfloat[$\lambda=1$]{
%
%
\begin{tikzpicture}

\begin{axis}[%
width=0.951\fwidth,
height=0.75\fwidth,
at={(0\fwidth,0\fwidth)},
scale only axis,
xmode=log,
xmin=1,
xmax=100,
xminorticks=true,
xlabel={$t \times 10$},
ymin=-2,
ymax=-0,
axis background/.style={fill=white}
]
\addplot [color=blue, line width=2.0pt, forget plot]
  table[row sep=crcr]{%
1	-0.00232770753700901\\
2	-1.97313695873102e-05\\
3	-2.37767934892516e-06\\
4	-5.41730805759902e-07\\
5	-7.32149168349694e-08\\
6	-2.39811834130854e-08\\
7	-5.18082968701592e-09\\
8	-5.80788316935145e-09\\
9	-1.76256829256548e-09\\
10	-1.38567073595432e-09\\
11	-2.7766696901406e-10\\
12	-2.3480673996106e-10\\
13	-8.92285754128819e-11\\
14	-1.19961233710043e-10\\
15	-2.67771625177127e-11\\
16	-2.1807283694217e-11\\
17	-9.888893758119e-12\\
18	-2.04587275550225e-11\\
19	-1.00294880288958e-11\\
20	-7.27273007814714e-12\\
21	-2.18409079376077e-12\\
22	-1.90552725474449e-12\\
23	-3.08552681136802e-12\\
24	-1.04407443807789e-12\\
25	-6.89534351010657e-13\\
26	-7.04452116553748e-13\\
27	-5.33265371320719e-13\\
28	-2.21733300835158e-13\\
29	-1.09917521624577e-13\\
30	-1.97092353458009e-13\\
31	-1.06901072975012e-13\\
32	-1.02563119096526e-13\\
33	-8.84436177317179e-14\\
34	-3.79100124163196e-14\\
35	-4.2067887937253e-14\\
36	-2.70027547549659e-14\\
37	-2.66635307650848e-14\\
38	-2.67287785600447e-14\\
39	-2.64342999326352e-14\\
40	-2.65179986462852e-14\\
41	-2.68229872585042e-14\\
42	-2.66997278710527e-14\\
43	-2.69707767201082e-14\\
44	-2.67044068810534e-14\\
45	-2.67210612428623e-14\\
46	-2.65715988911875e-14\\
47	-2.65002787170287e-14\\
48	-2.64873004782108e-14\\
49	-2.61218989350948e-14\\
50	-2.64084264642284e-14\\
51	-2.65721748735916e-14\\
52	-2.66424481150276e-14\\
53	-2.65125945760817e-14\\
54	-2.61390327175519e-14\\
55	-2.63298827930954e-14\\
56	-2.66575354658841e-14\\
57	-2.65806977191069e-14\\
58	-2.63970270948243e-14\\
59	-2.63521868646625e-14\\
60	-2.62708242678811e-14\\
61	-2.67337836777301e-14\\
};
\addplot [color=red, dashed, line width=3.0pt, forget plot]
  table[row sep=crcr]{%
1	-1.61749222996759\\
2	-0.501016985730073\\
3	-0.230141055920469\\
4	-0.211363679210412\\
5	-0.180351722603156\\
6	-0.199490315348852\\
7	-0.167503225130075\\
8	-0.197680427138255\\
9	-0.168326964955281\\
10	-0.108154945624742\\
11	-0.121796007045675\\
12	-0.123244092080481\\
13	-0.120059074298961\\
14	-0.12768082663036\\
15	-0.118695747205944\\
16	-0.14620525838291\\
17	-0.156154755993665\\
18	-0.15210970699618\\
19	-0.171018095781681\\
20	-0.159637204861449\\
21	-0.147168529193469\\
22	-0.282154360398568\\
23	-0.245469097598784\\
24	-0.252947839856794\\
25	-0.242482705863201\\
26	-0.244915288625687\\
27	-0.224359689620484\\
28	-0.214533978554422\\
29	-0.216114585568234\\
30	-0.208860180236654\\
31	-0.221529291908727\\
32	-0.205563531932023\\
33	-0.233055721377914\\
34	-0.23298981535382\\
35	-0.215062245554676\\
36	-0.226954153936872\\
37	-0.222039885512241\\
38	-0.198543948266946\\
39	-0.191065849011378\\
40	-0.189777971418282\\
41	-0.188793398961995\\
42	-0.190732460573578\\
43	-0.185246409681183\\
44	-0.184603807555449\\
45	-0.188647955865062\\
46	-0.191033781025243\\
47	-0.179745728112148\\
48	-0.178803266307534\\
49	-0.163926246089407\\
50	-0.173423292392264\\
51	-0.179302094659936\\
52	-0.181333239906458\\
53	-0.177046092214897\\
54	-0.184951112414626\\
55	-0.181338109208384\\
56	-0.147815635344293\\
57	-0.149881515239679\\
58	-0.14463523042435\\
59	-0.144488753661275\\
60	-0.150296785419916\\
61	-0.157708726945062\\
};
\end{axis}
\end{tikzpicture}%
  }
  \subfloat[legend]{
    \raisebox{4em}{\begin{tikzpicture}

  \begin{axis}[%
    hide axis,
    xmin=10,
    xmax=50,
    ymin=0,
    ymax=0.4,
    legend style={draw=white!15!black,legend cell align=left}
    ]
    \addlegendimage{color=mycolor1, line width=2.0pt}
    \addlegendentry{Bayes};
    \addlegendimage{color=mycolor2, dashed, line width=3.0pt};
    \addlegendentry{Marginal};
  \end{axis}
\end{tikzpicture}}
  }

  \caption{\textbf{COMPAS dataset.} Demonstration of balance on the COMPAS dataset. The plots show the value measured on the holdout set for the \textbf{Bayes} and \textbf{Marginal} balance.
  Figures (a-c) show the utility achieved under different choices of $\lambda$ as we we observe each of the  6,000 training data points. Utility and fairness are measured on the empirical distribution of the remaining data and it can be seen that the Bayesian approach dominates as soon as fairness becomes important, i.e. $\lambda > 0$.  }
  \label{fig:compas-dbn}
\end{figure*}

\begin{figure*}
\centering
  \subfloat[$\lambda=0$]{
%
%
\begin{tikzpicture}

\begin{axis}[%
width=0.951\fwidth,
height=0.75\fwidth,
at={(0\fwidth,0\fwidth)},
scale only axis,
xmode=log,
xmin=1,
xmax=100,
xminorticks=true,
xlabel style={font=\color{white!15!black}},
xlabel={t},
ymin=0.5,
ymax=0.64,
ylabel style={font=\color{white!15!black}},
ylabel={V},
axis background/.style={fill=white}
]
\addplot [color=blue, line width=2.0pt, forget plot]
  table[row sep=crcr]{%
1	0.533404412801252\\
2	0.567059367237681\\
3	0.574983684728794\\
4	0.595440941591605\\
5	0.597329784902714\\
6	0.598376899476387\\
7	0.608381899604062\\
8	0.614299903055019\\
9	0.614682328515722\\
10	0.622047920597586\\
11	0.624110283420211\\
12	0.623040076545587\\
13	0.622922815525939\\
14	0.622693703044924\\
15	0.624972493335543\\
16	0.623768267741923\\
17	0.625076223581217\\
18	0.624417108492841\\
19	0.624469368091222\\
20	0.626122024040929\\
21	0.627990230164747\\
22	0.627328710101433\\
23	0.627229984116765\\
24	0.627284500665436\\
25	0.628167660760683\\
26	0.62745907557843\\
27	0.627233400174376\\
28	0.626561170176999\\
29	0.626140421275828\\
30	0.627206627096838\\
31	0.626985313402765\\
32	0.627098543124935\\
33	0.626653517790801\\
34	0.62669323165553\\
35	0.626164136918594\\
36	0.625226096500675\\
37	0.624375614738648\\
38	0.626043052672986\\
39	0.625037401619658\\
40	0.624754173295781\\
41	0.625734326489491\\
42	0.624858959455676\\
43	0.624748862220927\\
44	0.625474170196668\\
45	0.624772055502031\\
46	0.624434789925371\\
47	0.624134265867323\\
48	0.624465786011358\\
49	0.62488377850369\\
50	0.624799243863898\\
51	0.623188769161788\\
52	0.622787015457619\\
53	0.625093107691527\\
54	0.62416250160529\\
55	0.624845855606283\\
56	0.625039701294087\\
57	0.625218860435963\\
58	0.626623043783051\\
59	0.626765171689463\\
60	0.626873002508487\\
};
\addplot [color=red, dashed, line width=3.0pt, forget plot]
  table[row sep=crcr]{%
1	0.537691550449067\\
2	0.52692425944197\\
3	0.51952806147473\\
4	0.516060917995437\\
5	0.525734464151558\\
6	0.526857982759685\\
7	0.527688486772251\\
8	0.553006242367054\\
9	0.580639537966924\\
10	0.583055435798432\\
11	0.583845485337254\\
12	0.585535021854244\\
13	0.589716220033718\\
14	0.591727972382982\\
15	0.592321109916687\\
16	0.593212336283605\\
17	0.59331788850473\\
18	0.592639117087132\\
19	0.594276492899578\\
20	0.593631864589812\\
21	0.593134373439542\\
22	0.592526975573227\\
23	0.59343823978082\\
24	0.59471503376963\\
25	0.594382031634667\\
26	0.59347186036185\\
27	0.594295880998185\\
28	0.592974355499971\\
29	0.594610054709994\\
30	0.594316274450942\\
31	0.59587687254927\\
32	0.594384171167264\\
33	0.59368169358206\\
34	0.593672061191477\\
35	0.59309575835021\\
36	0.591830149985295\\
37	0.592248365155305\\
38	0.592785655298101\\
39	0.592284775003569\\
40	0.591763395381349\\
41	0.591416781456517\\
42	0.593000957180953\\
43	0.591369641506618\\
44	0.592580298295271\\
45	0.592677941655292\\
46	0.591707641469623\\
47	0.592929785718232\\
48	0.591688464121018\\
49	0.591343754331541\\
50	0.589724575501398\\
51	0.590358577218326\\
52	0.590327436945083\\
53	0.590665615799822\\
54	0.595774337247164\\
55	0.595592315397091\\
56	0.595804057900644\\
57	0.59578874458756\\
58	0.596457066440588\\
59	0.597530717491227\\
60	0.597648444447298\\
};
\end{axis}
\end{tikzpicture}%
  }
  \subfloat[$\lambda=0.5$]{
%
%
\begin{tikzpicture}

\begin{axis}[%
width=0.951\fwidth,
height=0.75\fwidth,
at={(0\fwidth,0\fwidth)},
scale only axis,
xmode=log,
xmin=1,
xmax=100,
xminorticks=true,
xlabel style={font=\color{white!15!black}},
xlabel={t},
ymin=-2,
ymax=-0,
axis background/.style={fill=white}
]
\addplot [color=blue, line width=2.0pt, forget plot]
  table[row sep=crcr]{%
1	-0.139557730363354\\
2	-0.238920667169255\\
3	-0.0997916326764938\\
4	-0.182229548458017\\
5	-0.205267446075831\\
6	-0.279826712648142\\
7	-0.332249305176689\\
8	-0.323237535828966\\
9	-0.334848601807552\\
10	-0.312990041492882\\
11	-0.324619975130561\\
12	-0.328870423071226\\
13	-0.313326571427902\\
14	-0.319808076558834\\
15	-0.30709006893331\\
16	-0.317914410723469\\
17	-0.331045983578041\\
18	-0.342385254226511\\
19	-0.358246869022327\\
20	-0.369248314322897\\
21	-0.377973473224499\\
22	-0.415440383538921\\
23	-0.38990705668132\\
24	-0.36686729244785\\
25	-0.375042658748737\\
26	-0.366006828119552\\
27	-0.383213880975782\\
28	-0.386114407046102\\
29	-0.380250046047834\\
30	-0.399759884881641\\
31	-0.437509783568944\\
32	-0.439357807489347\\
33	-0.443493819937183\\
34	-0.43823346019647\\
35	-0.448743781654721\\
36	-0.46057428729186\\
37	-0.465056676021326\\
38	-0.455095054216753\\
39	-0.463602980734887\\
40	-0.497134299702848\\
41	-0.456405475996486\\
42	-0.461607953659762\\
43	-0.467431543238734\\
44	-0.481699178907957\\
45	-0.497488409531242\\
46	-0.491566651303831\\
47	-0.489462737779279\\
48	-0.508962806633321\\
49	-0.497498682067592\\
50	-0.49354033957002\\
51	-0.497912056057912\\
52	-0.496093017678467\\
53	-0.505412512936155\\
54	-0.503382028933859\\
55	-0.52416214442533\\
56	-0.522603926762207\\
57	-0.544550008258789\\
58	-0.532153628446791\\
59	-0.520440892587704\\
60	-0.532834541878595\\
};
\addplot [color=red, dashed, line width=3.0pt, forget plot]
  table[row sep=crcr]{%
1	-0.759593109224573\\
2	-0.851910476827388\\
3	-0.931248398068395\\
4	-0.899099878562075\\
5	-0.7313280766324\\
6	-0.866207448068707\\
7	-0.990998285997023\\
8	-1.0272085536674\\
9	-1.04372995832914\\
10	-1.07714261034557\\
11	-1.07998626920429\\
12	-1.1166916978057\\
13	-1.09921640121892\\
14	-1.11870116207787\\
15	-1.14149626270435\\
16	-1.13602637450519\\
17	-1.19256259507435\\
18	-1.18855636890274\\
19	-1.16210813506113\\
20	-1.19312997954371\\
21	-1.16517236023971\\
22	-1.09072259193678\\
23	-1.07693553149219\\
24	-1.09320599058905\\
25	-1.12532578022789\\
26	-1.14329006175479\\
27	-1.17276265772882\\
28	-1.15208234230952\\
29	-1.18083392582687\\
30	-1.20703772222562\\
31	-1.22005481633382\\
32	-1.21874241865254\\
33	-1.26208489068812\\
34	-1.22937949908594\\
35	-1.26117480001724\\
36	-1.2690068818918\\
37	-1.26700389402755\\
38	-1.27433797411748\\
39	-1.28632957244659\\
40	-1.31543453440127\\
41	-1.32705144995612\\
42	-1.33602820271486\\
43	-1.3294527996667\\
44	-1.31996093710498\\
45	-1.32018176557428\\
46	-1.32494053379554\\
47	-1.33735592039034\\
48	-1.34339401770564\\
49	-1.33957969692087\\
50	-1.35053901850318\\
51	-1.34476886404259\\
52	-1.37641480485034\\
53	-1.35196334938376\\
54	-1.38384298807204\\
55	-1.40343614175046\\
56	-1.38977437984097\\
57	-1.37450630646827\\
58	-1.39440827751891\\
59	-1.41230095819747\\
60	-1.42020708154538\\
};
\end{axis}
\end{tikzpicture}%
  }
  \subfloat[$\lambda=1$]{
%
%
\begin{tikzpicture}

\begin{axis}[%
width=0.951\fwidth,
height=0.75\fwidth,
at={(0\fwidth,0\fwidth)},
scale only axis,
xmode=log,
xmin=1,
xmax=100,
xminorticks=true,
xlabel style={font=\color{white!15!black}},
xlabel={t},
ymin=-0.7,
ymax=-0,
axis background/.style={fill=white}
]
\addplot [color=blue, line width=2.0pt, forget plot]
  table[row sep=crcr]{%
1	-0.015991389253571\\
2	-8.07710251594484e-14\\
3	-2.59936539286564e-14\\
4	-2.60348785986828e-14\\
5	-2.60224200315188e-14\\
6	-2.59820525418944e-14\\
7	-2.59965938614317e-14\\
8	-2.60060611148019e-14\\
9	-2.59506757310226e-14\\
10	-2.58498969706858e-14\\
11	-2.58919987094477e-14\\
12	-2.59483989064604e-14\\
13	-2.595303495495e-14\\
14	-2.59150661948695e-14\\
15	-2.59299891535716e-14\\
16	-2.59070821300713e-14\\
17	-2.58991891382556e-14\\
18	-2.59151399206172e-14\\
19	-2.59036430407802e-14\\
20	-2.59284539232954e-14\\
21	-2.58943058916708e-14\\
22	-2.58760956319817e-14\\
23	-2.59459746304027e-14\\
24	-2.59175295022054e-14\\
25	-2.59825035699981e-14\\
26	-2.58938635371844e-14\\
27	-2.58310058320324e-14\\
28	-2.58420733678091e-14\\
29	-2.58783117412223e-14\\
30	-2.58440162581022e-14\\
31	-2.58429971080601e-14\\
32	-2.58351474843313e-14\\
33	-2.5869000612965e-14\\
34	-2.59091421141991e-14\\
35	-2.58739836061497e-14\\
36	-2.59131536622373e-14\\
37	-2.5829626726869e-14\\
38	-2.58271634195331e-14\\
39	-2.59047142325266e-14\\
40	-2.58443241715192e-14\\
41	-2.58396707757949e-14\\
42	-2.58215472522796e-14\\
43	-2.57845499373457e-14\\
44	-2.57547777456893e-14\\
45	-2.58627556084514e-14\\
46	-2.58798513083072e-14\\
47	-2.58572825558847e-14\\
48	-2.58443111610931e-14\\
49	-2.58880695607746e-14\\
50	-2.60105757326481e-14\\
51	-2.59875602889306e-14\\
52	-2.60626044265013e-14\\
53	-2.58993130083538e-14\\
54	-2.5891459860968e-14\\
55	-2.58690168759975e-14\\
56	-2.59416063798498e-14\\
57	-2.59585459545927e-14\\
58	-2.5878549181498e-14\\
59	-2.57975375951699e-14\\
60	-2.59211019483637e-14\\
};
\addplot [color=red, dashed, line width=3.0pt, forget plot]
  table[row sep=crcr]{%
1	-0.677360003940489\\
2	-0.319188261527385\\
3	-0.238573589451842\\
4	-0.200742996733391\\
5	-0.188801791973792\\
6	-0.157848169305053\\
7	-0.147363581418297\\
8	-0.151590768638053\\
9	-0.164301512639608\\
10	-0.198659955535581\\
11	-0.185472157543024\\
12	-0.189488927671266\\
13	-0.156665545344133\\
14	-0.141346012616817\\
15	-0.142828076379696\\
16	-0.143980205895357\\
17	-0.14494774267628\\
18	-0.145557100703442\\
19	-0.14609092309342\\
20	-0.137324402147981\\
21	-0.134764797162596\\
22	-0.140554448009126\\
23	-0.142899121135694\\
24	-0.143712900049346\\
25	-0.138774590312461\\
26	-0.13836037206626\\
27	-0.135282806601259\\
28	-0.149453779259262\\
29	-0.147470735870606\\
30	-0.154207490480793\\
31	-0.155796490189568\\
32	-0.158732344859596\\
33	-0.163488711963298\\
34	-0.166258826196893\\
35	-0.172242532053593\\
36	-0.168635982124319\\
37	-0.170553091320478\\
38	-0.171303646165357\\
39	-0.185837521836149\\
40	-0.180259389616782\\
41	-0.174263442475532\\
42	-0.162013378918257\\
43	-0.167114368469105\\
44	-0.169568455307836\\
45	-0.173271142995335\\
46	-0.191054413001525\\
47	-0.202850261177353\\
48	-0.209494088522992\\
49	-0.213131370871494\\
50	-0.211058535404276\\
51	-0.207988868337055\\
52	-0.20444195101075\\
53	-0.209951638264728\\
54	-0.218554354326512\\
55	-0.227181804255464\\
56	-0.233332843632866\\
57	-0.240435706083566\\
58	-0.246103352664639\\
59	-0.243526218854078\\
60	-0.236038905669898\\
};
\end{axis}
\end{tikzpicture}%
  }
  \subfloat[legend]{
    \raisebox{4em}{\begin{tikzpicture}

  \begin{axis}[%
    hide axis,
    xmin=10,
    xmax=50,
    ymin=0,
    ymax=0.4,
    legend style={draw=white!15!black,legend cell align=left}
    ]
    \addlegendimage{color=mycolor1, line width=2.0pt}
    \addlegendentry{Bayes};
    \addlegendimage{color=mycolor2, dashed, line width=3.0pt};
    \addlegendentry{Marginal};
  \end{axis}
\end{tikzpicture}}
  }
\caption{\textbf{Sequential allocation.} Performance measured with respect to the empirical model of the holdout COMPAS data, when the DM's actions affect which data will be seen. This means that whenever a prisoner was not released, then the dependent variable $y$ will remain unseen. For that reason, the performance of the Bayesian approach dominates the classical approach even when fairness is not an issue, i.e. $\lambda = 0$.}
\label{fig:sequential-allocation}
\end{figure*}
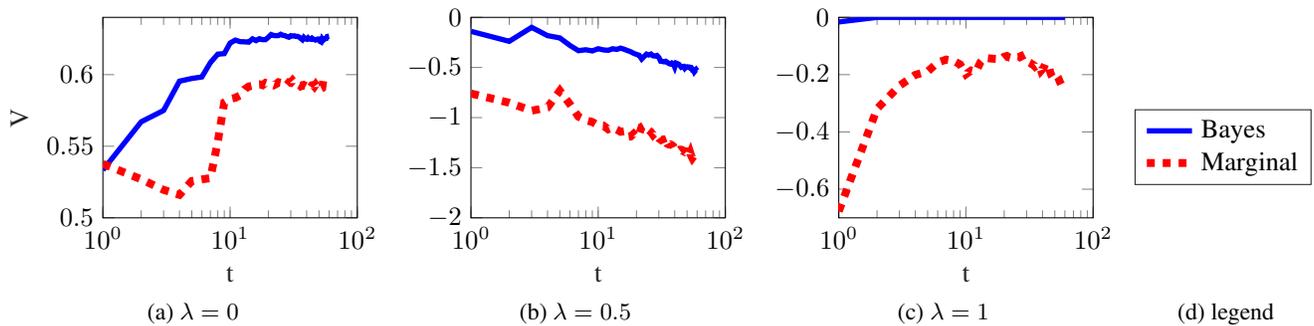

\subsection{Sequential allocation.}
\label{sec:sequential-experiment}
Here the DM, at each time $t$ observes $x_t$ and has a choice of actions $a_t \in \{0,1\}$. The action both predicts $y_t \in \{0,1\}$ and has the following side-effect: the DM only observes $y_t, z_t$ after he makes the choice $a_t=1$, otherwise only $x_t$. 
The utility is not directly observed by the DM, and is measured against the empirical model in the holdout set, as before. We use the same COMPAS dataset, and the results are broadly similar, apart from the fact that the Bayesian decision rule appears to remain robust in this setting, while the marginal one's performance degrades.
We presume that this is because that the Bayesian decision rule explicitly taking into account uncertainty leads to more robust performance relative to the marginal decision rule, which does not.
The results are shown 
in Figure~\ref{fig:sequential-allocation} and its extended version (Figure~\ref{fig:sequential-allocation_extend} in supplementary materials).
The larger discrepancy between for the Bayesian case
in Figure~\ref{fig:sequential-allocation}(a) implies that explicitly modelling uncertainty is also crucial for utility in this case.

\section{Conclusion and future directions}
\label{sec:conclusion}

Existing fairness criteria may be hard to satisfy or verify in a
learning setting because they are defined for the true model. For that
reason, we develop a Bayesian fairness framework, which 
deals explicitly with the information available to the decision maker.
Our framework allows us to more adequately incorporate uncertainty into fairness considerations.
We believe that a further exploration of the informational aspects of
fairness, and in particular for sequential decision problems in the Bayesian setting, will be extremely
fruitful.



\bibliographystyle{aaai}
\bibliography{bibliography}

\newpage
\onecolumn
\appendix

\newpage
\section*{\LARGE Supplementary materials for ``Bayesian Fairness''}

\setcounter{section}{0}

~\\

\section{Impossibility result}

\begin{theorem}\label{thm:impossible}
Calibration and balance conditions cannot hold simultaneously, except: (i) if there exists perfect decision rules that there exists $a,y$ s.t. $ P_\param^\pol(y\mid a) = 0 $ or $ P_\param^\pol(a\mid y) = 0 $, or (ii) $z$ is independent of $y$ that for each $z$, $ P_\param^\pol(z|y) \equiv \text{const.},~\forall y$.
\end{theorem}
\begin{proof}
We prove by contradiction. Using Bayes rule we first have
\begin{align}
    P_\param^\pol(a, z \mid y) =     P_\param^\pol(y, z \mid a) \cdot \frac{P_\param^\pol(a \mid y)}{P_\param^\pol(y \mid a)} ~.\label{eqn:1}
\end{align}
Suppose calibration condition holds, that is 
\[
P_\param^\pol(y, z \mid a) = P_\param^\pol(y \mid a) P_\param^\pol(z \mid a)
\]
Plug above into Eqn. (\ref{eqn:1}) we have 
\begin{align*}
    P_\param^\pol(a, z \mid y) &=    P_\param^\pol(y \mid a) P_\param^\pol(z \mid a) \cdot \frac{P_\param^\pol(a \mid y)}{P_\param^\pol(y \mid a)} \\
    &=P_\param^\pol(z \mid a) \cdot P_\param^\pol(a \mid y).
    \end{align*}
    On the other hand, if balanced condition holds too, we have
    \begin{align*}
    P_\param^\pol(a, z \mid y) =P_\param^\pol(a \mid y) \cdot P_\param^\pol(z \mid y) 
    \end{align*}
    Together we have that 
    \[
  P_\param^\pol(z \mid a) \cdot P_\param^\pol(a \mid y)  = P_\param^\pol(a \mid y) \cdot P_\param^\pol(z \mid y) \rightarrow   P_\param^\pol(z \mid a) =P_\param^\pol(z \mid y),
    \]
    which does not hold when condition (ii) does not hold, completing the proof. 
\end{proof}

~\\

\section{Trivial decision rules for balance}
\label{sec:counterexample}
 
\begin{theorem}
  A trivial decision rule of the form $\pol(a \mid x) = p_a$ can always satisfy balance for a Bayesian decision problem. However, it may be the only balanced decision rule, even when a non-trivial balanced policy can be found for every possible $\param \in \Param$.
  \label{lem:trivial-balance}
\end{theorem}
\begin{proof}
  For the first part, notice that Eqn. \eqref{eq:balanced-rule}
  can be always satisfied trivially if $\pol(a \mid x) = p_a$, i.e. we ignore the observations when taking our actions.
  For the second part, we can rewrite Eqn. \eqref{eq:balanced-rule} as
  \begin{align*}
    \sum_x \pol(a | x) \left[P_\theta(x, z | y) - P_\theta(x | y) P_\theta(z | y) \right] &= 0\\
    \sum_x \pol(a | x) \Delta_\theta(x, y, z) &= 0,
  \end{align*}
  where the $\Delta$ term is only dependent on the parameters.  This
  condition can be satisfied in two ways: the first is if the model
  $\param$ makes $x, z$ conditionally independent on $y$. The second
  is if the vector $\pol(a \mid \cdot)$ is orthogonal to
  $\Delta_\param(\cdot, y, z)$. If $|\CX| > |\CY \times \CZ|$, then,
  for any $\param$, we can always find a policy vector
  $\pol(a \mid \cdot)$ that is orthogonal to all vectors
  $\Delta_\param(\cdot, y, z)$. However, if these vectors across $\param$ have exactly degree of freedom being 1 (since they add up to 0, thus the rank of them can be at most the full rank - 1), then no single policy can be orthogonal
  to all, as otherwise the degree of freedom for this set of vectors will be at least 2. 
\end{proof}

\begin{example}
  In this balance example, there are two models. In the first model, for some value $y$, we have:
  \begin{align}
    P_{\param}(x=0|y) &= 1/4, &
                                P_{\param}(x=0|y,z=1) &= 1/4 - \epsilon, \\
    P_{\param}(x=1|y) &= 1/4, &
                                P_{\param}(x=1|y,z=1) &= 1/4 + \epsilon, \\
    P_{\param}(x=2|y) &= 1/4, &
                                P_{\param}(x=2|y,z=1) &= 1/4 + \epsilon, \\
  \end{align}
  so that
  \begin{align}
    P_{\param}(x=0|y) -  P_{\param}(x=0|y,z=1) &= \epsilon, \\
    P_{\param}(x=1|y) - P_{\param}(x=1|y,z=1) &=  - \epsilon\\
    P_{\param}(x=2|y) - P_{\param}(x=2|y,z=1) &=  - \epsilon
  \end{align}
  Similarly, we can construct models $\param'$ and $\param''$ so that the corresponding differences are $(-\epsilon, \epsilon, -\epsilon)$ and $(\epsilon, \epsilon, -\epsilon)$.
  For any policy $\pol(a \mid x)$ consider the vector $\pol_a = (\pol(a = 1 \mid x))_{x=1}^3$. Note that we can make the policy orthogonal to the first model simply by setting $\pol_a = (1/2, 1/2, 1)$.
\end{example}

~\\

\section{Proof of Theorem \ref{noise:model}}
\begin{proof}
We show the proof for Bayes-balance condition, while the proof for Marginal-balance resembles similarities. Denote the $(1-\delta)$-event that $\theta$ drawn from $\beta(\theta)$ that is $\epsilon$ close to the true model $\theta^*$ in all the conditional probabilities $P_{\theta}(x|y,z), P_{\theta}(x|y)$ as $\mathcal E$, then we have:
\begin{align*}
  &~~~~~\bigl|\sum_x \pol(a | x)
    \left[P_{\param^*}(x, z | y)
  - P_{\param^*}(x | y) P_{\param^*}(z | y) \right]\bigr|\\
  &= \biggl|\int_{\theta \in \mathcal E} \sum_x \pol(a | x)
    \left[P_{\param^*}(x, z | y)
  - P_{\param^*}(x | y) P_{\param^*}(z | y) \right]\\
  &+\int_{\theta \notin \mathcal E}\sum_x \pol(a | x) \left[P_{\param^*}(x, z | y)
  - P_{\param^*}(x | y) P_{\param^*}(z | y) \right]\biggr|\\
  &\leq \biggl|\int_{\theta \in \mathcal E} \sum_x \pol(a | x)
    \left[P_{\param}(x, z | y)
  - P_{\param}(x | y) P_{\param}(z | y) \right] +2\epsilon\\
  &+\int_{\theta \notin \mathcal E}\sum_x \pol(a | x) \left[P_{\param^*}(x, z | y)
  - P_{\param}(x | y) P_{\param}(z | y) \right] + 2\delta\biggr|\\
  &\le \bigl |\sum_x \pol(a | x)
  \int_{\mathrlap{\Param }}
    \left[P_\param(x, z | y)
  - P_\param(x | y) P_\param(z | y) \right]\bigr| +  2(\epsilon+\delta).
  \end{align*}
  Summing over all $a,y,z$ gives us the results. 
\end{proof}
~\\

\section{Gradient calculations for optimal balance decision}
\label{sec:gradient}

For simplicity, let us define the vector in $\Simplex^{\CA}$:
\[
c_w(y,z) = \sum_x \pol_w(\cdot \mid x) \Delta(x, y, z),
\]
so that
\[
f_\lambda(w) = \util(\bel, \pol_w) -  \lambda \sum_{y,z} c_w(y,z)^\top c_w(y,z).
\]
Now
\begin{align*}
  &~~~~\grad_w \left(c_w(y,z)^\top c_w(y,z)\right)
  \\
  &= 
  \grad_w \sum_a c_w(y,z)_a^2\\
  &= 
  \sum_a 2 c_w(y,z)_a
  \grad_w c_w(y, z)_a
  \\
  \grad_w & c_w(y, z)_a
  = 
    \sum_x \grad_w \pol_w(a \mid x) \Delta(x, y, z),
\end{align*}
while
\begin{align}
  \grad \util(\bel, \pol_w)
  &=
    \int_\CX \dd \Pr_\bel(x) \grad_w \pol_w(a \mid x) \E_\bel (\util \mid x, a) 
\end{align}
Combining the two terms, we have
\begin{align*}
  \grad_w f_\lambda(w)
  &= 
    \int_\CX \grad_w \pol_w(a \mid x)
    \bigl[
    \dd \Pr_\bel(x) \E_\bel(U \mid x, a)
  \\
  &- 2 \lambda \sum_{y,z} c_w(y, z)_a \Delta(x, y, z) \dd \Lambda(x),
    \bigr].
\end{align*}
where $\Lambda$ is the Lebesgue measure.
We now derive the gradient for the $\grad_w \pol_w$ term. We consider
two parameterizations.
\paragraph{Independent policy parameters.} When $\pol(a \mid x) = w_{ax}$, we obtain
$$\partial \pol(a' \mid x') / \partial ax = \ind{ax = a'x'}$$.
This unfortunately requires projecting the policy parameters back to  the simplex. For this reason, it might be better to use a parameterization that allows unconstrained optimization. 
\paragraph{Softmax policy parameters.} When 
$$\pol(a \mid x) = e^{w_{ax}} / \sum_{a'} e^{w_{a'x}},$$
we have the following gradients:
\begin{align*}
  \partial \pol(a \mid x) / \partial ax
  &= 
    e^{w_{ax}} \sum_{a' \neq a} e^{w_{a'x}} \left(\sum_{a'} e^{w_{a'x}}\right)^{-2}
  \\
  \partial \pol(a \mid x) / \partial a'x
  &= 
    e^{w_{ax} + w_{a'x}}\left(\sum_{a''} e^{w_{a''x}}\right)^{-2}, ~ a \neq a'\\
  \partial \pol(a \mid x) / \partial a'x'
  &= 
    0, ~ ax \neq a'x'.
\end{align*}

~\\

\section{Empirical formulation.}
For infinite $\CX$, it may be more efficient to rewrite
\eqref{eq:balanced-bayes-constraint} as
\begin{align}
  0
  &=
    \int_{\CX} \pol(a \mid x) 
    \dd \left[ P(x, z \mid y)
    - P(x \mid y) P(z \mid y)\right]
  \\
  &= 
    \int_{\CX} \pol(a \mid x) 
    \left[ P(z \mid y, x)
    - P(z \mid y)\right] \dd P(x \mid y)\\
  &=
    \int_{\CX} \pol(a \mid x) 
    \left[ P(z \mid y, x)
    - P(z \mid y)\right] \frac{P(y \mid x)}{P(y)} \dd P(x)\\
  &\approx
    \sum_{x \sim P_\param(x)} \pol(a \mid x) 
    \left[ P(z \mid y, x)
    - P(z \mid y)\right] \frac{P(y \mid x)}{P(y)}
\intertext{simplifying by dropping the $P(y)$ term:}
  0 &\approx
    \sum_{x \sim P_\param(x)} \pol(a \mid x) 
    \left[ P(z \mid y, x)
    - P(z \mid y)\right] P(y \mid x),
\end{align}
This allows us to approximate the integral by sampling $x$, and can be
useful for e.g. regression problems.

\section{Complete figures}
This section has complete versions of the figures which could not fit in the main text.
\begin{figure*}
\centering
 \subfloat[$\lambda=0$]{
%
%
\definecolor{mycolor1}{rgb}{0.00000,0.75000,0.75000}%
\begin{tikzpicture}

\begin{axis}[%
width=0.951\fwidth,
height=0.75\fwidth,
at={(0\fwidth,0\fwidth)},
scale only axis,
xmode=log,
xmin=1,
xmax=100,
xminorticks=true,
xlabel={t},
ymin=4.5,
ymax=8,
ylabel={$U,F$},
axis background/.style={fill=white}
]
\addplot [color=blue, line width=2.0pt, forget plot]
  table[row sep=crcr]{%
1	5.43561523059217\\
2	5.90370595362012\\
3	5.86193877425643\\
4	6.1974251447707\\
5	6.37538745144895\\
6	6.29084098710727\\
7	6.37533818756081\\
8	6.50213010808347\\
9	6.5991724386949\\
10	6.52046523295858\\
11	6.72543938736623\\
12	6.79086221447693\\
13	6.73720035131321\\
14	6.73338708212267\\
15	6.72543680895541\\
16	6.69340595276163\\
17	6.7374401920428\\
18	6.76459636329051\\
19	6.81917107124367\\
20	6.81961739074537\\
21	6.80848937741292\\
22	6.76496273549776\\
23	6.81943224558462\\
24	6.82707956663492\\
25	6.82513922838762\\
26	6.80872080097762\\
27	6.8194279261825\\
28	6.8189260277632\\
29	6.80836851035401\\
30	6.80853114441057\\
31	6.81931954445528\\
32	6.73856115284411\\
33	6.73816264740434\\
34	6.73350236577814\\
35	6.70687112292105\\
36	6.6785416628419\\
37	6.67686803452494\\
38	6.73847232580075\\
39	6.81942547936638\\
40	6.82688492791769\\
41	6.82666694620967\\
42	6.82685885290022\\
43	6.82704019063763\\
44	6.82689924230517\\
45	6.82677823960367\\
46	6.8270597587393\\
47	6.8194560097514\\
48	6.82689843014005\\
49	6.82701602901863\\
50	6.82681169930566\\
51	6.81949797532045\\
52	6.82669354234299\\
53	6.81943329511497\\
54	6.82651491512469\\
55	6.82704770792639\\
56	6.82679743590289\\
57	6.82658758155496\\
58	6.82669828218253\\
59	6.82660948782845\\
60	6.82669798274291\\
61	6.8265894448791\\
62	6.82677142197074\\
63	6.82681491374407\\
64	6.82666114793132\\
65	6.80855816288338\\
66	6.81943481045244\\
67	6.82660396672482\\
68	6.82386355702123\\
69	6.81947471789097\\
70	6.82700074175287\\
71	6.82693449357647\\
72	6.82709850214994\\
73	6.82660499877464\\
74	6.82697649873285\\
75	6.82670593695921\\
76	6.82675415174255\\
77	6.82687238452097\\
78	6.82695391513762\\
79	6.82698434483159\\
80	6.82686631929662\\
81	6.82688791087935\\
82	6.82646811164595\\
83	6.82644252014829\\
84	6.82665183718036\\
85	6.82669758685249\\
86	6.82676665652696\\
87	6.82695751973057\\
88	6.82658946086894\\
89	6.82655885737069\\
90	6.82664771118609\\
91	6.82677188302957\\
92	6.82669723134228\\
93	6.82649131899508\\
94	6.82694362296807\\
95	6.82701100800417\\
96	6.82689276186491\\
97	6.82697942973237\\
98	6.8265959260873\\
99	6.82685566499548\\
100	6.82661794251553\\
};
\addplot [color=black!50!green, dashed, line width=3.0pt, forget plot]
  table[row sep=crcr]{%
1	5.41383459952623\\
2	6.03977740818528\\
3	6.34445896962323\\
4	6.36360081485984\\
5	6.5327140246455\\
6	6.54312011538435\\
7	6.60180812138142\\
8	6.73373741924754\\
9	6.73566894837812\\
10	6.79427286043087\\
11	6.79651547742057\\
12	6.81745557776205\\
13	6.82361089403465\\
14	6.82453008513155\\
15	6.81832088322606\\
16	6.78744690174988\\
17	6.8245307167606\\
18	6.82360223044389\\
19	6.82320123869577\\
20	6.82664800772996\\
21	6.82680811620926\\
22	6.82647925470442\\
23	6.82692086547104\\
24	6.8269265759928\\
25	6.82693425257449\\
26	6.8268229060682\\
27	6.82700552417735\\
28	6.82667175358748\\
29	6.82668179704755\\
30	6.82696045898621\\
31	6.82660833511297\\
32	6.82680589726841\\
33	6.82671589421294\\
34	6.82682294359754\\
35	6.82701707308604\\
36	6.82650493640409\\
37	6.82699220517708\\
38	6.82668021066387\\
39	6.82672209879424\\
40	6.82681624339904\\
41	6.82673440179816\\
42	6.82689211021934\\
43	6.82677720152432\\
44	6.82694542985608\\
45	6.82666056532852\\
46	6.82671478552557\\
47	6.82693547661731\\
48	6.82672478318047\\
49	6.82677490664466\\
50	6.82656216462366\\
51	6.82648634438907\\
52	6.82704568401014\\
53	6.826862274337\\
54	6.82707378292588\\
55	6.82701064453263\\
56	6.82655872230876\\
57	6.82680284995162\\
58	6.82675484205814\\
59	6.82669898677491\\
60	6.8269587402534\\
61	6.82659491807925\\
62	6.8270850859566\\
63	6.82673295666194\\
64	6.82699550615579\\
65	6.82685223576769\\
66	6.82671348441515\\
67	6.82673202844262\\
68	6.82712820153405\\
69	6.82689184429809\\
70	6.82683272593571\\
71	6.82648201507614\\
72	6.82674866309257\\
73	6.82671747591638\\
74	6.82699024215261\\
75	6.82717941910823\\
76	6.82694493886605\\
77	6.8267239335274\\
78	6.82696730470131\\
79	6.82694340122485\\
80	6.82682976358288\\
81	6.82660119141053\\
82	6.82686412764447\\
83	6.82699084765003\\
84	6.8267722735916\\
85	6.82674429780571\\
86	6.82694302876559\\
87	6.82675714405635\\
88	6.82696031751509\\
89	6.82689255300199\\
90	6.82702023326717\\
91	6.82696488326549\\
92	6.82660850972417\\
93	6.8268004525055\\
94	6.82679611886495\\
95	6.82682594191007\\
96	6.82686735298234\\
97	6.82672601577022\\
98	6.82680403091386\\
99	6.82704486981606\\
100	6.8267221923614\\
};
\addplot [color=red, dashdotted, line width=2.0pt, forget plot]
  table[row sep=crcr]{%
1	4.65679844872143\\
2	5.69596508675862\\
3	6.12283278302887\\
4	6.64373241409889\\
5	7.32558978567296\\
6	6.08163155109078\\
7	6.19748304313167\\
8	6.87388547503056\\
9	7.08389848708165\\
10	7.66590296789734\\
11	7.53380043435034\\
12	7.62021550620541\\
13	6.55198691005539\\
14	7.31957601971931\\
15	7.35795487617778\\
16	6.56844918258317\\
17	6.60268271810485\\
18	6.93824349667414\\
19	6.54115540687113\\
20	6.91918785115645\\
21	7.21517127808337\\
22	7.13583366054737\\
23	6.66321201520444\\
24	6.98638303153576\\
25	6.78324605322824\\
26	7.38817377199521\\
27	6.63843044254195\\
28	6.75672123473148\\
29	7.3764343870945\\
30	7.19570861786257\\
31	6.65820831042757\\
32	6.9430024004874\\
33	6.85869875035141\\
34	6.84454259485917\\
35	6.98880616760368\\
36	7.08431244310177\\
37	7.25857722575256\\
38	6.98686399745706\\
39	6.6537099802175\\
40	6.7214629047179\\
41	6.57818486200004\\
42	6.78856195660524\\
43	6.98657442393879\\
44	6.70851618145228\\
45	6.77231060119333\\
46	6.91138081525664\\
47	6.80900520532078\\
48	6.8955319793459\\
49	6.84792233580796\\
50	6.79275017251076\\
51	6.71426555740713\\
52	6.59004662459356\\
53	6.77652926179091\\
54	6.65083295234543\\
55	6.85126864447383\\
56	6.74854673311024\\
57	6.75003213226614\\
58	6.89126758286218\\
59	6.82121018801547\\
60	6.73429925858514\\
61	6.55866795761166\\
62	6.77542717061313\\
63	6.7505053995152\\
64	6.61001242197726\\
65	7.25454977658437\\
66	6.66630203216666\\
67	6.49701316241904\\
68	6.74842532476086\\
69	6.71504163949096\\
70	6.97439204366755\\
71	6.76367082195242\\
72	6.96309910967014\\
73	6.82447935225019\\
74	6.97808966701071\\
75	6.92487707804123\\
76	6.66720140412107\\
77	6.79555773655483\\
78	6.8885209752185\\
79	6.82374109038333\\
80	6.72908264481633\\
81	6.86784379199626\\
82	6.63745368452247\\
83	6.5131420448435\\
84	6.88807751914145\\
85	6.72286398044753\\
86	6.73792752347492\\
87	6.93224164162795\\
88	6.837553331591\\
89	6.70329794971067\\
90	6.77161494678466\\
91	6.92391725886576\\
92	6.64034499050567\\
93	6.80661355109981\\
94	6.90661745847731\\
95	6.90906010305707\\
96	6.73117108750913\\
97	6.96005426572143\\
98	6.54290590591496\\
99	6.69200091471541\\
100	6.67305945104694\\
};
\addplot [color=mycolor1, dotted, line width=3.0pt, forget plot]
  table[row sep=crcr]{%
1	4.63856673683556\\
2	7.65090449353503\\
3	7.29720941968353\\
4	6.39411547781226\\
5	5.59180859447572\\
6	5.91103616867134\\
7	5.65345033980548\\
8	6.99157334254097\\
9	6.98378088352546\\
10	6.36878678535445\\
11	6.5703513033626\\
12	6.72120527280937\\
13	7.18827355673923\\
14	7.07400805356152\\
15	7.07676946422607\\
16	6.68918545343691\\
17	6.87716882988162\\
18	7.21295461204102\\
19	6.89059757703387\\
20	6.6418700279056\\
21	6.72969261054924\\
22	6.58527483721895\\
23	6.84422565408741\\
24	6.76291708398274\\
25	6.78586446833003\\
26	6.97929349712728\\
27	6.91929251936603\\
28	6.59430908038756\\
29	6.82168411870536\\
30	6.86077963582398\\
31	6.71389518734896\\
32	6.78498708606475\\
33	6.79899086704355\\
34	6.85315001850439\\
35	6.96486748159995\\
36	6.73627790220896\\
37	6.8592233857706\\
38	6.56505333926661\\
39	6.73635246282633\\
40	6.75269318472154\\
41	6.67726773578806\\
42	6.90187689510759\\
43	6.78042022470713\\
44	6.77019613843506\\
45	6.73856118706226\\
46	6.57185409950028\\
47	6.96761071311978\\
48	6.66130453339595\\
49	6.80708379997048\\
50	6.85748230435172\\
51	6.70135826348809\\
52	6.9938569734353\\
53	6.72539117699799\\
54	6.94640793627233\\
55	6.9385051726723\\
56	6.61631698807355\\
57	6.788520931626\\
58	6.88572558664489\\
59	6.68380106971717\\
60	6.86253282851936\\
61	6.65810790284873\\
62	6.8793638877059\\
63	6.78922309796063\\
64	6.94456473813911\\
65	6.70260082131938\\
66	6.72230045188807\\
67	6.69964002067404\\
68	7.00732110672733\\
69	6.81079303773709\\
70	6.68424630783847\\
71	6.52353247112277\\
72	6.73012734112837\\
73	6.69496855949307\\
74	6.92703451903552\\
75	7.02396178656437\\
76	6.81686683630093\\
77	6.59425523559705\\
78	6.95528299863492\\
79	6.75296627805458\\
80	6.78184138131174\\
81	6.63659860007914\\
82	6.88566575528246\\
83	6.92879059314807\\
84	6.64710033562532\\
85	6.65151555001541\\
86	6.9180522788148\\
87	6.71736971916958\\
88	6.90775109592753\\
89	6.89321397542376\\
90	6.80705343437269\\
91	6.93499379578587\\
92	6.59613086758874\\
93	6.86088232672981\\
94	6.67935174396831\\
95	6.81958814386339\\
96	6.78885095717227\\
97	6.62347677209479\\
98	6.66807026959861\\
99	6.99366307038507\\
100	6.7131226265173\\
};
\end{axis}
\end{tikzpicture}%
  }
  \subfloat[$\lambda=0.25$]{
%
%
\definecolor{mycolor1}{rgb}{0.00000,0.75000,0.75000}%
\begin{tikzpicture}

\begin{axis}[%
width=0.951\fwidth,
height=0.75\fwidth,
at={(0\fwidth,0\fwidth)},
scale only axis,
xmode=log,
xmin=1,
xmax=100,
xminorticks=true,
xlabel={t},
ylabel={$U, F$},
ymin=3,
ymax=8,
axis background/.style={fill=white}
]
\addplot [color=blue, line width=2.0pt, forget plot]
  table[row sep=crcr]{%
1	5.48092030379997\\
2	5.62187002690099\\
3	5.95863155946218\\
4	6.28898882529671\\
5	6.37263646084564\\
6	6.44619863535187\\
7	6.51978070369432\\
8	6.6260112803707\\
9	6.59969215942956\\
10	6.64710138462075\\
11	6.64739151299339\\
12	6.64510042770851\\
13	6.64152382540431\\
14	6.77898626490996\\
15	6.78942211106642\\
16	6.78461368559541\\
17	6.78747870326049\\
18	6.78938642097554\\
19	6.78832161718462\\
20	6.78775050690658\\
21	6.78944722876843\\
22	6.78917577314025\\
23	6.78819139814796\\
24	6.78724347966865\\
25	6.78821620436663\\
26	6.78831214850696\\
27	6.7880922368358\\
28	6.7881723236791\\
29	6.78794233482332\\
30	6.78813814586129\\
31	6.78820354351858\\
32	6.78841099428386\\
33	6.78808353541935\\
34	6.7881908225767\\
35	6.78774827696057\\
36	6.7885187243922\\
37	6.78845419838743\\
38	6.78814836495893\\
39	6.78821228904136\\
40	6.78818735684172\\
41	6.78799405430338\\
42	6.7882713816731\\
43	6.78927591624498\\
44	6.78946030227827\\
45	6.78821020841118\\
46	6.78827185230104\\
47	6.78860337850886\\
48	6.78748352584787\\
49	6.78799354670293\\
50	6.78842672785986\\
51	6.7880902860153\\
52	6.78808388647639\\
53	6.78747738661796\\
54	6.78819916095484\\
55	6.78835895289234\\
56	6.78843293590063\\
57	6.78825725240079\\
58	6.78829796145325\\
59	6.788057466024\\
60	6.78839237562291\\
61	6.78834851740735\\
62	6.78735253906824\\
63	6.78775411240988\\
64	6.78863576035963\\
65	6.78755861187669\\
66	6.78954119015546\\
67	6.78873337301612\\
68	6.78816033804725\\
69	6.78820344967906\\
70	6.78776361064147\\
71	6.78814177015364\\
72	6.7879366996741\\
73	6.7880936081835\\
74	6.78825561226921\\
75	6.78821316644409\\
76	6.78918787026003\\
77	6.78818333235925\\
78	6.78819645711209\\
79	6.78810751942358\\
80	6.78842982644288\\
81	6.78739048420254\\
82	6.78823104668009\\
83	6.78781609618644\\
84	6.78820734692313\\
85	6.78828044289443\\
86	6.78722601773148\\
87	6.78755406937019\\
88	6.78811157856524\\
89	6.78808107645364\\
90	6.789012996106\\
91	6.7881650039973\\
92	6.78821720889247\\
93	6.78807168410838\\
94	6.78810800021703\\
95	6.78793379320975\\
96	6.78826043758827\\
97	6.78807781437123\\
98	6.78818537206819\\
99	6.78850763801694\\
100	6.7884621883911\\
};
\addplot [color=black!50!green, dashed, line width=3.0pt, forget plot]
  table[row sep=crcr]{%
1	5.46839092366965\\
2	5.75703558072284\\
3	6.21187308663476\\
4	6.39102664381444\\
5	6.55586234089632\\
6	6.38413615637934\\
7	6.42300581162868\\
8	6.67498693237809\\
9	6.63039319016224\\
10	6.58773391003488\\
11	6.75296577578321\\
12	6.786443793693\\
13	6.78805569029723\\
14	6.78745585366746\\
15	6.78832428887165\\
16	6.78830066490763\\
17	6.78839106970021\\
18	6.78808802788535\\
19	6.78798694582777\\
20	6.78817693494755\\
21	6.7881008909547\\
22	6.78842032667698\\
23	6.78783128904056\\
24	6.78875615083487\\
25	6.78845727850269\\
26	6.78819401413074\\
27	6.78834196285126\\
28	6.7872101033022\\
29	6.7879980507075\\
30	6.78824430728795\\
31	6.78816970586208\\
32	6.78818801908122\\
33	6.78814017664448\\
34	6.78825078323049\\
35	6.7881484859051\\
36	6.78818691493551\\
37	6.78808987102091\\
38	6.78816298338296\\
39	6.78799846176474\\
40	6.78828094214078\\
41	6.78727871939925\\
42	6.78803975010666\\
43	6.78725571301125\\
44	6.78871518089493\\
45	6.78840508970577\\
46	6.7875707596091\\
47	6.7881728716072\\
48	6.78846952964538\\
49	6.78817772400779\\
50	6.78840021851439\\
51	6.78822370152359\\
52	6.78740754507343\\
53	6.78818214694285\\
54	6.78826729796029\\
55	6.78794709000275\\
56	6.78819769453097\\
57	6.78773687436855\\
58	6.78760780764942\\
59	6.78820024812653\\
60	6.78822638564773\\
61	6.78816361161253\\
62	6.78832451810031\\
63	6.78840398227284\\
64	6.78810379421081\\
65	6.7882192655374\\
66	6.7882812746018\\
67	6.78767994390469\\
68	6.78789414506478\\
69	6.78833152743326\\
70	6.78812570838307\\
71	6.78853051769768\\
72	6.78820591270197\\
73	6.78840134179175\\
74	6.78833646862341\\
75	6.78714175770295\\
76	6.78827649125506\\
77	6.78761416139175\\
78	6.78818828234791\\
79	6.78820835330531\\
80	6.78869172117444\\
81	6.78888773725641\\
82	6.78807839890497\\
83	6.78822911534438\\
84	6.7882605284278\\
85	6.78824522127325\\
86	6.78813733439761\\
87	6.78828150957633\\
88	6.787818875421\\
89	6.78808023524246\\
90	6.78808690828464\\
91	6.78820378585089\\
92	6.78810395152641\\
93	6.78826222843939\\
94	6.78821516715624\\
95	6.7884244130633\\
96	6.78815086989971\\
97	6.78823200461007\\
98	6.78762008169212\\
99	6.78764861589493\\
100	6.78759288324044\\
};
\addplot [color=red, dashdotted, line width=2.0pt, forget plot]
  table[row sep=crcr]{%
1	3.95263838270111\\
2	3.82503765619641\\
3	4.69915687329073\\
4	4.95179395517991\\
5	5.69532493678198\\
6	5.37334301528087\\
7	4.83235868797524\\
8	4.73554152570934\\
9	5.01931591636941\\
10	5.03432141135122\\
11	4.79254541499617\\
12	4.63914092183212\\
13	4.44750001169128\\
14	4.81685316598652\\
15	4.9147788424582\\
16	4.91327394352937\\
17	4.95995623494322\\
18	4.8945337204222\\
19	4.84879148065828\\
20	4.83909002285821\\
21	4.88259365579378\\
22	4.88389090533673\\
23	4.8443938066954\\
24	4.83840179753397\\
25	4.84558816878397\\
26	4.84905292663818\\
27	4.84147078022581\\
28	4.84372820754065\\
29	4.83890505616955\\
30	4.84232199646838\\
31	4.84523950410094\\
32	4.85078354779625\\
33	4.84135735710565\\
34	4.84520888873435\\
35	4.84215081653731\\
36	4.85535850522084\\
37	4.85185247997884\\
38	4.84260169610015\\
39	4.84620634763515\\
40	4.84523575187968\\
41	4.83902567068879\\
42	4.84701671196802\\
43	4.87609116537311\\
44	4.88206656252824\\
45	4.84539564610091\\
46	4.84731781906082\\
47	4.85657048784665\\
48	4.84464814453533\\
49	4.84034483610184\\
50	4.85098637586876\\
51	4.85234480413716\\
52	4.84223427910482\\
53	4.83645625338907\\
54	4.84534540433531\\
55	4.84989668463203\\
56	4.85152070699286\\
57	4.84710109351825\\
58	4.84850862271226\\
59	4.8417612393451\\
60	4.8505172651455\\
61	4.84854018635913\\
62	4.84166599328614\\
63	4.83795694229326\\
64	4.85812270662427\\
65	4.8507748368422\\
66	4.88437423777575\\
67	4.85926481587375\\
68	4.84382853660602\\
69	4.8450886394879\\
70	4.83777595797139\\
71	4.8436077111013\\
72	4.83723787081233\\
73	4.8419858580276\\
74	4.84700191788378\\
75	4.84679719167575\\
76	4.88422845265261\\
77	4.84492363174228\\
78	4.84542290433572\\
79	4.84308012880218\\
80	4.85170602717477\\
81	4.84406941230317\\
82	4.84661435851207\\
83	4.83788327052865\\
84	4.84567275514356\\
85	4.84723479294532\\
86	4.83903092538253\\
87	4.83854193953482\\
88	4.84165658956618\\
89	4.84096009660066\\
90	4.86812900077183\\
91	4.84417172920604\\
92	4.84595715415907\\
93	4.84059852688428\\
94	4.84207974058152\\
95	4.84058822114876\\
96	4.84739615953176\\
97	4.84066509683728\\
98	4.84408611041905\\
99	4.85435432738002\\
100	4.85140753714343\\
};
\addplot [color=mycolor1, dotted, line width=3.0pt, forget plot]
  table[row sep=crcr]{%
1	4.43942617049761\\
2	6.07517227920279\\
3	6.01511546400837\\
4	7.99001185146673\\
5	6.91853861731467\\
6	7.45736885559621\\
7	5.9696968988519\\
8	5.90044821555372\\
9	6.53253989933909\\
10	5.95793744141482\\
11	4.47853828240861\\
12	5.01187651983101\\
13	4.86994172661999\\
14	4.87457926082996\\
15	4.91605950930105\\
16	4.83664217269984\\
17	4.84487405751951\\
18	4.84572345362514\\
19	4.84496677844009\\
20	4.84859591315895\\
21	4.8439187330126\\
22	4.85436336951604\\
23	4.84095204112765\\
24	4.86334743088784\\
25	4.85229857867871\\
26	4.84514781002158\\
27	4.84965309635877\\
28	4.83928916416435\\
29	4.84189589312187\\
30	4.84704232748438\\
31	4.84549575927436\\
32	4.84486903809857\\
33	4.84236622368815\\
34	4.84766997345754\\
35	4.85758994431681\\
36	4.84579416567247\\
37	4.84096626422101\\
38	4.84505804961238\\
39	4.8440266791363\\
40	4.84775057449601\\
41	4.84209322722535\\
42	4.84050324640196\\
43	4.83980167280306\\
44	4.85923425685094\\
45	4.85111829634059\\
46	4.84132164675338\\
47	4.84370161642999\\
48	4.87223462946916\\
49	4.84432618495737\\
50	4.8500961427862\\
51	4.84586521700634\\
52	4.84466882674989\\
53	4.84385534411402\\
54	4.84618282826888\\
55	4.85639560861603\\
56	4.8456548959247\\
57	4.84125063954534\\
58	4.84248664382494\\
59	4.84531192138565\\
60	4.84571944195174\\
61	4.84383104434909\\
62	4.84870278970119\\
63	4.85126878016323\\
64	4.84210931804219\\
65	4.84625260846576\\
66	4.84761300621132\\
67	4.83951697497941\\
68	4.84891722764123\\
69	4.84990546659528\\
70	4.84234548798647\\
71	4.85552485686154\\
72	4.84534656853895\\
73	4.85214524633418\\
74	4.84833846211763\\
75	4.83836834799854\\
76	4.84747334309443\\
77	4.84577873600088\\
78	4.84566127923588\\
79	4.84656159051625\\
80	4.85932831037586\\
81	4.86458441816834\\
82	4.84154805775504\\
83	4.84631576231663\\
84	4.84721937222264\\
85	4.8458013115131\\
86	4.84218017264925\\
87	4.84713277748063\\
88	4.83723898611049\\
89	4.84187774364254\\
90	4.84171840375588\\
91	4.8452678398018\\
92	4.84182079956507\\
93	4.84699098516083\\
94	4.84549640548173\\
95	4.85158555409688\\
96	4.84409812448365\\
97	4.84640282767493\\
98	4.84023191773009\\
99	4.83639997556977\\
100	4.83899503618132\\
};
\end{axis}
\end{tikzpicture}%
  }
  \subfloat[$\lambda=0.5$]{
%
%
\definecolor{mycolor1}{rgb}{0.00000,0.75000,0.75000}%
\begin{tikzpicture}

\begin{axis}[%
width=0.951\fwidth,
height=0.75\fwidth,
at={(0\fwidth,0\fwidth)},
scale only axis,
xmode=log,
xmin=1,
xmax=100,
xminorticks=true,
xlabel={t},
ymin=2,
ymax=7,
axis background/.style={fill=white}
]
\addplot [color=blue, line width=2.0pt, forget plot]
  table[row sep=crcr]{%
1	5.43626321740328\\
2	5.62955957449209\\
3	5.83892212841107\\
4	6.08455180148673\\
5	6.45389213646402\\
6	6.51404201900693\\
7	6.55425288332807\\
8	6.47841299626561\\
9	6.62069875641757\\
10	6.51422248577047\\
11	6.59838520559784\\
12	6.59032275016397\\
13	6.52849619086474\\
14	6.68240075733416\\
15	6.70628180193069\\
16	6.70366440054659\\
17	6.70078300565614\\
18	6.70264005866184\\
19	6.70272040838255\\
20	6.70215960026927\\
21	6.7035515874975\\
22	6.70405002626636\\
23	6.70380004965903\\
24	6.70280860477349\\
25	6.70188832548547\\
26	6.70154419712915\\
27	6.70290303559077\\
28	6.70099710415468\\
29	6.70193379267451\\
30	6.70120150981987\\
31	6.70173152738914\\
32	6.70159866600691\\
33	6.70148331010774\\
34	6.70024700158092\\
35	6.70162388648996\\
36	6.70122567452083\\
37	6.70119760380918\\
38	6.7008942387301\\
39	6.70207192686816\\
40	6.70206106174765\\
41	6.70118803780591\\
42	6.7011708537815\\
43	6.70121592271\\
44	6.70090701893929\\
45	6.70090952627623\\
46	6.70172791977211\\
47	6.7004604519536\\
48	6.70214417070081\\
49	6.70087927793523\\
50	6.7009280299354\\
51	6.70192571875951\\
52	6.70171886250275\\
53	6.70150820728044\\
54	6.70126159777429\\
55	6.70041242169488\\
56	6.7002853540649\\
57	6.70046583835858\\
58	6.70020648356859\\
59	6.70116548685459\\
60	6.70050519866534\\
61	6.7015615201992\\
62	6.7007685612289\\
63	6.70095050594716\\
64	6.70154935714809\\
65	6.70152349860255\\
66	6.70070897977991\\
67	6.7009292344296\\
68	6.70068914149036\\
69	6.70163551712463\\
70	6.70142476387294\\
71	6.70069739926064\\
72	6.7017336129983\\
73	6.70128008473107\\
74	6.70162543894496\\
75	6.70146658599244\\
76	6.70170456840915\\
77	6.7013768002497\\
78	6.70182944797324\\
79	6.70116360226985\\
80	6.701438682881\\
81	6.70092732271296\\
82	6.70135911617293\\
83	6.70188062229687\\
84	6.70075474316751\\
85	6.70097329900569\\
86	6.70040857565787\\
87	6.70120199869424\\
88	6.70120379847468\\
89	6.70170383955328\\
90	6.70113156603443\\
91	6.70058910840515\\
92	6.70055443379444\\
93	6.70091509145021\\
94	6.70170584002308\\
95	6.70130045964099\\
96	6.70227846487977\\
97	6.70187435109027\\
98	6.70119790741984\\
99	6.70133993723436\\
100	6.70064902183277\\
};
\addplot [color=black!50!green, dashed, line width=3.0pt, forget plot]
  table[row sep=crcr]{%
1	5.51006868459902\\
2	5.94349925651812\\
3	6.15110188956242\\
4	6.02712399624356\\
5	6.32392569534522\\
6	6.36115851019674\\
7	6.36007362082859\\
8	6.42304614951736\\
9	6.46234178919198\\
10	6.46733988080431\\
11	6.45899027408367\\
12	6.3374431395461\\
13	6.54985803215858\\
14	6.3167027540054\\
15	6.46001002676741\\
16	6.56942958143311\\
17	6.61343634104224\\
18	6.63292515036324\\
19	6.63448250199397\\
20	6.63488106426994\\
21	6.63474315974286\\
22	6.68991621855335\\
23	6.6335990864829\\
24	6.69904685642871\\
25	6.70141875769818\\
26	6.7012806526023\\
27	6.70070437469069\\
28	6.70101301806078\\
29	6.70067654491554\\
30	6.70099908199053\\
31	6.70110335421526\\
32	6.70109454031283\\
33	6.70048011400453\\
34	6.70015802246409\\
35	6.70146803260837\\
36	6.70137519147395\\
37	6.70074615877402\\
38	6.70091163421494\\
39	6.7008946179437\\
40	6.70109518272528\\
41	6.70205145033657\\
42	6.70115273392061\\
43	6.70186718313822\\
44	6.70177380992648\\
45	6.7003804780666\\
46	6.70153952966726\\
47	6.70174797129619\\
48	6.70057765354134\\
49	6.70081922974977\\
50	6.70164333602487\\
51	6.70125515905557\\
52	6.70152742403787\\
53	6.70163646612874\\
54	6.70157301629276\\
55	6.701414269396\\
56	6.70151509569165\\
57	6.70113067220305\\
58	6.70067348960663\\
59	6.70094302840244\\
60	6.70066241369534\\
61	6.70171037678393\\
62	6.7017258902379\\
63	6.70100013577623\\
64	6.70084519318045\\
65	6.70123086292297\\
66	6.70103013736592\\
67	6.70117679885823\\
68	6.70136115786824\\
69	6.70167814502445\\
70	6.70122192741145\\
71	6.70183309433243\\
72	6.70100999249549\\
73	6.70135356970916\\
74	6.7002459876897\\
75	6.70146912361473\\
76	6.70130258398882\\
77	6.70075589014854\\
78	6.70080348452352\\
79	6.70112183391194\\
80	6.70120314765165\\
81	6.70119447031542\\
82	6.70144475842244\\
83	6.70137848453271\\
84	6.70168043308462\\
85	6.70077627802382\\
86	6.70071854612631\\
87	6.70101623736037\\
88	6.70104533420828\\
89	6.70135513010366\\
90	6.70056144078923\\
91	6.70102557689087\\
92	6.7006300362861\\
93	6.70107085890124\\
94	6.70153290457477\\
95	6.70133276766025\\
96	6.70119862303415\\
97	6.70089139040385\\
98	6.70152026042981\\
99	6.70007620915822\\
100	6.70036852816079\\
};
\addplot [color=red, dashdotted, line width=2.0pt, forget plot]
  table[row sep=crcr]{%
1	2.44661771734869\\
2	2.52576353700595\\
3	2.16127647791958\\
4	3.0310290948882\\
5	3.40887425643453\\
6	3.50341814626276\\
7	3.56683307706032\\
8	3.73648862770094\\
9	3.23565565554268\\
10	3.40502086781842\\
11	3.04321944910621\\
12	3.01889669235785\\
13	3.43068866230066\\
14	3.25900115565335\\
15	3.17273095025148\\
16	3.10459827121771\\
17	3.0526650791918\\
18	3.09139487891992\\
19	3.09579296675206\\
20	3.09234915988499\\
21	3.09851430049847\\
22	3.10488187422497\\
23	3.11031833793038\\
24	3.09025546200949\\
25	3.07799538254974\\
26	3.05651836407002\\
27	3.09294115326098\\
28	3.04643784169409\\
29	3.08269460708509\\
30	3.07199875153622\\
31	3.06343319656065\\
32	3.07375106751663\\
33	3.06639676553553\\
34	3.05995126777791\\
35	3.07517841982544\\
36	3.06389299412385\\
37	3.06951122517533\\
38	3.04740118393345\\
39	3.08209796772388\\
40	3.07641205413439\\
41	3.05399563766818\\
42	3.05893438157334\\
43	3.05541018738379\\
44	3.06393384924717\\
45	3.04896819364968\\
46	3.07750178838152\\
47	3.05587980147252\\
48	3.08182658514786\\
49	3.04540555034994\\
50	3.06253296378515\\
51	3.07015540037389\\
52	3.06967956686681\\
53	3.0665229029738\\
54	3.06112004082199\\
55	3.05154570012539\\
56	3.04896094300526\\
57	3.05485326870482\\
58	3.04000064341995\\
59	3.059682219609\\
60	3.06316955979892\\
61	3.06141679688057\\
62	3.04373363846218\\
63	3.06791502934542\\
64	3.0757031610305\\
65	3.07958268848526\\
66	3.05845004880157\\
67	3.06552120542576\\
68	3.0419291877304\\
69	3.06806252792471\\
70	3.06832192190414\\
71	3.05477491131693\\
72	3.06448088366623\\
73	3.04826170050289\\
74	3.06948504914674\\
75	3.06412523565948\\
76	3.07259231360712\\
77	3.06403226648865\\
78	3.06970412697863\\
79	3.06165521805732\\
80	3.0614821477698\\
81	3.052753008387\\
82	3.06648346891204\\
83	3.07741163270134\\
84	3.04639913183469\\
85	3.06273214783606\\
86	3.06596710461126\\
87	3.04983966925001\\
88	3.06539272911671\\
89	3.06856762532641\\
90	3.05256539118647\\
91	3.03999612848695\\
92	3.05343433068084\\
93	3.05168086156649\\
94	3.06665063702698\\
95	3.05274503115123\\
96	3.08200349747352\\
97	3.07318177136839\\
98	3.06615315275989\\
99	3.05534532007844\\
100	3.05738636915808\\
};
\addplot [color=mycolor1, dotted, line width=3.0pt, forget plot]
  table[row sep=crcr]{%
1	5.10841690753339\\
2	6.30366720328818\\
3	6.59816030254869\\
4	6.17790900319201\\
5	4.46475981640557\\
6	5.00765820609084\\
7	4.30951390227002\\
8	4.04161842546795\\
9	4.01485656248218\\
10	4.34312892872975\\
11	3.66541471620045\\
12	3.517951425459\\
13	3.61233274197403\\
14	3.74518751805426\\
15	4.16236954544025\\
16	3.64187637079015\\
17	3.4741715491685\\
18	3.29478144435632\\
19	3.18876107018515\\
20	3.20397854482619\\
21	3.19320600929792\\
22	3.27036325372438\\
23	3.20516039434347\\
24	3.05980848849459\\
25	3.05167931384347\\
26	3.05717743757322\\
27	3.0412954175818\\
28	3.06149776452893\\
29	3.04849873207019\\
30	3.04518142355283\\
31	3.05500058612968\\
32	3.05073964570803\\
33	3.05686069870436\\
34	3.04945434987303\\
35	3.07163211580473\\
36	3.07857793419349\\
37	3.07135457940943\\
38	3.05417963349527\\
39	3.06161284949563\\
40	3.07797428607014\\
41	3.08166389658056\\
42	3.06906824068655\\
43	3.07685730044911\\
44	3.06959642640651\\
45	3.04382768227879\\
46	3.06171377644948\\
47	3.07010403124248\\
48	3.05731202400714\\
49	3.07062597924447\\
50	3.06229126911754\\
51	3.05690352916751\\
52	3.06922983803813\\
53	3.05873407555027\\
54	3.0560740176733\\
55	3.0796699396199\\
56	3.05903847571533\\
57	3.05346016169011\\
58	3.05274905852315\\
59	3.06334635906267\\
60	3.05083683979572\\
61	3.06929869363696\\
62	3.07363736785062\\
63	3.06771725206619\\
64	3.0555037122552\\
65	3.06363688025101\\
66	3.05429921537096\\
67	3.05786906562215\\
68	3.06089932774921\\
69	3.08142837576992\\
70	3.07119110837081\\
71	3.07819963661889\\
72	3.07592939418843\\
73	3.06934163009663\\
74	3.04161092943315\\
75	3.06115524034447\\
76	3.05413353643378\\
77	3.04342239143652\\
78	3.06695751758088\\
79	3.06223131468353\\
80	3.06790529182809\\
81	3.07235109327085\\
82	3.06559871297931\\
83	3.06886571268794\\
84	3.07193206572329\\
85	3.06176523652693\\
86	3.05312708109413\\
87	3.05450888405249\\
88	3.0615663512455\\
89	3.05752676840432\\
90	3.05975652143678\\
91	3.07211379467446\\
92	3.0525903118369\\
93	3.04773333165863\\
94	3.06121170064906\\
95	3.06467535806715\\
96	3.07540315101183\\
97	3.06243086424018\\
98	3.06862373610331\\
99	3.04294148478522\\
100	3.05839878495847\\
};
\end{axis}
\end{tikzpicture}%
  }
  \\
  \subfloat[$\lambda=0.75$]{
%
%
\definecolor{mycolor1}{rgb}{0.00000,0.75000,0.75000}%
\begin{tikzpicture}

\begin{axis}[%
width=0.951\fwidth,
height=0.75\fwidth,
at={(0\fwidth,0\fwidth)},
scale only axis,
xmode=log,
xmin=1,
xmax=100,
xminorticks=true,
xlabel={t},
ymin=0,
ymax=7,
ylabel={$U, F$},
axis background/.style={fill=white}
]
\addplot [color=blue, line width=2.0pt, forget plot]
  table[row sep=crcr]{%
1	5.47461127098968\\
2	5.60478890856896\\
3	5.77191033058464\\
4	5.90011496868576\\
5	6.14377790957163\\
6	6.31041948276384\\
7	6.43552914365433\\
8	6.5151531251348\\
9	6.54729734736349\\
10	6.57798495912498\\
11	6.57849571632306\\
12	6.58765985244295\\
13	6.58900819766999\\
14	6.58807470034135\\
15	6.58544669469427\\
16	6.58533577780727\\
17	6.58370067428244\\
18	6.58971365581165\\
19	6.5878980368345\\
20	6.59048092832187\\
21	6.59125044154283\\
22	6.5896710150567\\
23	6.58987286386257\\
24	6.58946117658065\\
25	6.58999333445314\\
26	6.59031718167873\\
27	6.59139413998349\\
28	6.59090165755066\\
29	6.59051023277067\\
30	6.59219941170473\\
31	6.59124669937017\\
32	6.58902226344559\\
33	6.58930522013265\\
34	6.59156801046039\\
35	6.59140118319131\\
36	6.59071800018758\\
37	6.5904962121109\\
38	6.58994289370681\\
39	6.59067144643273\\
40	6.5908473033638\\
41	6.58935869164842\\
42	6.59079752619829\\
43	6.59045257663437\\
44	6.59038156874226\\
45	6.58922214845872\\
46	6.5920430806679\\
47	6.59061552462121\\
48	6.5910853518908\\
49	6.5904707507656\\
50	6.59147532593014\\
51	6.59088905580135\\
52	6.59228441522875\\
53	6.59101992136476\\
54	6.59064064465195\\
55	6.59183783311411\\
56	6.59023817047242\\
57	6.59061017708232\\
58	6.5906193777462\\
59	6.59140530114315\\
60	6.5912579079603\\
61	6.59103059374945\\
62	6.59073429065668\\
63	6.5902464349662\\
64	6.59060596841733\\
65	6.59015307164583\\
66	6.59133709302133\\
67	6.59065526384808\\
68	6.58958255899406\\
69	6.59157679882054\\
70	6.59033990286414\\
71	6.59149880127007\\
72	6.59081749662089\\
73	6.58978538380844\\
74	6.59092637473552\\
75	6.58927529131999\\
76	6.58850343746042\\
77	6.59070961917545\\
78	6.59053467074177\\
79	6.58991848266028\\
80	6.59074090596838\\
81	6.59116797494771\\
82	6.59120811456219\\
83	6.58951958489378\\
84	6.59076403901614\\
85	6.58992010133594\\
86	6.59063253300245\\
87	6.59074681389064\\
88	6.59131520405029\\
89	6.58979941235916\\
90	6.58994888558284\\
91	6.59168176024839\\
92	6.59066396864525\\
93	6.59105539706456\\
94	6.59127950787074\\
95	6.59064023962094\\
96	6.59008138513551\\
97	6.59140183039189\\
98	6.58838036017117\\
99	6.5909343332388\\
100	6.59170125491083\\
};
\addplot [color=black!50!green, dashed, line width=3.0pt, forget plot]
  table[row sep=crcr]{%
1	5.62056595098105\\
2	5.9249699465451\\
3	6.01582827122022\\
4	5.95571514939206\\
5	5.92294247321893\\
6	5.81452251942304\\
7	6.13426812770014\\
8	6.04595676543718\\
9	6.29227294231064\\
10	6.27799612179181\\
11	6.32051997859251\\
12	6.36527532407389\\
13	6.32925991780952\\
14	6.45465237806581\\
15	6.58223240737318\\
16	6.57987442531153\\
17	6.55874072799058\\
18	6.5753041864462\\
19	6.58119554829056\\
20	6.58728407958621\\
21	6.57809056851983\\
22	6.56764059133101\\
23	6.41013588735602\\
24	6.38917378777253\\
25	6.58721628658142\\
26	6.58203616141645\\
27	6.4153585168044\\
28	6.56214565465418\\
29	6.42451672084935\\
30	6.57980421789207\\
31	6.59042602754726\\
32	6.59044529564649\\
33	6.59080569187187\\
34	6.59089611744954\\
35	6.58916534077888\\
36	6.59030181816158\\
37	6.59088695127249\\
38	6.58937832439878\\
39	6.59148671433531\\
40	6.59073721108426\\
41	6.59093404750414\\
42	6.59087742004054\\
43	6.58928541970452\\
44	6.59058564316579\\
45	6.59132828384597\\
46	6.59058851035584\\
47	6.59110263530547\\
48	6.59147553116448\\
49	6.58668393005854\\
50	6.59062363782734\\
51	6.59101515662169\\
52	6.59176251951947\\
53	6.58857868597947\\
54	6.59013423143927\\
55	6.59111909588808\\
56	6.5911457483472\\
57	6.59003852850639\\
58	6.59143982256265\\
59	6.59048536674524\\
60	6.59039176821099\\
61	6.59137154542976\\
62	6.5912535226034\\
63	6.58981582043099\\
64	6.5908918306807\\
65	6.59083300972301\\
66	6.59079121501016\\
67	6.59030883873119\\
68	6.59129857301594\\
69	6.59128229744416\\
70	6.59095332658505\\
71	6.59023081358638\\
72	6.59081865364851\\
73	6.59094026619255\\
74	6.59040142378658\\
75	6.59163684757572\\
76	6.59110676104285\\
77	6.59119701707623\\
78	6.58989897914213\\
79	6.5918191325617\\
80	6.59000392128006\\
81	6.59005624567473\\
82	6.59175293524183\\
83	6.5908999928299\\
84	6.59095496039431\\
85	6.58972726779356\\
86	6.59027090444657\\
87	6.59089103275845\\
88	6.59083663036135\\
89	6.59079717575128\\
90	6.59042755723988\\
91	6.59032670713473\\
92	6.59144685236426\\
93	6.5909022764905\\
94	6.59173282877921\\
95	6.58975668648628\\
96	6.58944811703204\\
97	6.5909704505714\\
98	6.58887279368135\\
99	6.59056883095371\\
100	6.59080516986559\\
};
\addplot [color=red, dashdotted, line width=2.0pt, forget plot]
  table[row sep=crcr]{%
1	0.962011577752915\\
2	1.39512833468571\\
3	1.90148489293405\\
4	1.51648055618043\\
5	2.15248666521931\\
6	1.55704328138457\\
7	1.66059810154288\\
8	1.52166767379182\\
9	1.60756285452165\\
10	1.57032375887657\\
11	1.57143632706561\\
12	1.54570012712595\\
13	1.53062124352496\\
14	1.54268502801898\\
15	1.53021067085092\\
16	1.55206916663035\\
17	1.54627922395726\\
18	1.51390174557473\\
19	1.50375418924415\\
20	1.51530747681262\\
21	1.52196737407071\\
22	1.50354898705708\\
23	1.51332011123504\\
24	1.52232901674692\\
25	1.52509841996306\\
26	1.52238627498701\\
27	1.5324730708447\\
28	1.51163420208493\\
29	1.51812829528253\\
30	1.52774156845851\\
31	1.52183786551656\\
32	1.5016521070917\\
33	1.49397777104937\\
34	1.51952082470017\\
35	1.52980925719026\\
36	1.52606705237248\\
37	1.52515602796127\\
38	1.50973845097226\\
39	1.52151282867688\\
40	1.53441438633646\\
41	1.52415222096287\\
42	1.51132177031528\\
43	1.5246991042189\\
44	1.50540365539768\\
45	1.50579300125202\\
46	1.5332974431763\\
47	1.52602816377732\\
48	1.52498943304033\\
49	1.52505167508116\\
50	1.52936828998447\\
51	1.52360673777914\\
52	1.53058250934036\\
53	1.52357239030125\\
54	1.52025395181835\\
55	1.53499050028086\\
56	1.51742772358661\\
57	1.51441659560209\\
58	1.51420820595266\\
59	1.5262997058346\\
60	1.53616242827644\\
61	1.52204501826377\\
62	1.51266612225863\\
63	1.5190132195052\\
64	1.52431344447975\\
65	1.51530419794324\\
66	1.52201531315481\\
67	1.51061899113799\\
68	1.51304669429054\\
69	1.52532175891914\\
70	1.53260170132874\\
71	1.52225905085154\\
72	1.50963037140345\\
73	1.52293726685294\\
74	1.52308518163793\\
75	1.4950927590338\\
76	1.50963368680608\\
77	1.51272511429317\\
78	1.51042633785373\\
79	1.50938242359903\\
80	1.5271682127886\\
81	1.5117185420591\\
82	1.52766007832391\\
83	1.51122470668482\\
84	1.51941129140273\\
85	1.51728704039085\\
86	1.52688068863037\\
87	1.50792216758675\\
88	1.51772700059157\\
89	1.51734457771922\\
90	1.51448813532051\\
91	1.53151859073494\\
92	1.52116535307287\\
93	1.51578575724323\\
94	1.52090562538418\\
95	1.51986831699292\\
96	1.50477239990879\\
97	1.52614502228845\\
98	1.4935503656753\\
99	1.51752656157145\\
100	1.53496968961844\\
};
\addplot [color=mycolor1, dotted, line width=3.0pt, forget plot]
  table[row sep=crcr]{%
1	4.6864935622248\\
2	4.8994383088347\\
3	4.56109313648738\\
4	4.90229035547094\\
5	2.9314300803944\\
6	3.2784332597462\\
7	2.46963199592331\\
8	2.30125152271061\\
9	2.01471564245944\\
10	2.17712372051914\\
11	1.78625426786003\\
12	1.85353123296186\\
13	1.99614342802232\\
14	1.65287461428729\\
15	1.49633284810289\\
16	1.48610098565703\\
17	1.45156396181048\\
18	1.47406316737917\\
19	1.46740662240884\\
20	1.50873895369633\\
21	1.45078965236607\\
22	1.45245977605543\\
23	1.71819478818221\\
24	2.08550862826438\\
25	1.5059355521904\\
26	1.47394259570511\\
27	1.75330416344819\\
28	1.38731873977995\\
29	1.70173662068909\\
30	1.46762136054097\\
31	1.50989567888606\\
32	1.52057424298062\\
33	1.52296770227207\\
34	1.51890013441686\\
35	1.51167634567613\\
36	1.52013496455013\\
37	1.53042250209518\\
38	1.51614460001237\\
39	1.51068803218364\\
40	1.52349879587756\\
41	1.5292094373918\\
42	1.52723516835788\\
43	1.50219367701738\\
44	1.50889764085272\\
45	1.53357949409282\\
46	1.5063426671682\\
47	1.52330362351915\\
48	1.52806722140564\\
49	1.48492641984674\\
50	1.50496302104093\\
51	1.51726658029489\\
52	1.51662375013351\\
53	1.48434709843271\\
54	1.51950876152417\\
55	1.51591219634932\\
56	1.52244113525343\\
57	1.51228811211186\\
58	1.51907869642937\\
59	1.52153657235561\\
60	1.52002671698928\\
61	1.52644490901635\\
62	1.51264572580944\\
63	1.5160114727441\\
64	1.52871982636947\\
65	1.52455740252168\\
66	1.52881396551966\\
67	1.52041366900985\\
68	1.51920423367879\\
69	1.52572400377125\\
70	1.52896228076713\\
71	1.50694486971597\\
72	1.5191129634431\\
73	1.53589170731109\\
74	1.51554168107362\\
75	1.52114588847516\\
76	1.5073119787335\\
77	1.53177729643885\\
78	1.49938469465641\\
79	1.52493589459921\\
80	1.51681359788873\\
81	1.50273815004913\\
82	1.53692226281642\\
83	1.51074892880607\\
84	1.53056388532285\\
85	1.52037492759945\\
86	1.51626804395327\\
87	1.5065737477751\\
88	1.52211768630356\\
89	1.52115051067854\\
90	1.521429717609\\
91	1.50829188644711\\
92	1.53568509407678\\
93	1.51780971278461\\
94	1.5267654175131\\
95	1.51765382436885\\
96	1.51725665471583\\
97	1.51173138823043\\
98	1.50648110697348\\
99	1.50703573590319\\
100	1.5100847459919\\
};
\end{axis}
\end{tikzpicture}%
  }
  \subfloat[$\lambda=1$]{
    \input{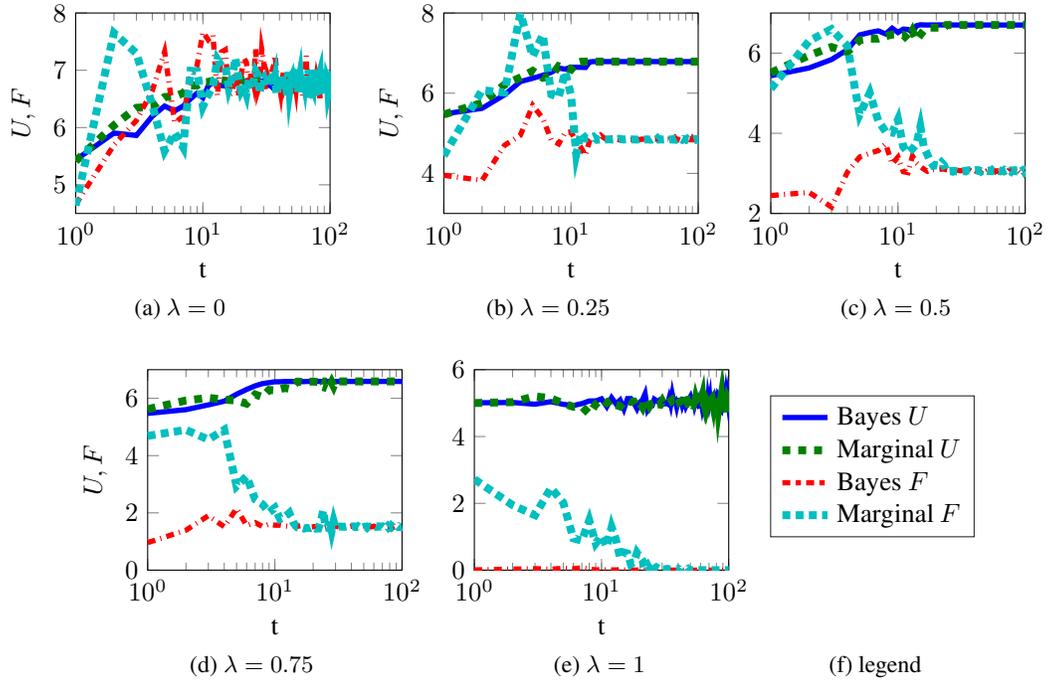}
  }
  \subfloat[legend]{
    \raisebox{4em}{\definecolor{mycolor1}{rgb}{0.00000,0.75000,0.75000}%

\begin{tikzpicture}

  \begin{axis}[%
    hide axis,
    xmin=10,
    xmax=50,
    ymin=0,
    ymax=0.4,
    legend style={draw=white!15!black,legend cell align=left}
    ]
    \addlegendimage{color=blue, line width=2.0pt}
    \addlegendentry{Bayes $U$};
    \addlegendimage{color=black!50!green, dashed, line width=3.0pt};
    \addlegendentry{Marginal $U$};
    \addlegendimage{color=red, line width=2.0pt, dashdotted}
    \addlegendentry{Bayes $F$};
    \addlegendimage{color=mycolor1, dotted, line width=3.0pt};
    \addlegendentry{Marginal $F$};
    
  \end{axis}
\end{tikzpicture}}
  }
  \caption{\textbf{Synthetic data, utility-fairness trade-off.} This plot is generated from the same data as Figure~\ref{fig_exp_1}. However, now we are plotting the utility and fairness of each individual policy separately. In all cases, it can be seen that the Bayesian policy achieves the same utility as the non-Bayesian policy, while achieving a lower fairness violation.}
  \label{fig_exp_1:tradeoff_extend}
\end{figure*}

\begin{figure*}
\centering
  \subfloat[$\lambda=0$]{
%
%
\begin{tikzpicture}

\begin{axis}[%
width=0.951\fwidth,
height=0.75\fwidth,
at={(0\fwidth,0\fwidth)},
scale only axis,
xmode=log,
xmin=1,
xmax=100,
xminorticks=true,
xlabel={$t \times 10$},
ymin=0.45,
ymax=0.7,
ylabel={V},
axis background/.style={fill=white}
]
\addplot [color=blue, line width=2.0pt, forget plot]
  table[row sep=crcr]{%
1	0.457219642968273\\
2	0.592769351348169\\
3	0.628816705720447\\
4	0.62505185351612\\
5	0.628217785699876\\
6	0.639218851111152\\
7	0.630873927371614\\
8	0.636227797683053\\
9	0.638026015044111\\
10	0.636823616047142\\
11	0.632410660316369\\
12	0.633597617267562\\
13	0.636527784056834\\
14	0.637892469025015\\
15	0.633703787028147\\
16	0.636328078671473\\
17	0.637880712599377\\
18	0.634684165215899\\
19	0.629978414138595\\
20	0.628159353629191\\
21	0.637486382196772\\
22	0.630647724209231\\
23	0.636203699395604\\
24	0.632324272807024\\
25	0.643830720098772\\
26	0.633034238724461\\
27	0.635123838924277\\
28	0.636057361395139\\
29	0.642104432395723\\
30	0.63957063775151\\
31	0.636676944667785\\
32	0.638328034343207\\
33	0.642713754763223\\
34	0.638813995795796\\
35	0.637222326954388\\
36	0.641876694490855\\
37	0.637397077631842\\
38	0.641452323880008\\
39	0.639449683802413\\
40	0.63998585281562\\
41	0.645562739043969\\
42	0.638373912813148\\
43	0.63869610092111\\
44	0.639248573866283\\
45	0.64553183974159\\
46	0.64574217834242\\
47	0.647159828876256\\
48	0.642734787921472\\
49	0.645351749272455\\
50	0.647986777824105\\
51	0.648498084835137\\
52	0.641666027891812\\
53	0.649788940171897\\
54	0.644719769480268\\
55	0.642816500136809\\
56	0.648563504192992\\
57	0.647484270514486\\
58	0.646522917641747\\
59	0.646362741362458\\
60	0.64943324763383\\
61	0.650725119238089\\
};
\addplot [color=red, dashed, line width=3.0pt, forget plot]
  table[row sep=crcr]{%
1	0.50533782292246\\
2	0.59699112902373\\
3	0.625888847773672\\
4	0.624084738953654\\
5	0.630607004491168\\
6	0.632831760088705\\
7	0.636757087308418\\
8	0.643065048813731\\
9	0.635831074924727\\
10	0.640172268918583\\
11	0.637618500393285\\
12	0.637063543459606\\
13	0.637801195651245\\
14	0.639132776275881\\
15	0.635237491522697\\
16	0.63711453971561\\
17	0.635829219303754\\
18	0.631722475647075\\
19	0.634325927457996\\
20	0.631533283036592\\
21	0.634408642218423\\
22	0.633190718318277\\
23	0.634003706915972\\
24	0.637844298466811\\
25	0.632637216101121\\
26	0.630133780104625\\
27	0.629845209068525\\
28	0.633468263655565\\
29	0.636356158428604\\
30	0.641203029255701\\
31	0.638991861329888\\
32	0.640709523818283\\
33	0.638816379513034\\
34	0.63827201778754\\
35	0.639309083888535\\
36	0.640322622034011\\
37	0.636038032575251\\
38	0.641624574388145\\
39	0.640674600649206\\
40	0.636610965813643\\
41	0.637518273518509\\
42	0.637490556877938\\
43	0.642085811731694\\
44	0.642900366689167\\
45	0.64514581492819\\
46	0.644982417065964\\
47	0.645169810588366\\
48	0.646840576097694\\
49	0.649222992097267\\
50	0.645768342344984\\
51	0.646975527164982\\
52	0.648342200211879\\
53	0.646810289670806\\
54	0.648718371906363\\
55	0.649127848411779\\
56	0.650500132527732\\
57	0.650312507421561\\
58	0.650414667141449\\
59	0.648262622639421\\
60	0.646204756681934\\
61	0.64655970208432\\
};
\end{axis}
\end{tikzpicture}%
  }
  \subfloat[$\lambda=0.25$]{
%
%
\begin{tikzpicture}

\begin{axis}[%
width=0.951\fwidth,
height=0.75\fwidth,
at={(0\fwidth,0\fwidth)},
scale only axis,
xmode=log,
xmin=1,
xmax=100,
xminorticks=true,
xlabel={$t \times 10$},
ymin=-0.4,
ymax=0.2,
axis background/.style={fill=white}
]
\addplot [color=blue, line width=2.0pt, forget plot]
  table[row sep=crcr]{%
1	-0.107309117160114\\
2	-0.184362598610003\\
3	-0.202687343033588\\
4	-0.16405353282543\\
5	-0.0404647596781372\\
6	-0.142185989426937\\
7	-0.0867070351173507\\
8	-0.103112003957054\\
9	-0.147973876301626\\
10	-0.0559405398752546\\
11	-0.0904189398323078\\
12	-0.143167279778263\\
13	-0.090396490677853\\
14	-0.0933118500216149\\
15	-0.118988174174151\\
16	-0.0678767358522148\\
17	-0.129987246948872\\
18	-0.186594510932254\\
19	-0.0215939023932779\\
20	-0.104510953104091\\
21	-0.13723838749799\\
22	-0.18620095487661\\
23	-0.212719706095191\\
24	-0.225196694541515\\
25	-0.113776632704301\\
26	-0.143391389100894\\
27	-0.243940532329244\\
28	-0.125244153045619\\
29	-0.266760554277998\\
30	-0.275116708400049\\
31	-0.219899848624522\\
32	-0.218521639522065\\
33	-0.221299282448636\\
34	-0.233220662963895\\
35	-0.209380517899935\\
36	-0.236064029388241\\
37	-0.212863027765073\\
38	-0.201394889444524\\
39	-0.237944741868967\\
40	-0.259988067335625\\
41	-0.212928270520192\\
42	-0.279107144092697\\
43	-0.247823881984387\\
44	-0.244473723197659\\
45	-0.27712897889869\\
46	-0.2939114684982\\
47	-0.262173085855822\\
48	-0.224572653789633\\
49	-0.277207371017433\\
50	-0.264808868869637\\
51	-0.255175102937293\\
52	-0.29510233737979\\
53	-0.284122136547406\\
54	-0.243506723452448\\
55	-0.263471235968001\\
56	-0.273090002692023\\
57	-0.282418540838345\\
58	-0.252055818552359\\
59	-0.258663444499106\\
60	-0.291293216722242\\
61	-0.229192294502873\\
};
\addplot [color=red, dashed, line width=3.0pt, forget plot]
  table[row sep=crcr]{%
1	0.156154700909796\\
2	-0.202529646913159\\
3	-0.334501624221845\\
4	-0.318155142166199\\
5	-0.153483607896056\\
6	-0.207611067782485\\
7	-0.185766432059466\\
8	-0.267942726423269\\
9	-0.230990254470568\\
10	-0.187853728223043\\
11	-0.207663739985019\\
12	-0.248003162170955\\
13	-0.226471600713461\\
14	-0.235871101675925\\
15	-0.286725135969644\\
16	-0.264493187308551\\
17	-0.253062643392712\\
18	-0.27076819747302\\
19	-0.241941070940066\\
20	-0.259674995051562\\
21	-0.261194466691931\\
22	-0.243502937063522\\
23	-0.258896757797132\\
24	-0.229200967686043\\
25	-0.243544823178511\\
26	-0.246523459323759\\
27	-0.239579465861304\\
28	-0.2058848765737\\
29	-0.244769717827798\\
30	-0.277044962490945\\
31	-0.261511328888575\\
32	-0.261486322591778\\
33	-0.261059257567426\\
34	-0.240391853350955\\
35	-0.209165659984946\\
36	-0.215611668997483\\
37	-0.219133607027449\\
38	-0.240239168432912\\
39	-0.26495817875911\\
40	-0.290806207461142\\
41	-0.271959186404205\\
42	-0.268198074806926\\
43	-0.26509097620102\\
44	-0.324057188614507\\
45	-0.32565688601495\\
46	-0.311684764206851\\
47	-0.333186997244714\\
48	-0.27544649692533\\
49	-0.250408797980645\\
50	-0.261630623670323\\
51	-0.250494309595537\\
52	-0.258998661381812\\
53	-0.277524904868859\\
54	-0.290114185134193\\
55	-0.268412694296168\\
56	-0.268939296683932\\
57	-0.267282191301071\\
58	-0.271482157224587\\
59	-0.272799271069239\\
60	-0.255638700719698\\
61	-0.255788000793323\\
};
\end{axis}
\end{tikzpicture}%
  }
  \subfloat[$\lambda=0.5$]{
%
%
\begin{tikzpicture}

\begin{axis}[%
width=0.951\fwidth,
height=0.75\fwidth,
at={(0\fwidth,0\fwidth)},
scale only axis,
xmode=log,
xmin=1,
xmax=100,
xminorticks=true,
xlabel={$t \times 10$},
ymin=-1.4,
ymax=0,
axis background/.style={fill=white}
]
\addplot [color=blue, line width=2.0pt, forget plot]
  table[row sep=crcr]{%
1	-1.09622117340752\\
2	-0.439738831631238\\
3	-0.477010262732245\\
4	-0.459566994825697\\
5	-0.267438768251134\\
6	-0.261884198822098\\
7	-0.314800806678737\\
8	-0.323934739140698\\
9	-0.204202454295437\\
10	-0.220866546747302\\
11	-0.224784320984486\\
12	-0.174850401323677\\
13	-0.204336810453115\\
14	-0.279220020694667\\
15	-0.304215110722433\\
16	-0.111899168962922\\
17	-0.196698433914746\\
18	-0.212212607848973\\
19	-0.234718297036081\\
20	-0.322668180703975\\
21	-0.286071421580403\\
22	-0.259499649273306\\
23	-0.316375891732036\\
24	-0.273840483206877\\
25	-0.355624719979204\\
26	-0.379757992540423\\
27	-0.288379194513057\\
28	-0.259927407512447\\
29	-0.331340818267421\\
30	-0.342494252379114\\
31	-0.343352143939611\\
32	-0.415313928296665\\
33	-0.422017029459941\\
34	-0.335179784452087\\
35	-0.440060447624601\\
36	-0.514029990888593\\
37	-0.43325890065029\\
38	-0.381527909556519\\
39	-0.357585870240933\\
40	-0.342513238337931\\
41	-0.351232807159593\\
42	-0.491960446194131\\
43	-0.414521352705201\\
44	-0.426725954088604\\
45	-0.442441956965056\\
46	-0.48463572780352\\
47	-0.539087119210198\\
48	-0.444488011314339\\
49	-0.50689510900555\\
50	-0.577616030570053\\
51	-0.438450894005509\\
52	-0.538629948619004\\
53	-0.427097161747581\\
54	-0.449574311025686\\
55	-0.549606168957645\\
56	-0.592101430233288\\
57	-0.633608692550753\\
58	-0.517199531938076\\
59	-0.606102743738442\\
60	-0.513954328614164\\
61	-0.657646711598371\\
};
\addplot [color=red, dashed, line width=3.0pt, forget plot]
  table[row sep=crcr]{%
1	-0.138097439995355\\
2	-0.748893406957869\\
3	-0.82460686564522\\
4	-0.686300787106743\\
5	-0.590927424830682\\
6	-0.644914108373764\\
7	-0.621536393076511\\
8	-0.732227183784745\\
9	-0.678985832465308\\
10	-0.677538882070215\\
11	-0.649492934072179\\
12	-0.626953188541119\\
13	-0.514430272564223\\
14	-0.552610078526062\\
15	-0.651427496731807\\
16	-0.610373827755142\\
17	-0.618730957216695\\
18	-0.62662953336171\\
19	-0.647029602765731\\
20	-0.771071949704786\\
21	-0.881634832338758\\
22	-0.787624432290519\\
23	-0.749027999750425\\
24	-0.768887901516103\\
25	-0.766265274524894\\
26	-0.754819034015115\\
27	-0.682709696959525\\
28	-0.755854713298122\\
29	-0.795893759226808\\
30	-0.87634212964454\\
31	-0.837761717305747\\
32	-0.800158278551069\\
33	-0.788079320268578\\
34	-0.757587585669062\\
35	-0.735977528675732\\
36	-0.792743783163313\\
37	-0.826465829671237\\
38	-0.79949075466688\\
39	-0.763326650871933\\
40	-0.812995026125147\\
41	-0.841258496638301\\
42	-0.811392209795322\\
43	-0.925172456274953\\
44	-0.981952927632142\\
45	-0.903004676499468\\
46	-0.985923064692685\\
47	-1.00004897431729\\
48	-0.981936952406439\\
49	-0.955363797496084\\
50	-0.95804465183494\\
51	-1.02561492499815\\
52	-1.00549051871228\\
53	-0.993986917287425\\
54	-1.03469402307911\\
55	-1.0713271054811\\
56	-1.11079025820582\\
57	-1.12157352569169\\
58	-1.13388551393297\\
59	-1.10161020503734\\
60	-1.10864470943151\\
61	-1.03972034327087\\
};
\end{axis}
\end{tikzpicture}%
  }
  \\
  \subfloat[$\lambda=0.75$]{
%
%
\begin{tikzpicture}

\begin{axis}[%
width=0.951\fwidth,
height=0.75\fwidth,
at={(0\fwidth,0\fwidth)},
scale only axis,
xmode=log,
xmin=1,
xmax=100,
xminorticks=true,
xlabel={$t \times 10$},
ymin=-1.4,
ymax=-0,
ylabel={V},
axis background/.style={fill=white}
]
\addplot [color=blue, line width=2.0pt, forget plot]
  table[row sep=crcr]{%
1	-0.577093123554267\\
2	-0.248954758659022\\
3	-0.35764860428935\\
4	-0.384489263771175\\
5	-0.137724767472587\\
6	-0.287617263649582\\
7	-0.238022447402015\\
8	-0.227177403202236\\
9	-0.167308886039736\\
10	-0.161376703486556\\
11	-0.244705523926808\\
12	-0.225120239981884\\
13	-0.207276792295031\\
14	-0.171776517037101\\
15	-0.172712713325723\\
16	-0.185088779203056\\
17	-0.167798962290238\\
18	-0.164028130972016\\
19	-0.203935674606876\\
20	-0.229890198812524\\
21	-0.206840332491921\\
22	-0.238091646259464\\
23	-0.175086145068783\\
24	-0.256161743484268\\
25	-0.191277333731401\\
26	-0.211347306437156\\
27	-0.19898722167119\\
28	-0.208658653631961\\
29	-0.259816824867176\\
30	-0.237928172498554\\
31	-0.23358861416853\\
32	-0.316512017404449\\
33	-0.204916004601392\\
34	-0.239785151677314\\
35	-0.285154711351199\\
36	-0.201371426057616\\
37	-0.279393824605483\\
38	-0.236736659060055\\
39	-0.258409042059961\\
40	-0.212586736782501\\
41	-0.366452766032414\\
42	-0.225163197594042\\
43	-0.400058658771094\\
44	-0.254925092948163\\
45	-0.283929055975851\\
46	-0.288689458039081\\
47	-0.303848260007582\\
48	-0.306914990520981\\
49	-0.299002998751941\\
50	-0.259695035886033\\
51	-0.320960950250418\\
52	-0.291040251313417\\
53	-0.354678279092358\\
54	-0.226764497936105\\
55	-0.406969528769398\\
56	-0.35735127127503\\
57	-0.450091430573555\\
58	-0.384435368648727\\
59	-0.344316593562978\\
60	-0.342268282505291\\
61	-0.356751765642486\\
};
\addplot [color=red, dashed, line width=3.0pt, forget plot]
  table[row sep=crcr]{%
1	-0.558867096944863\\
2	-1.09288402759395\\
3	-0.839601895925869\\
4	-1.13434859230935\\
5	-0.938751243985718\\
6	-1.02852382062057\\
7	-0.815833675753988\\
8	-1.05133374684411\\
9	-0.880020524351822\\
10	-0.710895330347298\\
11	-0.757053571799602\\
12	-0.693359796897153\\
13	-0.609044947073203\\
14	-0.619878551256528\\
15	-0.770333297321465\\
16	-0.745953732884505\\
17	-0.752773746980868\\
18	-0.842252441020057\\
19	-0.844781239988941\\
20	-0.999171417591342\\
21	-1.15237330771864\\
22	-1.01314018000917\\
23	-0.939869183249194\\
24	-0.82712673891788\\
25	-0.828311262541784\\
26	-0.773600331997152\\
27	-0.745596732067385\\
28	-0.841592443963456\\
29	-0.921138544961516\\
30	-1.02085899320577\\
31	-0.928805249697245\\
32	-0.87047255661201\\
33	-0.86578428332466\\
34	-0.821524625735823\\
35	-0.802852201942028\\
36	-0.869243511936604\\
37	-0.857984019399831\\
38	-0.789458059808354\\
39	-0.729281534798616\\
40	-0.84646490496439\\
41	-0.772860543050829\\
42	-0.79444912383881\\
43	-0.850558520633302\\
44	-0.943529592739689\\
45	-0.911562918124687\\
46	-0.953111291385843\\
47	-1.00327605127269\\
48	-1.02712018838934\\
49	-1.00136473675006\\
50	-1.01619500858576\\
51	-1.16391263914022\\
52	-1.20232686441201\\
53	-1.09666354315773\\
54	-1.16203449202619\\
55	-1.15049813617522\\
56	-1.15281503321226\\
57	-1.21027496105179\\
58	-1.15787421113916\\
59	-1.28995903552509\\
60	-1.21658911309429\\
61	-1.13001821316672\\
};
\end{axis}
\end{tikzpicture}%
  }
  \subfloat[$\lambda=1$]{
%
%
\begin{tikzpicture}

\begin{axis}[%
width=0.951\fwidth,
height=0.75\fwidth,
at={(0\fwidth,0\fwidth)},
scale only axis,
xmode=log,
xmin=1,
xmax=100,
xminorticks=true,
xlabel={$t \times 10$},
ymin=-2,
ymax=-0,
axis background/.style={fill=white}
]
\addplot [color=blue, line width=2.0pt, forget plot]
  table[row sep=crcr]{%
1	-0.00232770753700901\\
2	-1.97313695873102e-05\\
3	-2.37767934892516e-06\\
4	-5.41730805759902e-07\\
5	-7.32149168349694e-08\\
6	-2.39811834130854e-08\\
7	-5.18082968701592e-09\\
8	-5.80788316935145e-09\\
9	-1.76256829256548e-09\\
10	-1.38567073595432e-09\\
11	-2.7766696901406e-10\\
12	-2.3480673996106e-10\\
13	-8.92285754128819e-11\\
14	-1.19961233710043e-10\\
15	-2.67771625177127e-11\\
16	-2.1807283694217e-11\\
17	-9.888893758119e-12\\
18	-2.04587275550225e-11\\
19	-1.00294880288958e-11\\
20	-7.27273007814714e-12\\
21	-2.18409079376077e-12\\
22	-1.90552725474449e-12\\
23	-3.08552681136802e-12\\
24	-1.04407443807789e-12\\
25	-6.89534351010657e-13\\
26	-7.04452116553748e-13\\
27	-5.33265371320719e-13\\
28	-2.21733300835158e-13\\
29	-1.09917521624577e-13\\
30	-1.97092353458009e-13\\
31	-1.06901072975012e-13\\
32	-1.02563119096526e-13\\
33	-8.84436177317179e-14\\
34	-3.79100124163196e-14\\
35	-4.2067887937253e-14\\
36	-2.70027547549659e-14\\
37	-2.66635307650848e-14\\
38	-2.67287785600447e-14\\
39	-2.64342999326352e-14\\
40	-2.65179986462852e-14\\
41	-2.68229872585042e-14\\
42	-2.66997278710527e-14\\
43	-2.69707767201082e-14\\
44	-2.67044068810534e-14\\
45	-2.67210612428623e-14\\
46	-2.65715988911875e-14\\
47	-2.65002787170287e-14\\
48	-2.64873004782108e-14\\
49	-2.61218989350948e-14\\
50	-2.64084264642284e-14\\
51	-2.65721748735916e-14\\
52	-2.66424481150276e-14\\
53	-2.65125945760817e-14\\
54	-2.61390327175519e-14\\
55	-2.63298827930954e-14\\
56	-2.66575354658841e-14\\
57	-2.65806977191069e-14\\
58	-2.63970270948243e-14\\
59	-2.63521868646625e-14\\
60	-2.62708242678811e-14\\
61	-2.67337836777301e-14\\
};
\addplot [color=red, dashed, line width=3.0pt, forget plot]
  table[row sep=crcr]{%
1	-1.61749222996759\\
2	-0.501016985730073\\
3	-0.230141055920469\\
4	-0.211363679210412\\
5	-0.180351722603156\\
6	-0.199490315348852\\
7	-0.167503225130075\\
8	-0.197680427138255\\
9	-0.168326964955281\\
10	-0.108154945624742\\
11	-0.121796007045675\\
12	-0.123244092080481\\
13	-0.120059074298961\\
14	-0.12768082663036\\
15	-0.118695747205944\\
16	-0.14620525838291\\
17	-0.156154755993665\\
18	-0.15210970699618\\
19	-0.171018095781681\\
20	-0.159637204861449\\
21	-0.147168529193469\\
22	-0.282154360398568\\
23	-0.245469097598784\\
24	-0.252947839856794\\
25	-0.242482705863201\\
26	-0.244915288625687\\
27	-0.224359689620484\\
28	-0.214533978554422\\
29	-0.216114585568234\\
30	-0.208860180236654\\
31	-0.221529291908727\\
32	-0.205563531932023\\
33	-0.233055721377914\\
34	-0.23298981535382\\
35	-0.215062245554676\\
36	-0.226954153936872\\
37	-0.222039885512241\\
38	-0.198543948266946\\
39	-0.191065849011378\\
40	-0.189777971418282\\
41	-0.188793398961995\\
42	-0.190732460573578\\
43	-0.185246409681183\\
44	-0.184603807555449\\
45	-0.188647955865062\\
46	-0.191033781025243\\
47	-0.179745728112148\\
48	-0.178803266307534\\
49	-0.163926246089407\\
50	-0.173423292392264\\
51	-0.179302094659936\\
52	-0.181333239906458\\
53	-0.177046092214897\\
54	-0.184951112414626\\
55	-0.181338109208384\\
56	-0.147815635344293\\
57	-0.149881515239679\\
58	-0.14463523042435\\
59	-0.144488753661275\\
60	-0.150296785419916\\
61	-0.157708726945062\\
};
\end{axis}
\end{tikzpicture}%
  }
  \subfloat[legend]{
    \raisebox{4em}{\begin{tikzpicture}

  \begin{axis}[%
    hide axis,
    xmin=10,
    xmax=50,
    ymin=0,
    ymax=0.4,
    legend style={draw=white!15!black,legend cell align=left}
    ]
    \addlegendimage{color=mycolor1, line width=2.0pt}
    \addlegendentry{Bayes};
    \addlegendimage{color=mycolor2, dashed, line width=3.0pt};
    \addlegendentry{Marginal};
  \end{axis}
\end{tikzpicture}}
  }

  \caption{\textbf{COMPAS dataset.} Demonstration of balance on the COMPAS dataset. The plots show the value measured on the holdout set for the \textbf{Bayes} and \textbf{Marginal} balance.
  Figures (a-e) show the utility achieved under different choices of $\lambda$ as we we observe each of the  6,000 training data points. Utility and fairness are measured on the empirical distribution of the remaining data and it can be seen that the Bayesian approach dominates as soon as fairness becomes important, i.e. $\lambda > 0$.  }
  \label{fig:compas-dbn_extend}
\end{figure*}
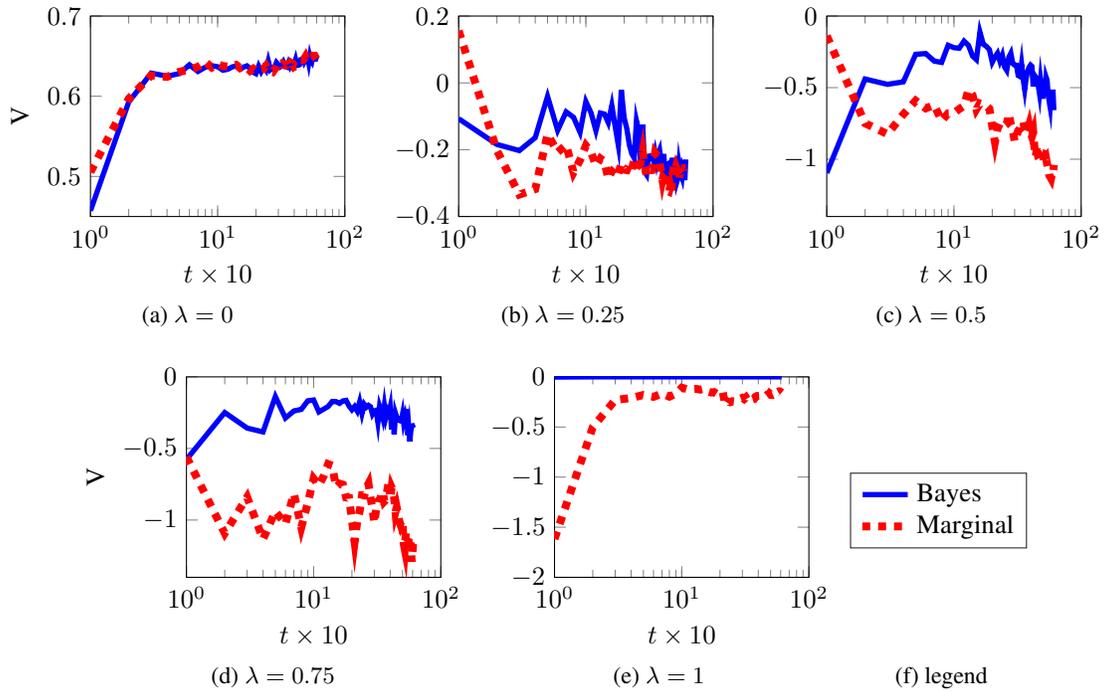

\begin{figure*}
\centering
  \subfloat[$\lambda=0$]{
%
%
\begin{tikzpicture}

\begin{axis}[%
width=0.951\fwidth,
height=0.75\fwidth,
at={(0\fwidth,0\fwidth)},
scale only axis,
xmode=log,
xmin=1,
xmax=100,
xminorticks=true,
xlabel style={font=\color{white!15!black}},
xlabel={t},
ymin=0.5,
ymax=0.64,
ylabel style={font=\color{white!15!black}},
ylabel={V},
axis background/.style={fill=white}
]
\addplot [color=blue, line width=2.0pt, forget plot]
  table[row sep=crcr]{%
1	0.533404412801252\\
2	0.567059367237681\\
3	0.574983684728794\\
4	0.595440941591605\\
5	0.597329784902714\\
6	0.598376899476387\\
7	0.608381899604062\\
8	0.614299903055019\\
9	0.614682328515722\\
10	0.622047920597586\\
11	0.624110283420211\\
12	0.623040076545587\\
13	0.622922815525939\\
14	0.622693703044924\\
15	0.624972493335543\\
16	0.623768267741923\\
17	0.625076223581217\\
18	0.624417108492841\\
19	0.624469368091222\\
20	0.626122024040929\\
21	0.627990230164747\\
22	0.627328710101433\\
23	0.627229984116765\\
24	0.627284500665436\\
25	0.628167660760683\\
26	0.62745907557843\\
27	0.627233400174376\\
28	0.626561170176999\\
29	0.626140421275828\\
30	0.627206627096838\\
31	0.626985313402765\\
32	0.627098543124935\\
33	0.626653517790801\\
34	0.62669323165553\\
35	0.626164136918594\\
36	0.625226096500675\\
37	0.624375614738648\\
38	0.626043052672986\\
39	0.625037401619658\\
40	0.624754173295781\\
41	0.625734326489491\\
42	0.624858959455676\\
43	0.624748862220927\\
44	0.625474170196668\\
45	0.624772055502031\\
46	0.624434789925371\\
47	0.624134265867323\\
48	0.624465786011358\\
49	0.62488377850369\\
50	0.624799243863898\\
51	0.623188769161788\\
52	0.622787015457619\\
53	0.625093107691527\\
54	0.62416250160529\\
55	0.624845855606283\\
56	0.625039701294087\\
57	0.625218860435963\\
58	0.626623043783051\\
59	0.626765171689463\\
60	0.626873002508487\\
};
\addplot [color=red, dashed, line width=3.0pt, forget plot]
  table[row sep=crcr]{%
1	0.537691550449067\\
2	0.52692425944197\\
3	0.51952806147473\\
4	0.516060917995437\\
5	0.525734464151558\\
6	0.526857982759685\\
7	0.527688486772251\\
8	0.553006242367054\\
9	0.580639537966924\\
10	0.583055435798432\\
11	0.583845485337254\\
12	0.585535021854244\\
13	0.589716220033718\\
14	0.591727972382982\\
15	0.592321109916687\\
16	0.593212336283605\\
17	0.59331788850473\\
18	0.592639117087132\\
19	0.594276492899578\\
20	0.593631864589812\\
21	0.593134373439542\\
22	0.592526975573227\\
23	0.59343823978082\\
24	0.59471503376963\\
25	0.594382031634667\\
26	0.59347186036185\\
27	0.594295880998185\\
28	0.592974355499971\\
29	0.594610054709994\\
30	0.594316274450942\\
31	0.59587687254927\\
32	0.594384171167264\\
33	0.59368169358206\\
34	0.593672061191477\\
35	0.59309575835021\\
36	0.591830149985295\\
37	0.592248365155305\\
38	0.592785655298101\\
39	0.592284775003569\\
40	0.591763395381349\\
41	0.591416781456517\\
42	0.593000957180953\\
43	0.591369641506618\\
44	0.592580298295271\\
45	0.592677941655292\\
46	0.591707641469623\\
47	0.592929785718232\\
48	0.591688464121018\\
49	0.591343754331541\\
50	0.589724575501398\\
51	0.590358577218326\\
52	0.590327436945083\\
53	0.590665615799822\\
54	0.595774337247164\\
55	0.595592315397091\\
56	0.595804057900644\\
57	0.59578874458756\\
58	0.596457066440588\\
59	0.597530717491227\\
60	0.597648444447298\\
};
\end{axis}
\end{tikzpicture}%
  }
  \subfloat[$\lambda=0.25$]{
%
%
\begin{tikzpicture}

\begin{axis}[%
width=0.951\fwidth,
height=0.75\fwidth,
at={(0\fwidth,0\fwidth)},
scale only axis,
xmode=log,
xmin=1,
xmax=100,
xminorticks=true,
xlabel style={font=\color{white!15!black}},
xlabel={t},
ymin=-0.7,
ymax=0.1,
axis background/.style={fill=white}
]
\addplot [color=blue, line width=2.0pt, forget plot]
  table[row sep=crcr]{%
1	-0.207196320566641\\
2	-0.425253875761661\\
3	-0.446970299122139\\
4	-0.41023006864754\\
5	-0.373322872219761\\
6	-0.367076094492257\\
7	-0.412878957537136\\
8	-0.413231937902242\\
9	-0.396049100003292\\
10	-0.373751675886177\\
11	-0.307771466008203\\
12	-0.142776108558127\\
13	-0.0925268949955039\\
14	-0.043827489725773\\
15	-0.0362719426975273\\
16	-0.0534966546269057\\
17	-0.0470706573695844\\
18	-0.0439917899143233\\
19	-0.0731027476537716\\
20	-0.0710715562201365\\
21	0.0127395664509126\\
22	0.0125535248449869\\
23	-0.0189231624399058\\
24	-0.0277232488542588\\
25	-0.0134581577090796\\
26	-0.0217588559150046\\
27	-0.0216893000684138\\
28	-0.0301266863372562\\
29	-0.0423852719635449\\
30	-0.0837378934262495\\
31	-0.0752088135448239\\
32	-0.0701158205484715\\
33	-0.0550343965029364\\
34	-0.0598645896360828\\
35	-0.0535803870353569\\
36	-0.0486584450371605\\
37	-0.0528834512409008\\
38	-0.09374799126579\\
39	-0.108641858592263\\
40	-0.116196652395813\\
41	-0.116755534241694\\
42	-0.0790210057246608\\
43	-0.101120407287318\\
44	-0.119825974775341\\
45	-0.114650814449016\\
46	-0.127224682409873\\
47	-0.15229978166135\\
48	-0.150131385782543\\
49	-0.158781774372714\\
50	-0.15623051925632\\
51	-0.167159607725622\\
52	-0.192822190369984\\
53	-0.187108430802122\\
54	-0.181673668158932\\
55	-0.208850712289315\\
56	-0.2069069917245\\
57	-0.211187039445116\\
58	-0.20459968944502\\
59	-0.190336379312404\\
60	-0.193366570484612\\
};
\addplot [color=red, dashed, line width=3.0pt, forget plot]
  table[row sep=crcr]{%
1	-0.310007172235672\\
2	-0.423714921637176\\
3	-0.410133018536837\\
4	-0.439747920197301\\
5	-0.455407707547421\\
6	-0.498202819436065\\
7	-0.5120224711627\\
8	-0.520697951706963\\
9	-0.523509545840905\\
10	-0.513259555010702\\
11	-0.534619050417941\\
12	-0.524377332258372\\
13	-0.540265530899788\\
14	-0.546714144110614\\
15	-0.557246115736649\\
16	-0.556180713008244\\
17	-0.506115024268118\\
18	-0.452408989176526\\
19	-0.438493393679697\\
20	-0.301299392319608\\
21	-0.143032833276735\\
22	-0.22920493119444\\
23	-0.36646075589015\\
24	-0.294566905747852\\
25	-0.243896037045222\\
26	-0.26842926302992\\
27	-0.276511738663329\\
28	-0.334838492474182\\
29	-0.36257984378723\\
30	-0.363663238136526\\
31	-0.377404312691054\\
32	-0.339176332999594\\
33	-0.316367727873905\\
34	-0.0775681732383819\\
35	-0.115417577113182\\
36	-0.174373859692571\\
37	-0.0543473840941811\\
38	-0.0646683542813737\\
39	-0.0736473944879555\\
40	-0.108832687588216\\
41	-0.153868009813003\\
42	-0.148133426159907\\
43	-0.191274332541181\\
44	-0.207928678469066\\
45	-0.206719417210556\\
46	-0.189150556247738\\
47	-0.20885630192149\\
48	-0.179619375326677\\
49	-0.196970384143575\\
50	-0.192503247410312\\
51	-0.237320806979659\\
52	-0.226844801344638\\
53	-0.261983151124036\\
54	-0.288565282674517\\
55	-0.298741252764581\\
56	-0.301494430295086\\
57	-0.296720140763373\\
58	-0.287235205015912\\
59	-0.278661589789526\\
60	-0.283053655854063\\
};
\end{axis}
\end{tikzpicture}%
  }
  \subfloat[$\lambda=0.5$]{
%
%
\begin{tikzpicture}

\begin{axis}[%
width=0.951\fwidth,
height=0.75\fwidth,
at={(0\fwidth,0\fwidth)},
scale only axis,
xmode=log,
xmin=1,
xmax=100,
xminorticks=true,
xlabel style={font=\color{white!15!black}},
xlabel={t},
ymin=-2,
ymax=-0,
axis background/.style={fill=white}
]
\addplot [color=blue, line width=2.0pt, forget plot]
  table[row sep=crcr]{%
1	-0.139557730363354\\
2	-0.238920667169255\\
3	-0.0997916326764938\\
4	-0.182229548458017\\
5	-0.205267446075831\\
6	-0.279826712648142\\
7	-0.332249305176689\\
8	-0.323237535828966\\
9	-0.334848601807552\\
10	-0.312990041492882\\
11	-0.324619975130561\\
12	-0.328870423071226\\
13	-0.313326571427902\\
14	-0.319808076558834\\
15	-0.30709006893331\\
16	-0.317914410723469\\
17	-0.331045983578041\\
18	-0.342385254226511\\
19	-0.358246869022327\\
20	-0.369248314322897\\
21	-0.377973473224499\\
22	-0.415440383538921\\
23	-0.38990705668132\\
24	-0.36686729244785\\
25	-0.375042658748737\\
26	-0.366006828119552\\
27	-0.383213880975782\\
28	-0.386114407046102\\
29	-0.380250046047834\\
30	-0.399759884881641\\
31	-0.437509783568944\\
32	-0.439357807489347\\
33	-0.443493819937183\\
34	-0.43823346019647\\
35	-0.448743781654721\\
36	-0.46057428729186\\
37	-0.465056676021326\\
38	-0.455095054216753\\
39	-0.463602980734887\\
40	-0.497134299702848\\
41	-0.456405475996486\\
42	-0.461607953659762\\
43	-0.467431543238734\\
44	-0.481699178907957\\
45	-0.497488409531242\\
46	-0.491566651303831\\
47	-0.489462737779279\\
48	-0.508962806633321\\
49	-0.497498682067592\\
50	-0.49354033957002\\
51	-0.497912056057912\\
52	-0.496093017678467\\
53	-0.505412512936155\\
54	-0.503382028933859\\
55	-0.52416214442533\\
56	-0.522603926762207\\
57	-0.544550008258789\\
58	-0.532153628446791\\
59	-0.520440892587704\\
60	-0.532834541878595\\
};
\addplot [color=red, dashed, line width=3.0pt, forget plot]
  table[row sep=crcr]{%
1	-0.759593109224573\\
2	-0.851910476827388\\
3	-0.931248398068395\\
4	-0.899099878562075\\
5	-0.7313280766324\\
6	-0.866207448068707\\
7	-0.990998285997023\\
8	-1.0272085536674\\
9	-1.04372995832914\\
10	-1.07714261034557\\
11	-1.07998626920429\\
12	-1.1166916978057\\
13	-1.09921640121892\\
14	-1.11870116207787\\
15	-1.14149626270435\\
16	-1.13602637450519\\
17	-1.19256259507435\\
18	-1.18855636890274\\
19	-1.16210813506113\\
20	-1.19312997954371\\
21	-1.16517236023971\\
22	-1.09072259193678\\
23	-1.07693553149219\\
24	-1.09320599058905\\
25	-1.12532578022789\\
26	-1.14329006175479\\
27	-1.17276265772882\\
28	-1.15208234230952\\
29	-1.18083392582687\\
30	-1.20703772222562\\
31	-1.22005481633382\\
32	-1.21874241865254\\
33	-1.26208489068812\\
34	-1.22937949908594\\
35	-1.26117480001724\\
36	-1.2690068818918\\
37	-1.26700389402755\\
38	-1.27433797411748\\
39	-1.28632957244659\\
40	-1.31543453440127\\
41	-1.32705144995612\\
42	-1.33602820271486\\
43	-1.3294527996667\\
44	-1.31996093710498\\
45	-1.32018176557428\\
46	-1.32494053379554\\
47	-1.33735592039034\\
48	-1.34339401770564\\
49	-1.33957969692087\\
50	-1.35053901850318\\
51	-1.34476886404259\\
52	-1.37641480485034\\
53	-1.35196334938376\\
54	-1.38384298807204\\
55	-1.40343614175046\\
56	-1.38977437984097\\
57	-1.37450630646827\\
58	-1.39440827751891\\
59	-1.41230095819747\\
60	-1.42020708154538\\
};
\end{axis}
\end{tikzpicture}%
  }
  \\
  \subfloat[$\lambda=0.75$]{
%
%
\begin{tikzpicture}

\begin{axis}[%
width=0.951\fwidth,
height=0.75\fwidth,
at={(0\fwidth,0\fwidth)},
scale only axis,
xmode=log,
xmin=1,
xmax=100,
xminorticks=true,
xlabel style={font=\color{white!15!black}},
xlabel={t},
ymin=-0.8,
ymax=-0,
ylabel style={font=\color{white!15!black}},
ylabel={V},
axis background/.style={fill=white}
]
\addplot [color=blue, line width=2.0pt, forget plot]
  table[row sep=crcr]{%
1	-0.0698670388969574\\
2	-0.0961890539986359\\
3	-0.132699144478001\\
4	-0.183101635988515\\
5	-0.250742140214741\\
6	-0.293266160267643\\
7	-0.276102246336143\\
8	-0.293448152823888\\
9	-0.344386347148344\\
10	-0.319149736907851\\
11	-0.297734928600805\\
12	-0.311571017833386\\
13	-0.294400148828445\\
14	-0.302231089798049\\
15	-0.243116324401653\\
16	-0.212083661389009\\
17	-0.233746426259225\\
18	-0.25286643122445\\
19	-0.236721961942786\\
20	-0.227153876544933\\
21	-0.239960732080747\\
22	-0.252464630359502\\
23	-0.221722655206653\\
24	-0.242234634999856\\
25	-0.215727187789023\\
26	-0.207245184391738\\
27	-0.198602366718053\\
28	-0.216362141232685\\
29	-0.203299221987791\\
30	-0.231551022346517\\
31	-0.223148752456014\\
32	-0.223513820452022\\
33	-0.234256655945202\\
34	-0.236933370776921\\
35	-0.266032198231509\\
36	-0.278628900515368\\
37	-0.264393610861771\\
38	-0.274638150183318\\
39	-0.29837283368362\\
40	-0.300363381099991\\
41	-0.302679151560454\\
42	-0.314870541600552\\
43	-0.278041572603173\\
44	-0.305742628530841\\
45	-0.286874359979552\\
46	-0.303919077702958\\
47	-0.328469119416157\\
48	-0.295742198577038\\
49	-0.32463667447715\\
50	-0.312820348014414\\
51	-0.322105149700822\\
52	-0.324641287036458\\
53	-0.297480905675212\\
54	-0.324260928208976\\
55	-0.33325489418501\\
56	-0.341060371144084\\
57	-0.342033707353245\\
58	-0.328434744104012\\
59	-0.330087267902574\\
60	-0.325796950505763\\
};
\addplot [color=red, dashed, line width=3.0pt, forget plot]
  table[row sep=crcr]{%
1	-0.798003405407076\\
2	-0.795088043311391\\
3	-0.748309023109161\\
4	-0.612609627478908\\
5	-0.58164859025222\\
6	-0.427949774021936\\
7	-0.454360176873906\\
8	-0.437345512177144\\
9	-0.402815955453931\\
10	-0.390244257342745\\
11	-0.419601914557706\\
12	-0.440904840008107\\
13	-0.39469526906995\\
14	-0.35483201391564\\
15	-0.353091583286641\\
16	-0.385734545472699\\
17	-0.387596782234304\\
18	-0.406852358977653\\
19	-0.409701930267887\\
20	-0.394138146001548\\
21	-0.417418249782155\\
22	-0.427383645577882\\
23	-0.393622565060579\\
24	-0.382161985890078\\
25	-0.375151768834539\\
26	-0.364797112866392\\
27	-0.397971634740718\\
28	-0.428159932609351\\
29	-0.382225019086256\\
30	-0.410457553255195\\
31	-0.435679219800089\\
32	-0.412057421507584\\
33	-0.441729441705593\\
34	-0.440313423318161\\
35	-0.416706751428254\\
36	-0.418300438155118\\
37	-0.397801980051729\\
38	-0.424732450230448\\
39	-0.437152223855908\\
40	-0.433518092412821\\
41	-0.39931985670458\\
42	-0.435267220777679\\
43	-0.437592218978853\\
44	-0.419300957364138\\
45	-0.416881037406579\\
46	-0.396128656626779\\
47	-0.425764606269481\\
48	-0.438534695701502\\
49	-0.432377081083798\\
50	-0.441463980518391\\
51	-0.444304763034727\\
52	-0.452474923000981\\
53	-0.448621837213215\\
54	-0.42064435309106\\
55	-0.421020067123781\\
56	-0.454665887594057\\
57	-0.465649610557996\\
58	-0.472048153258424\\
59	-0.471067458917294\\
60	-0.448310148007703\\
};
\end{axis}
\end{tikzpicture}%
  }
  \subfloat[$\lambda=1$]{
%
%
\begin{tikzpicture}

\begin{axis}[%
width=0.951\fwidth,
height=0.75\fwidth,
at={(0\fwidth,0\fwidth)},
scale only axis,
xmode=log,
xmin=1,
xmax=100,
xminorticks=true,
xlabel style={font=\color{white!15!black}},
xlabel={t},
ymin=-0.7,
ymax=-0,
axis background/.style={fill=white}
]
\addplot [color=blue, line width=2.0pt, forget plot]
  table[row sep=crcr]{%
1	-0.015991389253571\\
2	-8.07710251594484e-14\\
3	-2.59936539286564e-14\\
4	-2.60348785986828e-14\\
5	-2.60224200315188e-14\\
6	-2.59820525418944e-14\\
7	-2.59965938614317e-14\\
8	-2.60060611148019e-14\\
9	-2.59506757310226e-14\\
10	-2.58498969706858e-14\\
11	-2.58919987094477e-14\\
12	-2.59483989064604e-14\\
13	-2.595303495495e-14\\
14	-2.59150661948695e-14\\
15	-2.59299891535716e-14\\
16	-2.59070821300713e-14\\
17	-2.58991891382556e-14\\
18	-2.59151399206172e-14\\
19	-2.59036430407802e-14\\
20	-2.59284539232954e-14\\
21	-2.58943058916708e-14\\
22	-2.58760956319817e-14\\
23	-2.59459746304027e-14\\
24	-2.59175295022054e-14\\
25	-2.59825035699981e-14\\
26	-2.58938635371844e-14\\
27	-2.58310058320324e-14\\
28	-2.58420733678091e-14\\
29	-2.58783117412223e-14\\
30	-2.58440162581022e-14\\
31	-2.58429971080601e-14\\
32	-2.58351474843313e-14\\
33	-2.5869000612965e-14\\
34	-2.59091421141991e-14\\
35	-2.58739836061497e-14\\
36	-2.59131536622373e-14\\
37	-2.5829626726869e-14\\
38	-2.58271634195331e-14\\
39	-2.59047142325266e-14\\
40	-2.58443241715192e-14\\
41	-2.58396707757949e-14\\
42	-2.58215472522796e-14\\
43	-2.57845499373457e-14\\
44	-2.57547777456893e-14\\
45	-2.58627556084514e-14\\
46	-2.58798513083072e-14\\
47	-2.58572825558847e-14\\
48	-2.58443111610931e-14\\
49	-2.58880695607746e-14\\
50	-2.60105757326481e-14\\
51	-2.59875602889306e-14\\
52	-2.60626044265013e-14\\
53	-2.58993130083538e-14\\
54	-2.5891459860968e-14\\
55	-2.58690168759975e-14\\
56	-2.59416063798498e-14\\
57	-2.59585459545927e-14\\
58	-2.5878549181498e-14\\
59	-2.57975375951699e-14\\
60	-2.59211019483637e-14\\
};
\addplot [color=red, dashed, line width=3.0pt, forget plot]
  table[row sep=crcr]{%
1	-0.677360003940489\\
2	-0.319188261527385\\
3	-0.238573589451842\\
4	-0.200742996733391\\
5	-0.188801791973792\\
6	-0.157848169305053\\
7	-0.147363581418297\\
8	-0.151590768638053\\
9	-0.164301512639608\\
10	-0.198659955535581\\
11	-0.185472157543024\\
12	-0.189488927671266\\
13	-0.156665545344133\\
14	-0.141346012616817\\
15	-0.142828076379696\\
16	-0.143980205895357\\
17	-0.14494774267628\\
18	-0.145557100703442\\
19	-0.14609092309342\\
20	-0.137324402147981\\
21	-0.134764797162596\\
22	-0.140554448009126\\
23	-0.142899121135694\\
24	-0.143712900049346\\
25	-0.138774590312461\\
26	-0.13836037206626\\
27	-0.135282806601259\\
28	-0.149453779259262\\
29	-0.147470735870606\\
30	-0.154207490480793\\
31	-0.155796490189568\\
32	-0.158732344859596\\
33	-0.163488711963298\\
34	-0.166258826196893\\
35	-0.172242532053593\\
36	-0.168635982124319\\
37	-0.170553091320478\\
38	-0.171303646165357\\
39	-0.185837521836149\\
40	-0.180259389616782\\
41	-0.174263442475532\\
42	-0.162013378918257\\
43	-0.167114368469105\\
44	-0.169568455307836\\
45	-0.173271142995335\\
46	-0.191054413001525\\
47	-0.202850261177353\\
48	-0.209494088522992\\
49	-0.213131370871494\\
50	-0.211058535404276\\
51	-0.207988868337055\\
52	-0.20444195101075\\
53	-0.209951638264728\\
54	-0.218554354326512\\
55	-0.227181804255464\\
56	-0.233332843632866\\
57	-0.240435706083566\\
58	-0.246103352664639\\
59	-0.243526218854078\\
60	-0.236038905669898\\
};
\end{axis}
\end{tikzpicture}%
  }
  \subfloat[legend]{
    \raisebox{4em}{\begin{tikzpicture}

  \begin{axis}[%
    hide axis,
    xmin=10,
    xmax=50,
    ymin=0,
    ymax=0.4,
    legend style={draw=white!15!black,legend cell align=left}
    ]
    \addlegendimage{color=mycolor1, line width=2.0pt}
    \addlegendentry{Bayes};
    \addlegendimage{color=mycolor2, dashed, line width=3.0pt};
    \addlegendentry{Marginal};
  \end{axis}
\end{tikzpicture}}
  }
\caption{\textbf{Sequential allocation.} Performance measured with respect to the empirical model of the holdout COMPAS data, when the DM's actions affect which data will be seen. This means that wif a prisoner was not released, then the dependent variable $y$ will remain unseen. For that reason, the performance of the Bayesian approach dominates the classical approach even when fairness is not an issue, i.e. $\lambda = 0$.}
\label{fig:sequential-allocation_extend}
\end{figure*}
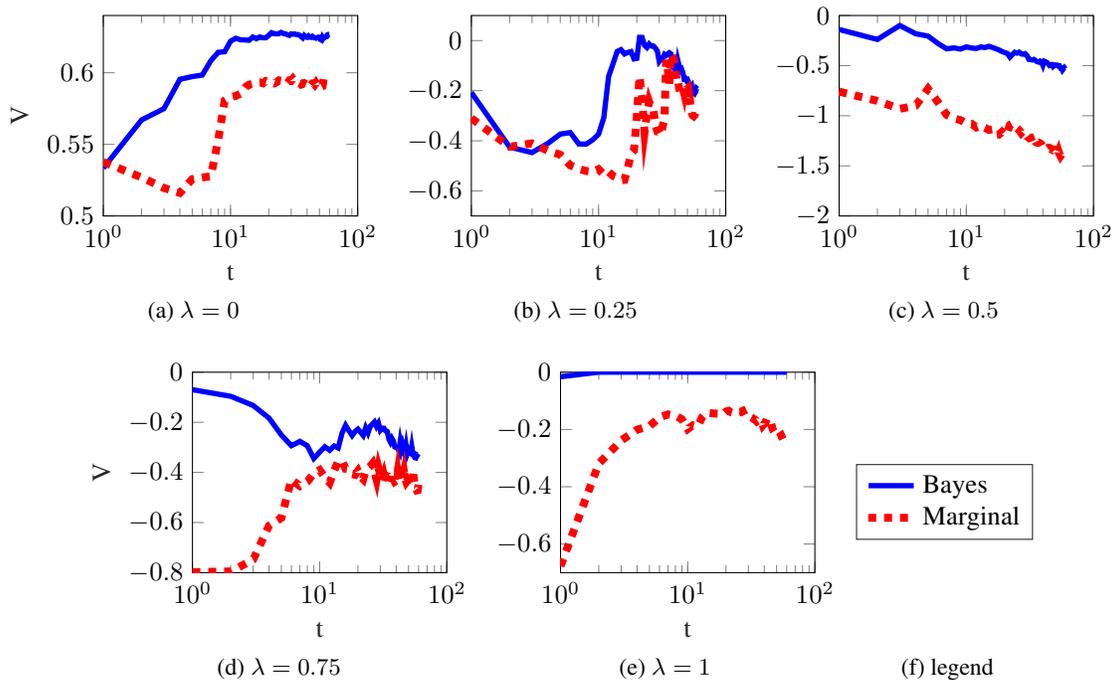



\end{document}